\documentclass[letterpaper]{article} 
\usepackage{aaai23}  
\usepackage{times}  
\usepackage{helvet}  
\usepackage{courier}  
\usepackage[hyphens]{url}  
\usepackage{graphicx} 
\usepackage{natbib}  
\usepackage{caption} 
\usepackage{algorithm}
\usepackage{algorithmic}

\usepackage{newfloat}
\usepackage{listings}
\DeclareCaptionStyle{ruled}{labelfont=normalfont,labelsep=colon,strut=off} 
\floatstyle{ruled}
\usepackage{algorithm}
\usepackage{algorithmic}

\usepackage{newfloat}
\usepackage{listings}
\usepackage{amsmath}
\usepackage{amsfonts}
\usepackage{amssymb}
\usepackage{multirow}
\usepackage{booktabs}
\usepackage{dsfont}
\usepackage{xcolor}
\usepackage{diagbox}
\usepackage{amsthm}
\usepackage{bm}
\usepackage{makecell}
\usepackage{subfigure}
\newtheorem{theorem}{Theorem}

\theoremstyle{definition}

\title{Leveraging Contaminated Datasets to Learn Clean-Data Distribution with Purified Generative Adversarial Networks}

\author {
    Bowen Tian\textsuperscript{\rm 1},
    Qinliang Su\textsuperscript{\rm 1,2 $\footnote{Corresponding author.}$},
    Jianxing Yu\textsuperscript{\rm 3}
}
\affiliations {
    \textsuperscript{\rm 1} School of Computer Science and Engineering, Sun Yat-sen University, Guangzhou, China\\
    \textsuperscript{\rm 2} Guangdong Key Laboratory of Big Data Analysis and Processing, Guangzhou, China\\
    \textsuperscript{\rm 3} School of Artificial Intelligence, Sun Yat-sen University, Guangdong, China\\
    tianbw@mail2.sysu.edu.cn, \{suqliang,yujx26\}@mail.sysu.edu.cn
}

\usepackage{bibentry}

\begin{document}

\maketitle

\begin{abstract}
Generative adversarial networks (GANs) are known for their strong abilities on capturing the underlying distribution of training instances. Since the seminal work of GAN, many variants of GAN have been proposed. However, existing GANs are almost established on the assumption that the training dataset is clean. But in many real-world applications, this may not hold, that is, the training dataset may be contaminated by a proportion of undesired instances. When training on such datasets, existing GANs will learn a mixture distribution of desired and contaminated instances, rather than the desired distribution of desired data only (target distribution). To learn the target distribution from contaminated datasets, two purified generative adversarial networks (PuriGAN) are developed, in which the discriminators are augmented with the capability to distinguish between target and contaminated instances by leveraging an extra dataset solely composed of contamination instances. We prove that under some mild conditions, the proposed PuriGANs are guaranteed to converge to the distribution of desired instances. Experimental results on several datasets demonstrate that the proposed PuriGANs are able to generate much better images from the desired distribution than comparable baselines when trained on contaminated datasets. In addition, we also demonstrate the usefulness of PuriGAN on downstream applications by applying it to the tasks of semi-supervised anomaly detection on contaminated datasets and PU-learning. Experimental results show that PuriGAN is able to deliver the best performance over comparable baselines on both tasks\footnote{Code is available at \url{https://github.com/tbw162/PuriGAN}.}.
\end{abstract}

\section{Introduction}

Learning data distribution from a dataset can be applied to various kinds of applications, like inpainting \cite{yu2018generative,liu2021pd}, anomaly detection \cite{AnoGAN17,zenati2018adversarially,akcay2018ganomaly}, image translation \cite{isola2017image,liu2017unsupervised}, AI medical diagnosis \cite{kazeminia2020gans,izadi2018generative}, {\it etc}. Among the existing methods of distribution learning,
generative adversarial networks (GANs) \cite{goodfellow2014generative} and variational auto-encoder (VAE) \cite{VAE2014} are the two most widely used ones. However, existing deep generative models are mostly established on the assumption of clean training datasets, that is, all training instances are drawn from the target distribution that we are interested in. But in real-world applications, it is quite common to see that some instances from an undesired distribution is mistakenly put into the training dataset, resulting in a contaminated dataset. This could be caused by the lack of sufficient expertise or enough labor to correctly recognize every undesired instance when building the dataset, or the high cost to clean a very large contaminated dataset etc. For example, it has been reported that a tiny ratio of images in ImageNet were labelled with incorrect categories, although their impacts on the final model trained on it is negligible due to their small proportion \cite{northcutt2021confident}. But for many applications, the quality of collected datasets could be much poorer than ImageNet. When training on such datasets, generative models will only capture the distribution of the entire dataset, that is, mixture of the target distribution and contamination distribution. However, as we use the generative models, our primary intention is always to learn the distribution of desired instances ({\it i.e.} target distribution), which can later be used to generate new desired instances or assist downstream tasks like anomaly detection \cite{AnoGAN17,zenati2018adversarially,akcay2018ganomaly}, inpainting \cite{yu2018generative,liu2021pd} etc. Therefore, investigating how to learn a generative model that only captures the distribution of desired instances from a contaminated dataset is important both theoretically and practically.

Generally, the goal above cannot be achieved by solely leveraging the contaminated dataset, which is denoted by ${\mathcal{X}}$. In this paper, in addition to ${\mathcal{X}}$, we also assume the availability of another small dataset ${\mathcal{X}}^-$ that is only composed of contamination instances. In many real-world applications, it is often possible to collect a small number of representative contamination instances, although collecting a large number of them may be difficult. For example, in the task of anomaly detection, it is possible to collect some anomalies, which can be seen as the contamination instances here. In fact, several recent works have been proposed to introduce an extra negative dataset and leverage it to boost the generation performance \cite{asokan2020teaching,sinha2020negative}. However, they are all restricted to the scenario that the dataset ${\mathcal{X}}$ is clean. To only learn the distribution of desired instances from a contaminated dataset, a possible solution is to adopt a two-stage training strategy: 1) first, seeking to separate out the target samples from ${\mathcal{X}}$; 2) then, training generative models on the separated target samples. The objective of first stage can be partially achieved by resorting to positive-unlabelled (PU) learning \cite{elkan2008learning,kiryo2017positive,du2015convex,kato2018learning,bekker2020learning}, whose goal is to partition unlabelled instances into two classes by leveraging datasets ${\mathcal{X}}$ and ${\mathcal{X}}^-$. However, it is observed that it is generally difficult to obtain a satisfactory performance at the first stage, especially when the data instances are complex. The poor performance of the first stage will be passed down to the second stage, resulting in an even worse performance of the whole model. Moreover, the valuable negative dataset ${\mathcal{X}}^-$ is not fully utilized by the two-stage method since it is only used in the first stage to separate out desired instance but is never used in the second stage.

In this paper, two purified generative adversarial networks (PuriGAN) are proposed, which can not only be trained in an end-to-end manner by simultaneously leveraging the two datasets ${\mathcal{X}}$ and ${\mathcal{X}}^-$, but also are guaranteed to converge to the target distribution theoretically. This is achieved by augmenting the discriminator to have it able to distinguish between the target and contaminated instances, in addition to the basic role of discriminating real and generated ones. The augmented discriminator can prevent the generator from generating contaminated instances, while the generator can generate extra  instances to increase the discrimination ability of discriminator. This further improves the discriminator's robustness even in the case of insufficient data. Extensive experiments are conducted to evaluate the target-data generation ability of PuriGAN when it is trained on contaminated datasets. Experimental results demonstrate that the proposed PuriGANs are able to generate much better desired images over competitive baselines under various kinds of conditions. In addition, we also apply PuriGAN to two downstream applications, semi-supervised anomaly detection on contaminated dataset and PU-learning, with the experimental results demonstrating that PuriGAN outperforms comparable baselines remarkably.

\section{The Proposed Purified Generative Adversarial Networks}
\label{method}
\subsection{Problem Description}
In this problem, we suppose the training dataset 
\begin{equation}
    {\mathcal{X}}=\{{\mathbf{x}}_1, {\mathbf{x}}_2, \cdots, {\mathbf{x}}_n\}
\end{equation}
is not pure, but is contaminated by some undesired instances. That is, the training dataset ${\mathcal{X}}$ contains desired and undesired (contamination) instances simultaneously. Mathematically, 
we can think the instances ${\mathbf{x}}_i$ are drawn from the mixture distribution
\begin{equation}
	p_d({\mathbf{x}}) = \pi p^+({\mathbf{x}}) + (1-\pi)p^-({\mathbf{x}}),
\end{equation}
where $p^+({\mathbf{x}})$ and $p^-({\mathbf{x}})$ denote the distributions of desired instances (target distribution) and contamination instances (contamination distribution), respectively; and $\pi$ is the proportion of desired instances . In addition to ${\mathcal{X}}$, there exists another training dataset 
\begin{equation}
    {\mathcal{X}}^{-} = \{{\mathbf{x}}_1^-, {\mathbf{x}}_2^-, \cdots, {\mathbf{x}}_m^- \},
\end{equation}
which is only composed of contamination instances with ${\mathbf{x}}_i^- \sim p^-({\mathbf{x}})$. Here, we argue that it is often possible to obtain a small number of contamination instances. For examples, in anomaly detection, we can collect a small number of anomalies, which can be viewed as the contamination instances; or given a contaminated dataset, we can assign some labors to find a fraction of contaminations in the dataset manually. But due to the often low proportion/frequency of contamination instances, the number of collected contamination instances cannot be too large. Thus, the size of ${\mathcal{X}}^-$ is assumed to be much smaller than ${\mathcal{X}}$. 

The problem concerned in this paper is to obtain a generative model which only captures the target distribution $p^+({\mathbf{x}})$ by leveraging the available datasets ${\mathcal{X}}$ and ${\mathcal{X}}^-$. Obviously, if we directly train a generative model ({\it e.g.}, GANs, VAE) on the dataset ${\mathcal{X}}$, the generative distribution $p_g({\mathbf{x}})$ will only converge to the distribution of dataset ${\mathcal{X}}$, {\it i.e.}, the mixture distribution $p_d({\mathbf{x}}) = \pi p^+({\mathbf{x}}) + (1-\pi) p^-({\mathbf{x}})$, instead of our desired target distribution $p^+({\mathbf{x}})$. Thus, the key to address this problem lies at how to leverage the provided contamination instances ${\mathbf{x}}_i^-$ effectively, in addition to the dataset ${\mathcal{X}}$.

\subsection{PuriGAN with Two-Level Discriminator}\label{PuriGAN_Disjoint}

To address the problem above, in this paper, we establish our model on the framework of GANs, or more specifically on the least-square GAN (LSGAN), thanks to its flexibility in the design of discriminators. In LSGAN, the generator $G(\cdot)$ and discriminator $D(\cdot)$ are updated as
\begin{align} 
	\min_D V_{LS}(D) \! =& {\mathbb{E}}_{{\mathbf{x}}\sim p_{d}({\mathbf{x}})}\left[(D({\mathbf{x}}) - 1)^2\right] \nonumber \\
	& +  {\mathbb{E}}_{{\mathbf{x}} \sim p_g({\mathbf{x}})}\left[(D({\mathbf{x}}) - 0)^2\right], \label{Discrim_LSGAN}\\
	\min_{G} V_{LS}(G) \! = &{\mathbb{E}}_{{\mathbf{x}}\sim p_{d}({\mathbf{x}})}\left[(D({\mathbf{x}}) - 0.5 )^2\right] \nonumber \\ 
	&+  {\mathbb{E}}_{{\mathbf{z}} \sim p({\mathbf{z}})} \left[(D(G({\mathbf{z}})) - 0.5)^2\right], \label{G_LSGAN}
\end{align}
where $p_g({\mathbf{x}})$ denotes the distribution of generated samples, that is, $G({\mathbf{z}}) \sim p_g({\mathbf{x}})$; and $p({\mathbf{z}})$ is a standard normal distribution. It can be seen from \eqref{Discrim_LSGAN} and \eqref{G_LSGAN} that LSGAN encourages the discriminator to output 1 for samples from the data distribution $p_d({\mathbf{x}})$ and 0 for the generated ones, while forcing the generator to confuse the discriminator, {\it i.e.}, letting it output 0.5. The training objective of LSGAN is consistent with traditional GANs, except that it is implemented under the least-squared loss. From \eqref{Discrim_LSGAN}, it can be seen that the optimal discriminator of LSGAN is $D^*({\mathbf{x}}) = \frac{p_g({\mathbf{x}})}{p_d({\mathbf{x}})+ p_g({\mathbf{x}})}$, which views all samples from $p_d(x) = \pi p^+({\mathbf{x}}) + (1-\pi)p^-({\mathbf{x}})$ are identical and does not distinguish between the samples from $p^+({\mathbf{x}})$ and $p^-({\mathbf{x}})$.

To enable the LSGAN to learn the desired data distribution from contaminated datasets, we propose to augment the LSGAN's discriminator by adding an extra term ${\mathbb{E}}_{{\mathbf{x}}\sim p^-({\mathbf{x}})}\!\! \left[(D(\mathbf{x}) \!-\! 0)^2\right]$ in the $V_{LS}(D)$ to enable the induced discriminator to recognize the contamination instances
\begin{align} \label{Discrim_pol}
	\min_D V(D)  = & {\mathbb{E}}_{{\mathbf{x}}\sim p_{d}({\mathbf{x}})}\!\left[(D({\mathbf{x}}) - 1)^2\right] \nonumber \\
	& +\! {\mathbb{E}}_{{\mathbf{x}} \sim p_g({\mathbf{x}})}\! \left[(D({\mathbf{x}})- 0)^2\right] \nonumber \\
	& + \! \lambda {\mathbb{E}}_{{\mathbf{x}}\sim p^-({\mathbf{x}})}\! \left[(D(\mathbf{x}) - 0)^2\right],
\end{align}
where $\lambda$ is a weighting parameter; and $p_g({\mathbf{x}})$ denotes the distribution of generated samples. Obviously, with the extra term ${\mathbb{E}}_{{\mathbf{x}}\sim p^-({\mathbf{x}})}\!\! \left[(D(\mathbf{x}) \!-\! 0)^2\right]$, the discriminator seeks to output $0$ for samples from ${\mathcal{X}}^-$ and $1$ for samples from ${\mathcal{X}}$. Although ${\mathcal{X}}$ contains both desired and contamination instances, if we set $\lambda$ to be a very large value, the discriminator can recognize the contamination instances from $p^-({\mathbf{x}})$. With the discriminator $D(\cdot)$ derived from \eqref{Discrim_pol}, if we train the generator to have $D(\cdot)$ outputting a nonzero value $c$, the generator will endeavour to avoid generating instances that look like contamination instances. Specifically, we propose to update the generator as
\begin{align} 
	\min_G V(G) = & {\mathbb{E}}_{{\mathbf{x}}\sim p_{d}({\mathbf{x}})}\!\! \left[(D({\mathbf{x}}) \!-\! c)^2\right] \nonumber \\
	& +\! {\mathbb{E}}_{{\mathbf{z}} \sim p({\mathbf{z}})}\!\! \left[(D(G({\mathbf{z}})) \!-\! c)^2\right] \!\nonumber \\
	& +{\mathbb{E}}_{{\mathbf{x}}\sim p^-({\mathbf{x}})}\!\! \left[(D(\mathbf{x}) \!-\! c)^2\right], \label{Genera_pol}
\end{align}
where $c$ could be any value within $(0, 1)$. With the updating rules specified in \eqref{Discrim_pol} and \eqref{Genera_pol} for $D(\cdot)$ and $G(\cdot)$, we can prove that under some conditions, the generator distribution $p_g({\mathbf{x}})$ will converge to the desired target distribution $p^+({\mathbf{x}})$.
\begin{theorem}
	\label{theorem_disjoint}
	When $D(\cdot)$ and $G(\cdot)$ are updated according to \eqref{Discrim_pol} and \eqref{Genera_pol}, the optimal discriminator is
	\begin{equation}\label{Opt_Discrim_Disjoint}
		D^*({\mathbf{x}}) = \frac{p_d(\mathbf{x})}{p_d({\mathbf{x}}) + p_g({\mathbf{x}}) + \lambda p^-({\mathbf{x}})};
	\end{equation}
Moreover, by supposing the support of target and contamination distributions are disjoint, {\it i.e.}, $Supp(p^+({\mathbf{x}})) \cap Supp(p^-({\mathbf{x}})) = \emptyset$ and $\lambda \to + \infty$, the generator distribution $p_g({\mathbf{x}})$ will converge to the target distribution $p^+({\mathbf{x}})$.
\end{theorem}
\begin{proof}
	Please refer to the Supplementary Materials.
\end{proof}

Here, we provide a sketch of the proof to gain some insights to the theorem. By deriving the derivatives of $V(D)$ in \eqref{Discrim_pol} w.r.t. $D(\cdot)$ and setting it to $0$, the optimal discriminator $D^*({\mathbf{x}})$ in \eqref{Opt_Discrim_Disjoint} can be easily obtained.  To see the convergence result $p_g({\mathbf{x}}) \to p^+({\mathbf{x}})$, substituting $p_d({\mathbf{x}}) = \pi p^+({\mathbf{x}}) + (1-\pi) p^-({\mathbf{x}})$ and the optimal discriminator $D^*({\mathbf{x}})$ in \eqref{Opt_Discrim_Disjoint} into the $V(G)$ in \eqref{Genera_pol} gives
\begin{align}
\begin{split}
    &V(G)
    \\&=
    \int_{\mathcal{S}}\!\bigg \{ \pi\left(\frac{\pi p^+ + (1-\pi)p^-}{\pi p^+ + (1-\pi)p^-+\lambda p^- +p_g}-c\right)^{2}\!\!p^{+}
    \\&\quad+\left(1\!-\!\pi\right)\left(\frac{\pi p^+ + (1-\pi)p^-}{\pi p^+ + (1-\pi)p^-+\lambda p^- +p_g} \!-\! c\right)^{2}\!\! p^{-}
    \\&\quad +\left(\frac{\pi p^+ + (1-\pi)p^-}{\pi p^+ + (1-\pi)p^-+\lambda p^- +p_g}-c\right)^{2}\!\! p_{g} 
    \\&\quad +\!\left( \frac{\pi p^+ + (1-\pi)p^-}{\pi p^+ + (1-\pi)p^-+\lambda p^- +p_g} \!-\! c \!\right)^{2}\!\!\! p^{-} \!\bigg \} \mathrm{d} \boldsymbol{x},
    \label{proof2:GEobjective1}
\end{split}
\end{align}
where the argument in the distribution is omitted for conciseness, {\it e.g.}, $p^+({\mathbf{x}})$ is abbreviated as $p^+$; and $\mathcal{S}$ represents the entire real space. 
From the assumption that the supports of $p^+(\mathbf{x})$ and $p^-(\mathbf{x})$ are disjoint, we have $p^-(\mathbf{x})=0$ when $p^+(\mathbf{x})>0$. Similarly, we must have $p^+(\mathbf{x})=0$ when $p^-(\mathbf{x})>0$. We thus divide the whole space $\mathcal{S}$ into two non-overlapped sub-spaces $\mathcal{S}_1$ and $\mathcal{S}_2$, with $p^+(\mathbf{x})\ge0$ and $p^-(\mathbf{x})=0$ in $\mathcal{S}_1$, and $p^-(\mathbf{x})\ge0$ and $p^+(\mathbf{x})=0$ in $\mathcal{S}_2$. Based on this observation, \eqref{proof2:GEobjective1} can be simplified as
\begin{align}
\begin{split}
    &V(G)=
    \int_{\mathcal{S}_1}\bigg \{\left(\frac{\pi p^+}{\pi p^+ +p_g}-c\right)^{2}\!\! (\pi p^{+} + p_g) \bigg \}\mathrm{d} \boldsymbol{x}
    \\& \;\;\; +\int_{\mathcal{S}_2} \!\!\! \bigg \{ \!\! \left(2-\pi\right)\!\left(\frac{(1-\pi)p^-}{(1-\pi)p^-+\lambda p^- +p_g}-c\right)^{2}\!\!p^{-}
    \\&\;\;\; +\left(\frac{(1-\pi)p^-}{(1-\pi)p^-+\lambda p^- +p_g}-c\right)^{2} \!\!p_{g} \bigg \} \mathrm{d} \boldsymbol{x}.
    \label{proof1:Gobjective2}
\end{split}
\end{align}
When $\lambda$ is set to be a very lage number, we can see that $\frac{(1-\pi)p^-(\mathbf{x})}{(1-\pi)p^-(\mathbf{x})+\lambda p^-(\mathbf{x}) +p_g(\mathbf{x})}$ converges to 0. Based on this observation, \eqref{proof1:Gobjective2} can be further written as
\begin{align}
\begin{split}
    V(G) =& \int_{\mathcal{S}_1}\bigg \{\left(\frac{\pi p^+}{\pi p^+ +p_g}-c\right)^{2}\!\! (\pi p^{+} + p_g) \bigg \}\mathrm{d} \boldsymbol{x} 
    \\&+ \int_{\mathcal{S}_2}\!\!\bigg \{ c^2(2-\pi) p^-\!+\!c^2p_g \bigg \} \mathrm{d} \boldsymbol{x}.
    \label{proof1:Gobjective3}
\end{split}
\end{align}
Define the function $\varphi(x) \triangleq (x-c)^2$. Obviously, $\varphi(x)$ is a convex function, thus according to Jensen inequality, we must have $\varphi \!\left(\int_{-\infty}^{\infty} g(x)p(x) \mathrm{d} x\right)\!\!\le\!\! \int_{-\infty}^{\infty} \varphi\left(g(x)\right)\!p(x)\ \mathrm{d} {x} $ for any  distribution $p(x)$ and function $g(\cdot)$. By denoting $\int_{\mathcal{S}_1}p_g \mathrm{d} \boldsymbol{x}$ as $\alpha$, then we must have $\int_{\mathcal{S}_2}p_g \mathrm{d} \boldsymbol{x} = 1 - \alpha$. Combining with the inequality above, we can infer from \eqref{proof1:Gobjective3} that
\begin{align}
    V(G) \ge  (\pi+\alpha) \ \varphi\left(\frac{\pi}{\pi+\alpha}\right)+c^2(3- \pi - \alpha).
    \label{proof1:Gobjective5}
\end{align}
It can be easily shown that the r.h.s. of \eqref{proof1:Gobjective5} is a  monotonic decreasing function of $\alpha$, thus the r.h.s. of \eqref{proof1:Gobjective5} is minimized when $\alpha$ is equal to 1. Thus,  the inequality
\begin{align}
    V(G) \ge (1+\pi)\ \varphi\left(\frac{\pi}{1+\pi}\right) +c^2 \left(2-\pi\right)
    \label{proof1:optimalg1}
\end{align}
always holds, which is obtained by setting $\alpha=1$. On the other hand, if we substitute $p_g=p^+$ into \eqref{proof1:Gobjective3}, we obtain
\begin{align}
    V(G) = (1+\pi)\ \varphi\left(\frac{\pi}{1+\pi}\right) +c^2 \left(2-\pi\right).
    \label{proof1:V(G)pg_equal_p+}
\end{align}
Comparing \eqref{proof1:V(G)pg_equal_p+} to \eqref{proof1:optimalg1}, we can see that $V(G)$ attains its global minima when $p_g=p^+$. Therefore, if $D(\cdot)$ and $G(\cdot)$ are updated according to \eqref{Discrim_pol} and \eqref{Genera_pol}, under the specified conditions, the generator distribution will converge to the target distribution $p^+({\mathbf{x}})$. The rigorous and detailed proof is given in the Supplementary.

In practice, the disjoint support condition $Supp(p^+({\mathbf{x}})) \cap Supp(p^-({\mathbf{x}})) = \emptyset$ could be considered being satisfied when the desired and contamination instances look sufficiently different. But when they share many similarities, the generator's ability of only generating desired instances could be compromised. For the parameter $\lambda$, theoretically, it should be set very large. But in practice, since we may not be able to find a classifier to separate the target and contamination instances, if $\lambda$ is set too large, the classifier is likely to classify all instances to $0$, which, obviously, will weaken the discriminator's ability of distinguishing between the target and contamination instances. Hence, there should be a balance on the choice of $\lambda$. We observe that it only has a minor influence on the performance as long as it is not set too large or too small ({\it e.g.,} $1\le \lambda \le 5$). We simply set it to $1$ in all experiments of this paper.

\subsection{PuriGAN with Three-Level Discriminator}\label{PuriGAN_three-level}

For the discriminator of two-level PuriGAN, it outputs the same value $0$ for both contamination instances from $p^-({\mathbf{x}})$ and generated instances from $p_g({\mathbf{x}})$, making it lack the ability to distinguish between the two types of insances. To further improve the generation performance, we propose to further augment the discriminator by requiring it output three different values for instances from ${\mathcal{X}}$, ${\mathcal{X}}^-$ and $p_g({\mathbf{x}})$, respectively. To this end, we propose to update the discriminator as follows
\begin{align} \label{Discrim_Relax}
	\min_D V(D) = &{\mathbb{E}}_{{\mathbf{x}}\sim p_d({\mathbf{x}})}\!\left[(D({\mathbf{x}}) - 1)^2\right] \nonumber \\
	& + {\mathbb{E}}_{\mathbf{x} \sim p_g({\mathbf{x}})}\!\left[(D({\mathbf{x}})-0)^2\right]\nonumber \\
	&+ {\mathbb{E}}_{{\mathbf{x}} \sim p^-({\mathbf{x}})} \! \left[(D({\mathbf{x}})- d)^2\right] ,
\end{align}
where $d$ is a value that will be specifically set later. Since the discriminator is designed to output three different values, it should possess some ability to distinguish the three types of instances. With the augmented discriminator derived from \eqref{Discrim_Relax}, the generator can be trained by encouraging the discriminator $D(\cdot)$ to output $c$ for all instances, that is,
\begin{align} \label{Generator_Relax}
	\min_G V(G) = &{\mathbb{E}}_{{\mathbf{x}}\sim p_{d}({\mathbf{x}})}\!\! \left[(D({\mathbf{x}}) \!-\! c)^2\right]  \nonumber \\ 
	&\!+\! {\mathbb{E}}_{{\mathbf{z}} \sim p({\mathbf{z}})}\!\! \left[(D(G({\mathbf{z}})) \!-\! c)^2\right] \nonumber \\ 
	&\!+\!   {\mathbb{E}}_{{\mathbf{x}}\sim p^-({\mathbf{x}})}\!\! \left[(D(\mathbf{x}) \!-\! c)^2\right],
\end{align}
where $c$ could be any value within $(0, 1)$. We can also prove that under some mild conditions, the generator distribution $p_g({\mathbf{x}})$ will converge to the target distribution.
\begin{theorem}
	\label{theorem_contaminated}
	When $D(\cdot)$ and $G(\cdot)$ are updated according to \eqref{Discrim_Relax} and \eqref{Generator_Relax}, the optimal discriminator is
	\begin{equation}
	\label{optimal_three_level_discri}
		D^*({\mathbf{x}}) = \frac{p_d({\mathbf{x}}) + d \cdot p^-({\mathbf{x}})}{p_d({\mathbf{x}}) + p_g({\mathbf{x}}) + p^-({\mathbf{x}})};
	\end{equation}
 Moreover, if $d$ is set as 
 \begin{equation}
 	d=\frac{2\pi-1}{\pi+1},
 	\end{equation}
  the generator distribution $p_g({\mathbf{x}})$ will converge to the target distribution $p^+({\mathbf{x}})$.
\end{theorem}
\begin{proof}
	Please refer to the Supplementary Materials.
\end{proof}

The proof of Theorem \ref{theorem_contaminated} is very different from that of Theorem  \ref{theorem_disjoint} since the subtle premises of disjoint support $Supp(p^+({\mathbf{x}})) \cap Supp(p^-({\mathbf{x}})) = \emptyset$ and infinite weighting parameter $\lambda$ in Theorem \ref{theorem_disjoint} are not required. The key to prove Theorem \ref{theorem_contaminated} is to find an appropriate value for $d$ such that the generator distribution $p_g({\mathbf{x}})$ induced from the updating equations equals to target distribution $p^+({\mathbf{x}})$. We refer readers to the Supplementary for detailed and rigorous proof.

Theorem \ref{theorem_contaminated} does not rely on the subtle disjoint support and infinite weighting parameter $\lambda$ condition, but it requires to know the ratio of desired instances $\pi$. In practice, despite it is difficult to know the exact value of $\pi$, many methods have been developed to estimate the ratio that a type of instances accounts for in a dataset \cite{ramaswamy2016mixture,christoffel2016class,jain2016estimating}. Thus, we can use these existing methods to estimate the value of $\pi$. To evaluate the influence of using a estimated $\pi$, we conduct sensitive analysis to the parameter $\pi$ in our experiments, revealing that the generation performance is not sensitive to small estimation error of $\pi$. Thus, an estimate of $\pi$ is enough to deliver a competitive performance for PuriGAN. It is also worth pointing out that under the special case of $\pi = 0.5$, we can see that the updating rules of $G(\cdot)$ and $D(\cdot)$ in three-level PuriGAN are exactly the same as those in two-level PuriGAN with $\lambda=1$.

\section{Applications}
\label{sec_application_examples}

In addition to generate novel instances from $p^+({\mathbf{x}})$, PuriGAN can be applied to lots of downstream tasks. In this section, two examples of them are demonstrated.

\begin{table*}[!tb]
\centering
\small
\setlength\tabcolsep{3pt}
\renewcommand\arraystretch{0.95}
{\begin{tabular}{cccccccccc}
\toprule
\makecell[c]{Data}   & \makecell[c]{LSGAN} & \makecell[c]{NDA} &  \makecell[c]{GenPU} & \makecell[c]{Rumi-LSGAN}& \makecell[c]{PU-LSGAN} & \makecell[c]{PU-NDA} &  \makecell[c]{PuriGAN$_1$} & \makecell[c]{PuriGAN$_2$} \\
\midrule

MNIST  & 26.93 $\pm$ 5.7 &21.94 $\pm$ 8.9 & 21.28 $\pm$ 6.0 & 13.31 $\pm$ 2.4 & 15.21 $\pm$ 3.2 & 10.35 $\pm$ 3.1&9.71 $\pm$ 1.9 &\textbf{9.51 $\pm$ 1.0} \\

F-MNIST  & 62.44 $\pm$ 6.8 & 58.88 $\pm$ 14.3 & 58.73 $\pm$ 8.0 & 37.24 $\pm$ 4.4 & 46.94 $\pm$ 5.1 & 44.43 $\pm$ 5.4 & 37.61 $\pm$ 4.6 &\textbf{34.86 $\pm$ 3.2} \\

SVHN   & 28.50 $\pm$ 4.9 & 27.58 $\pm$ 5.2  & 26.87 $\pm$ 5.6 & 21.08 $\pm$ 2.5 & 26.44 $\pm$ 3.9 & 23.21 $\pm$ 3.5 & 20.32 $\pm$ 2.4 & \textbf{19.63 $\pm$ 3.6} \\

CelebA & 49.29 $\pm$ 1.9 & 52.93 $\pm$ 2.3 & 45.81 $\pm$ 1.8 & 42.37 $\pm$ 1.9 & 44.75 $\pm$ 2.3  & 46.34 $\pm$ 2.8  & 36.43 $\pm$ 1.5 & \textbf{35.67 $\pm$ 1.5} \\

CIFAR-10 & 61.08 $\pm$ 9.9 & 72.95 $\pm$ 10.7 &  62.59 $\pm$ 10.8 & 56.28 $\pm$ 10.7 & 59.26 $\pm$ 10.7 & 70.12 $\pm$ 11.8  & 54.70 $\pm$ 10.4 & \textbf{52.70 $\pm$ 10.6}\\
\bottomrule
\end{tabular}
\vspace{-2mm}
\caption{Comparison of FID scores$\downarrow$ on different datasets under $\gamma_p=0.4$ and $\gamma_c=0.2$, where PuriGAN$_1$ and PuriGAN$_2$ denote PuriGAN using two- and three-level discriminator, respectively.}
\vspace{-4mm}
\label{Overall_fidscore}}
\end{table*}

\subsection{Anomaly Detection on Contaminated Datasets}  
One of the widely used anomaly detection approaches are to train a generative model (GAN \cite{goodfellow2014generative} or VAE \cite{VAE2014}) on a dataset that only contains normal samples to capture the distribution of normal samples. Anomalies can then be detected using the criteria derived from the generative models, such as the density value, reconstruction error \cite{AnoGAN17,zenati2018adversarially,akcay2018ganomaly} etc. However, in many scenarios, the collected dataset is often mixed with a proportion of anomalous instances. Thus, if we directly train a generative model on the contaminated dataset, it would learn the mixture distribution of nomral instances and anomalies, compromising its ability to detect anomalies. A remedy to this is to collect a small number of anomalies and then train PuriGAN on the two datasets to only learn the distribution of normal instances. In our experiments, we employ the output probability of PuriGAN's discriminator as the detection criteria by noticing that the value partially indicates the normality of input instance. Obviously, more sophisticated criteria can be used, such as leveraging the features in the discriminator or using the reconstruction error etc. But this is beyond the focus of this paper, hence we leave it for future work.

\subsection{PU-learning} 
The task of PU-learning is to classify an unlabelled dataset into two classes when an extra dataset containing only one class of instances is available \cite{kiryo2017positive,du2015convex}. Obviously, the setting of PU-learning fits with PuriGAN naturally. To perform PU-learning, we propose to train the two-level PuriGAN on the two provided datasets and then use the discriminator to classify the unlabelled instances by noticing that the discriminator is designed to distinguish between the two types of instances. Comparing to traditional PU-learning methods that directly train a classifier, the generation part in PuriGAN plays a role of data augmentation by generating new training samples, which is potentially able to lead to a more competitive performance.

\begin{figure*}[!tb]
    \centering
    \begin{minipage}{0.2\linewidth}
        \centering
        \includegraphics[width=1\linewidth]{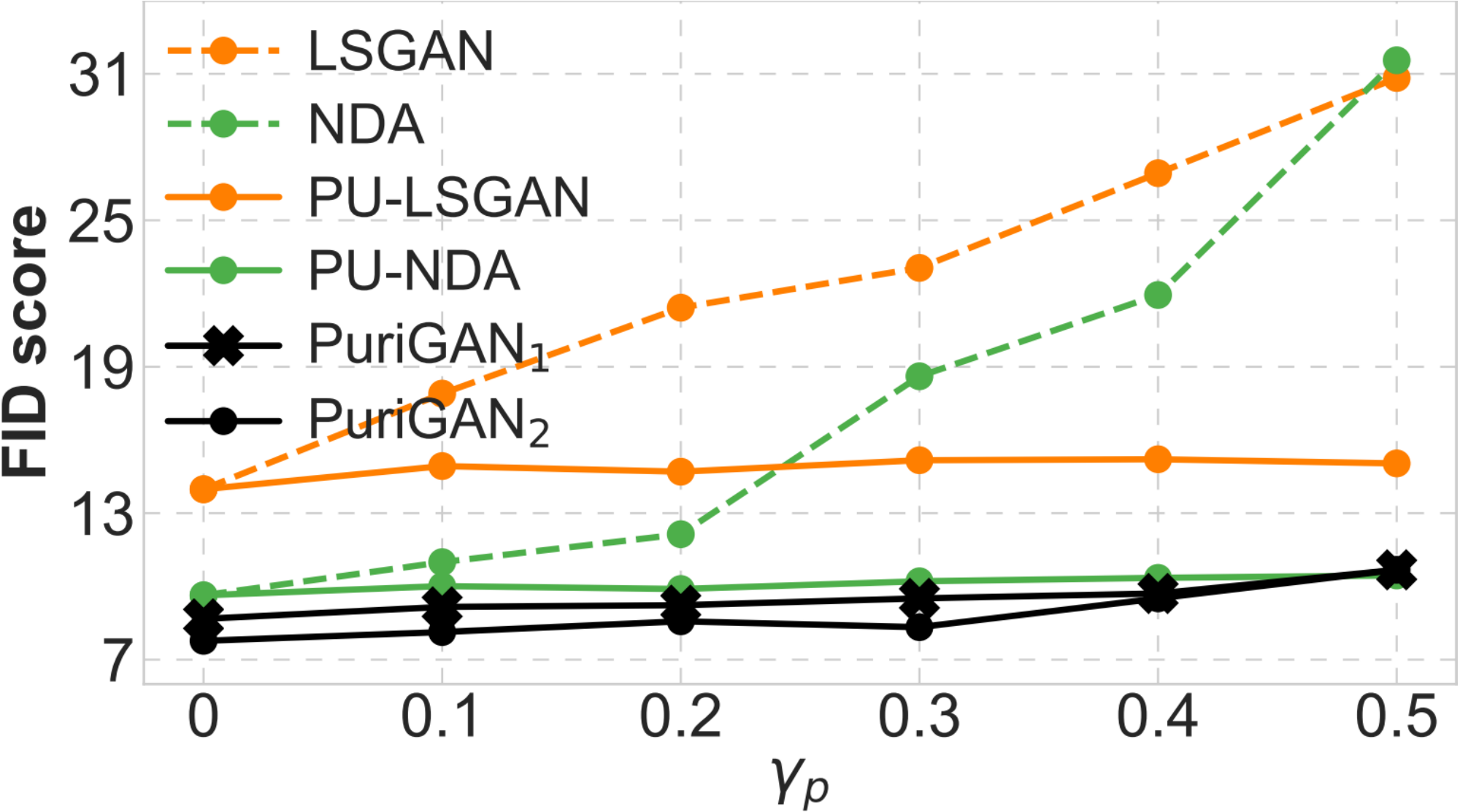}
        \vspace{-6mm}
        \caption*{\scriptsize MNIST}
    \end{minipage}
    \hspace{0.03\linewidth}
    \begin{minipage}{0.2\linewidth}
        \centering
        \includegraphics[width=1\linewidth]{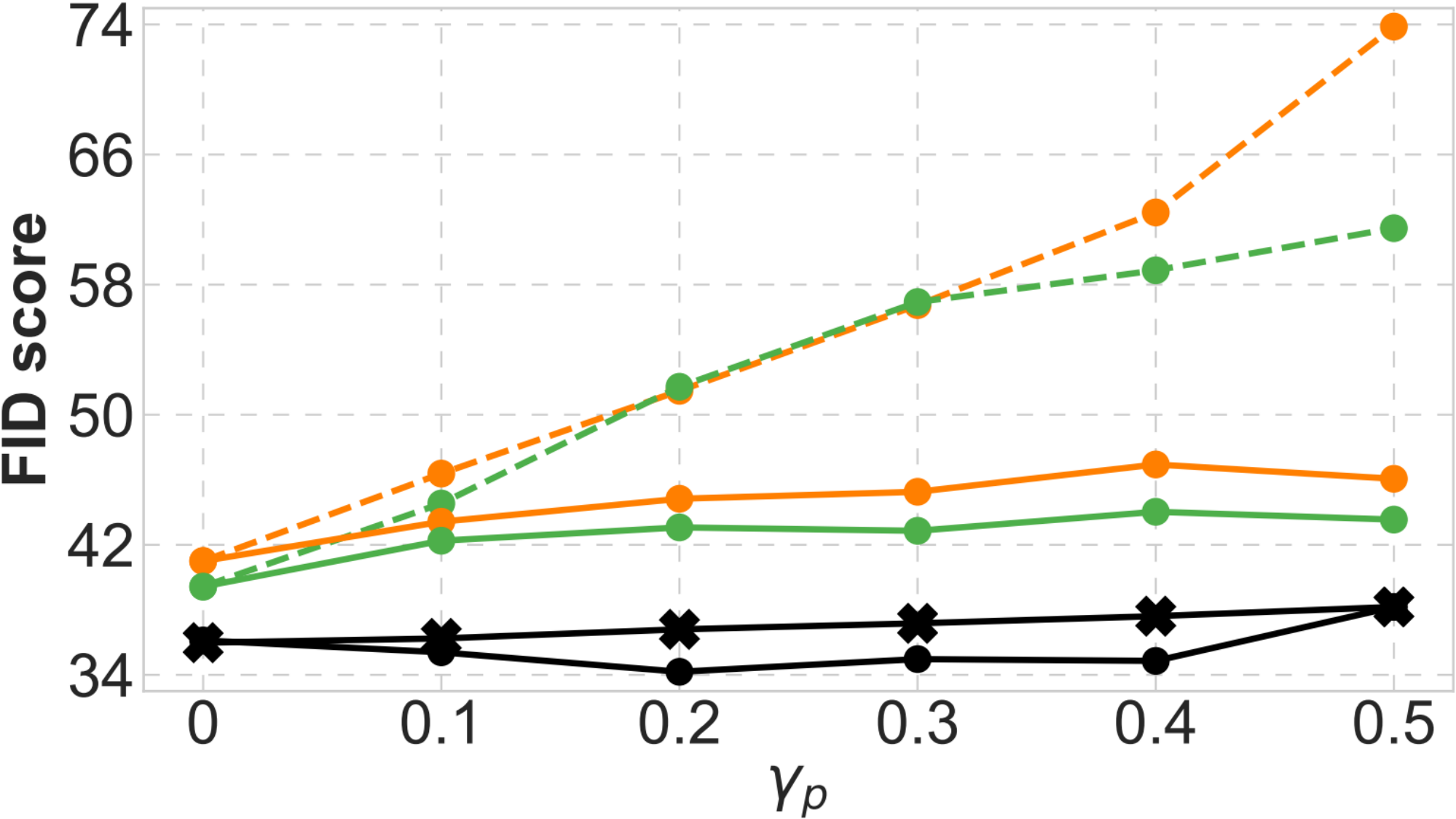}
        \vspace{-6mm}
        \caption*{\scriptsize F-MNIST}
    \end{minipage}
    \hspace{0.03\linewidth}
    \begin{minipage}{0.2\linewidth}
        \centering
        \includegraphics[width=1\linewidth]{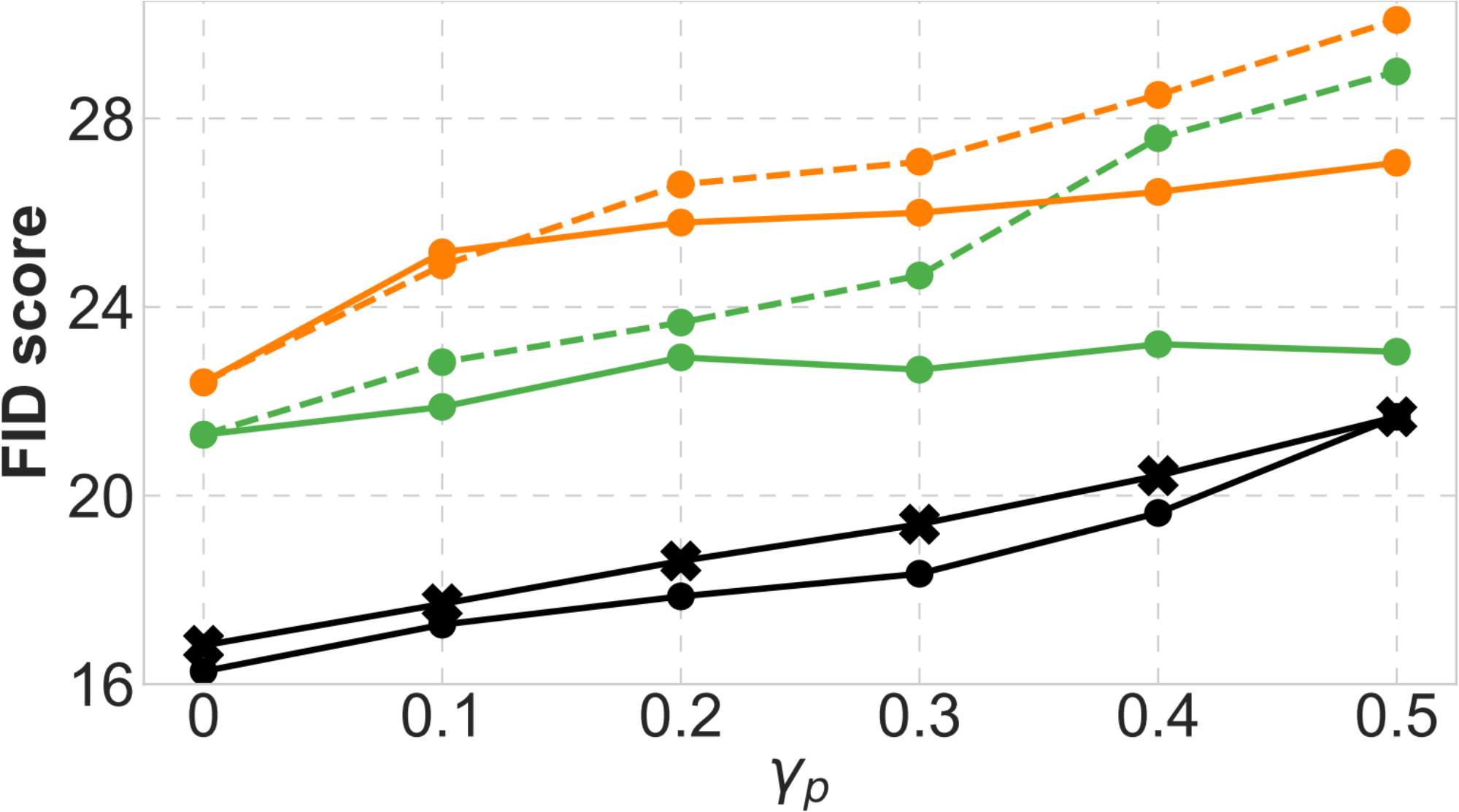}
        \vspace{-6mm}
        \caption*{\scriptsize SVHN}
    \end{minipage}
    \hspace{0.03\linewidth}
    \begin{minipage}{0.2\linewidth}
        \centering
        \includegraphics[width=1\linewidth]{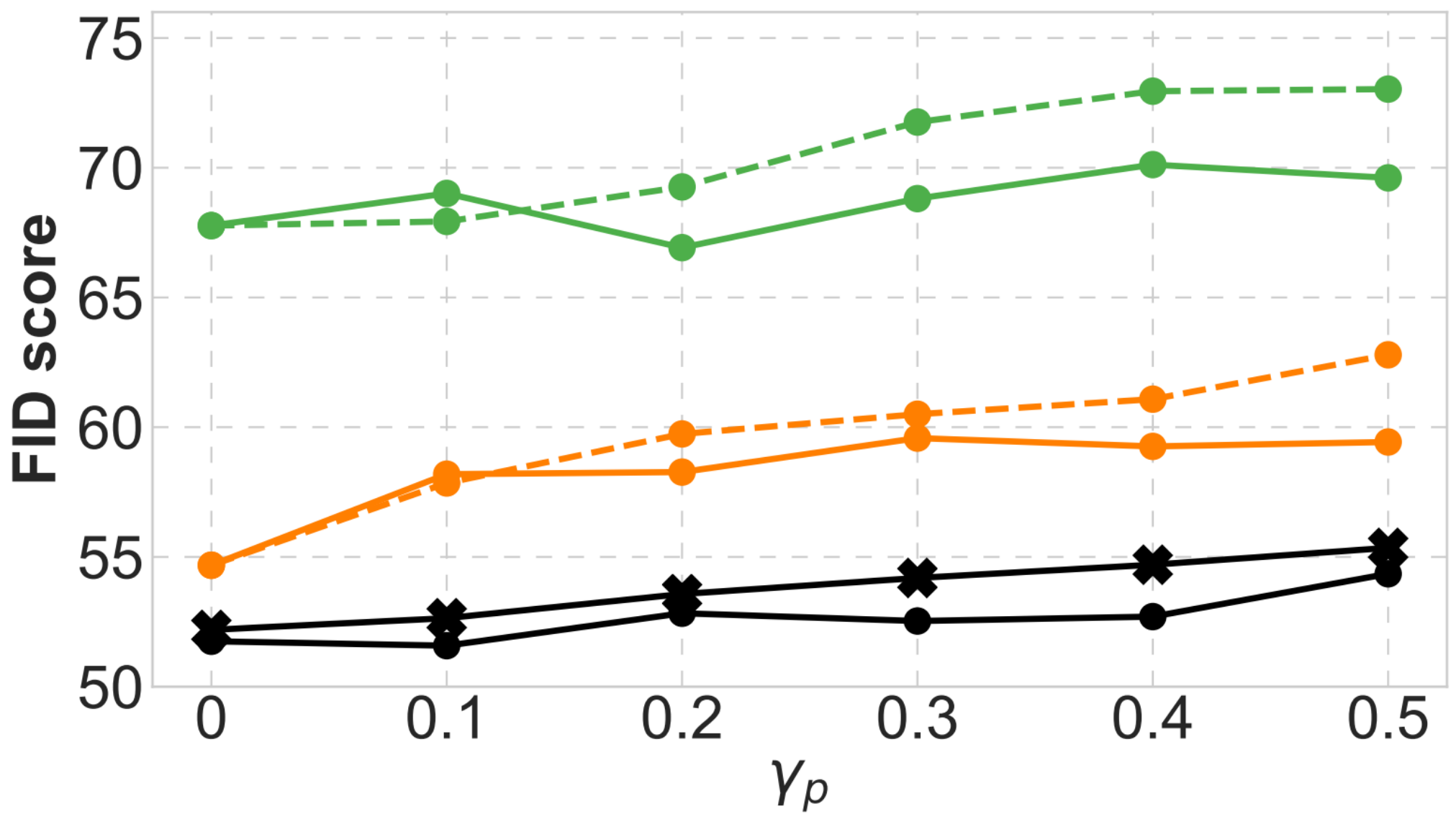}
        \vspace{-6mm}
        \caption*{\scriptsize CIFAR-10}
    \end{minipage}
\vspace{-2mm}
\caption{FID score as a function of contamination ratio
$\gamma_p$ under a fixed $\gamma_c=0.2$, where PuriGAN$_1$ and PuriGAN$_2$ denotes PuriGAN using two-level and three-level discriminator, respectively.}
\vspace{-1mm}
\label{FID-contamination-ratio}
\end{figure*}

\begin{figure*}[!tb]
    \centering
    \begin{minipage}{0.2\linewidth}
        \centering
        \includegraphics[width=1\linewidth]{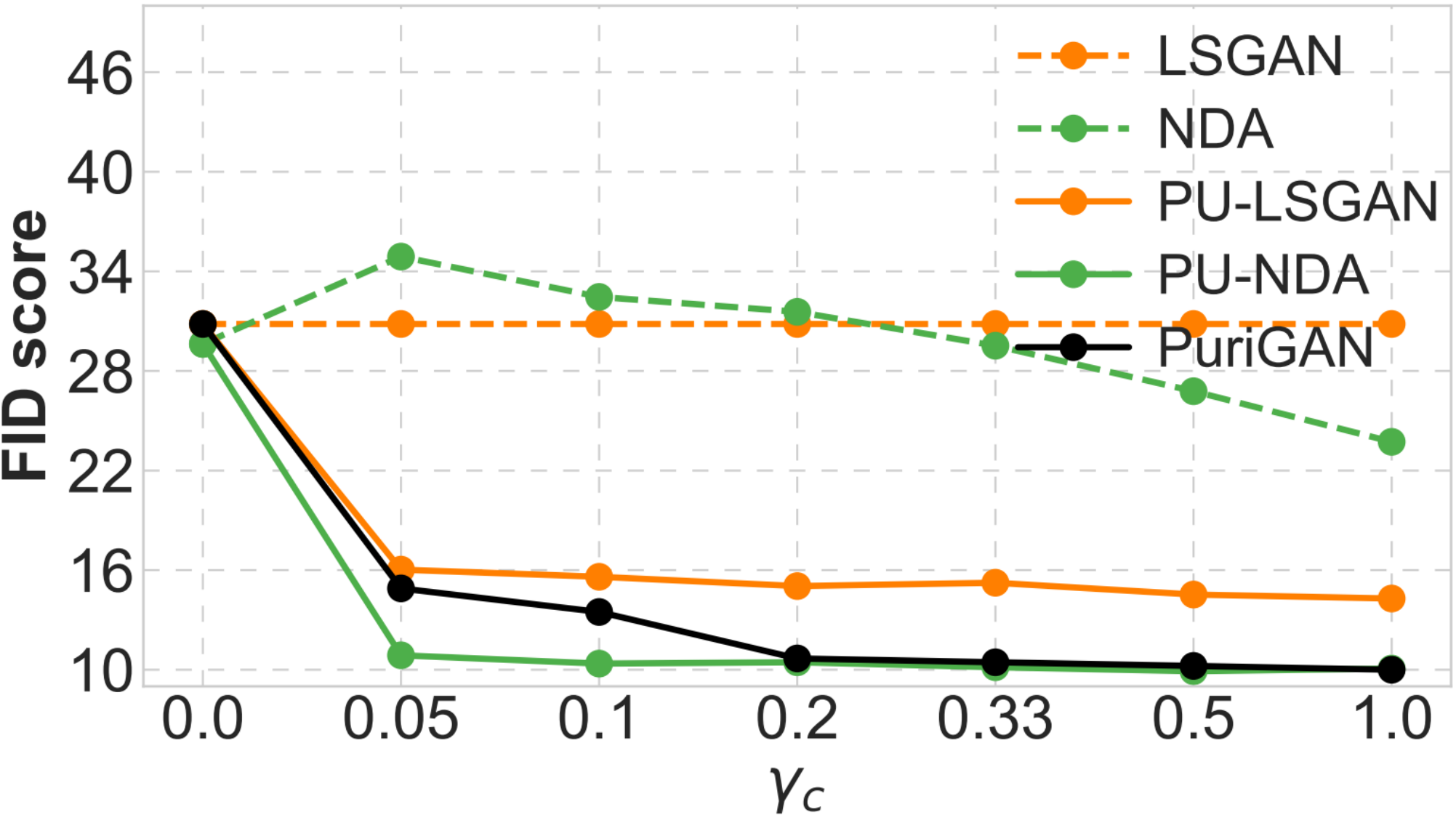}
        \vspace{-6mm}
        \caption*{\scriptsize MNIST}
    \end{minipage}
    \hspace{0.03\linewidth}
    \begin{minipage}{0.2\linewidth}
        \centering
        \includegraphics[width=1\linewidth]{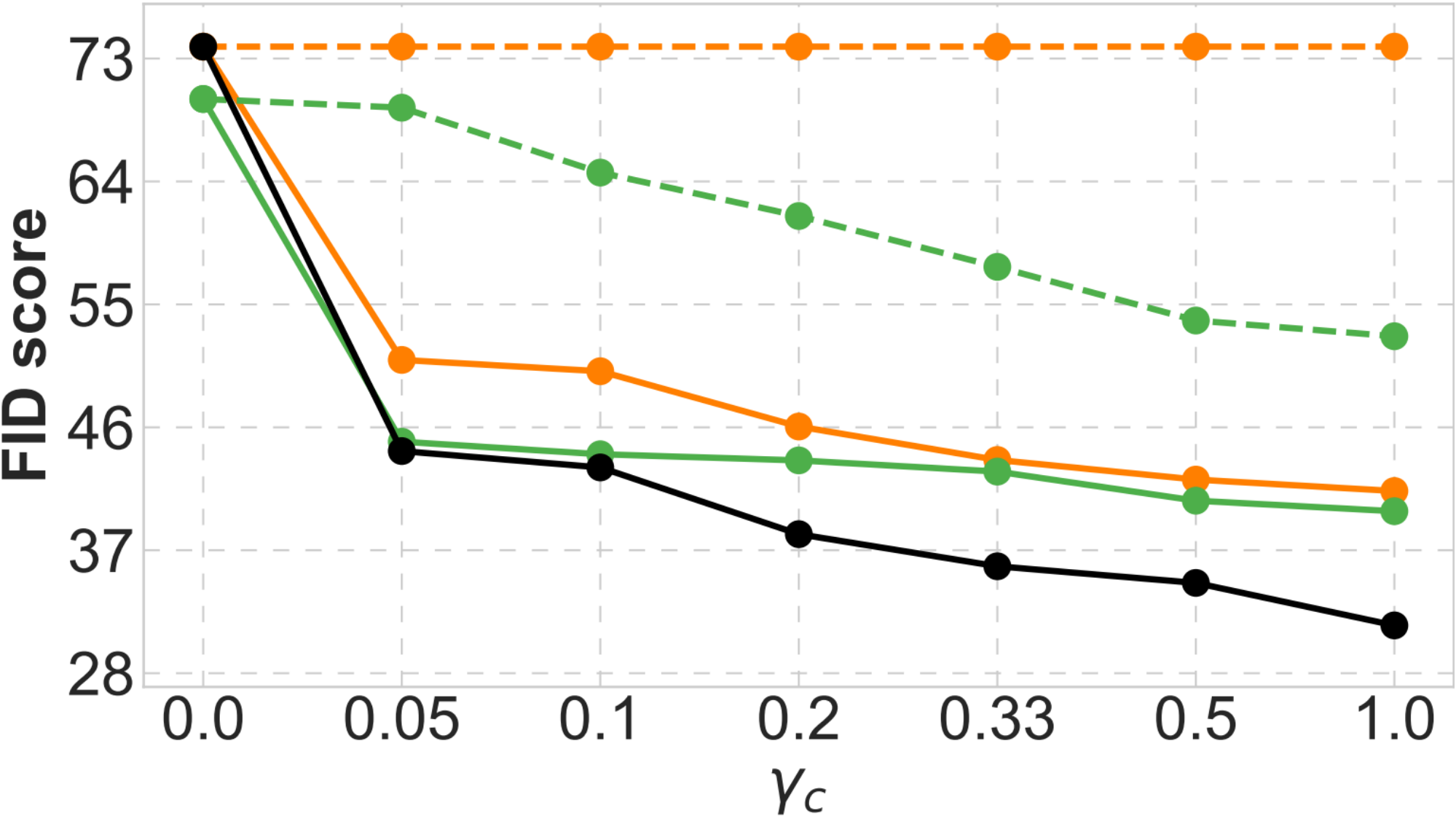}
        \vspace{-6mm}
        \caption*{\scriptsize F-MNIST}
    \end{minipage}
    \hspace{0.03\linewidth}
    \begin{minipage}{0.2\linewidth}
        \centering
        \includegraphics[width=1\linewidth]{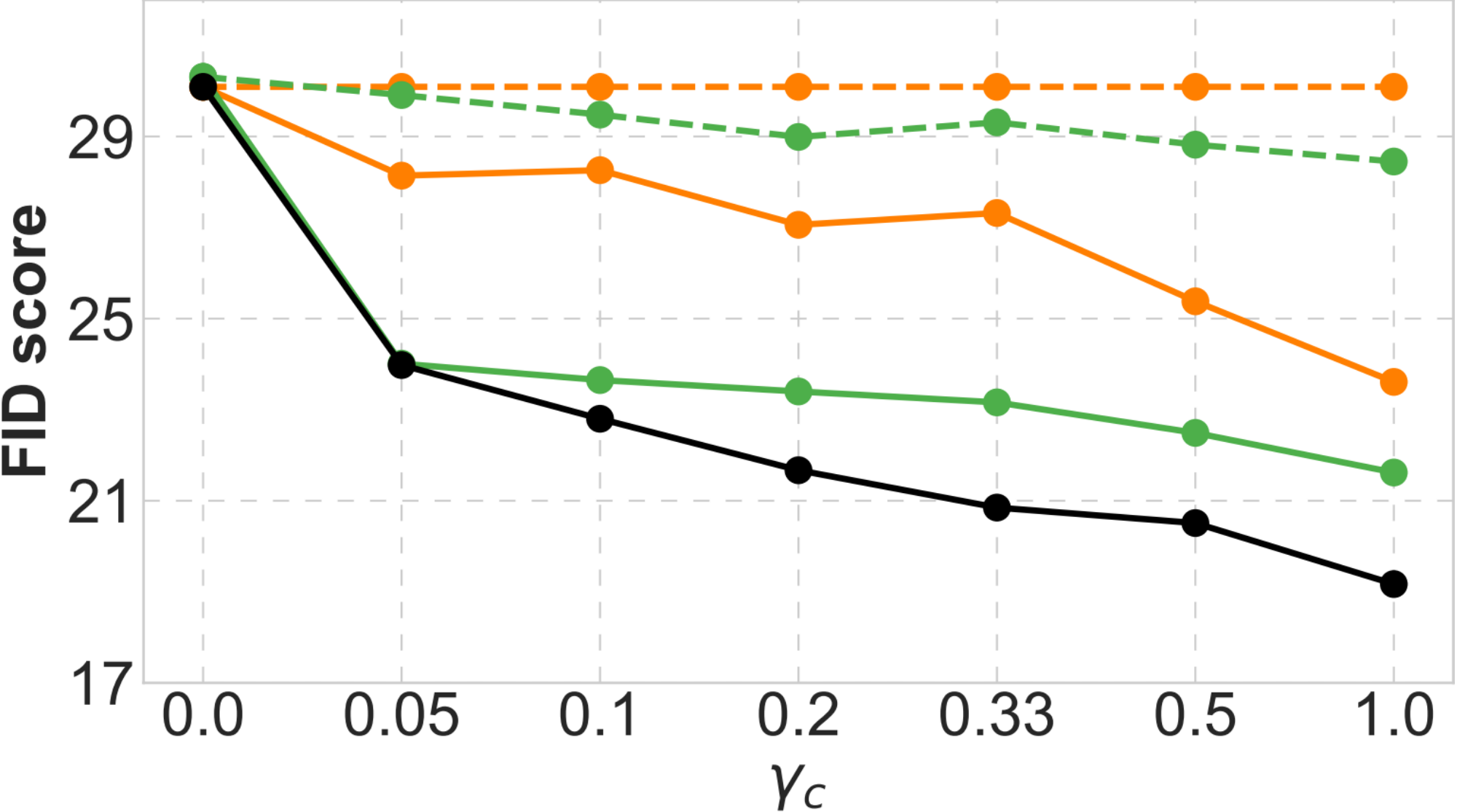}
        \vspace{-6mm}
        \caption*{\scriptsize SVHN}
    \end{minipage}
    \hspace{0.03\linewidth}
    \begin{minipage}{0.2\linewidth}
        \centering
        \includegraphics[width=1\linewidth]{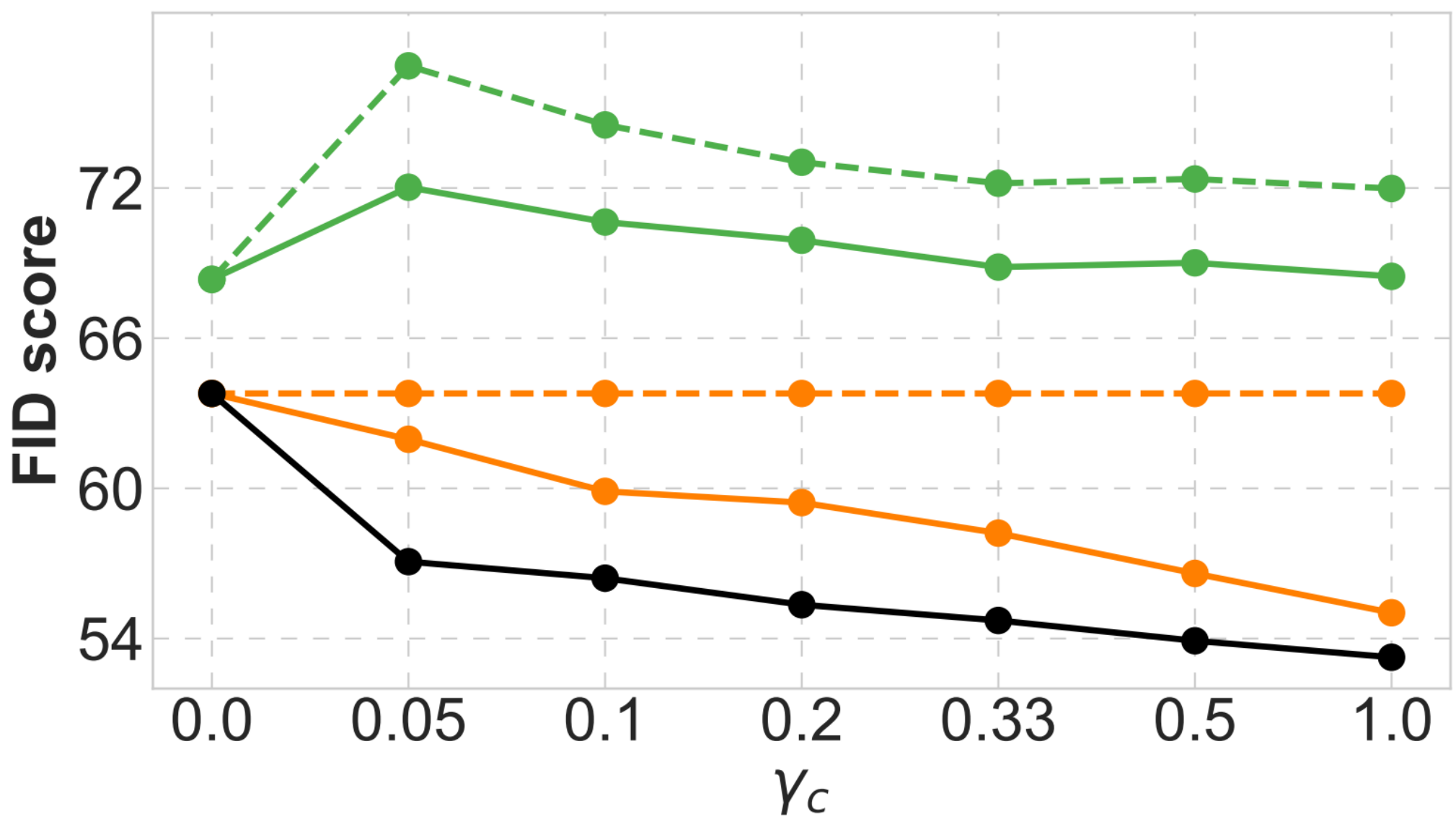}
        \vspace{-6mm}
        \caption*{\scriptsize CIFAR-10}
    \end{minipage}
\vspace{-2mm}
\caption{FID score as a function of $\gamma_c$, the ratio between the collected contamination instances and target instances, under a fixed $\gamma_p=0.5$ .}
\vspace{-1mm}
\label{FID-collected-ratio}
\end{figure*}

\section{Related Work}

Generative adversarial networks are known for their strong capability to generate realistic-looking samples through adversarial training in an unsupervised way \cite{goodfellow2014generative}. 
Since the seminal work of GAN, many kinds of variants of GAN have been proposed to further improve its modeling ability and training stability by using new model architectures \cite{radford2015unsupervised,karras2018progressive,brock2018large,karras2019style}, different distance metrics or divergences \cite{arjovsky2017wasserstein,nowozin2016f,mao2017least} and novel training techniques \cite{miyato2018spectral,wu2021gradient}.
Although these models improve the generation performance or training stability of GANs significantly, they are all established on the assumption that all instances in the training datasets are desirable. When facing with contaminated datasets, they do not have the ability to counter the influences of contamination instances. Two recent works that are relevant to our PuriGAN are the negative data augmentation (NDA) \cite{sinha2020negative} and Rumi-GAN \cite{asokan2020teaching}, which both explicitly teach a GAN what not to learn by leveraging an extra negative dataset that are composed of undesired instances. Although the two models also make use of an extra negative dataset during the training, their primary goal of using negative dataset is to make GANs avoid generating undesired instances due to the strong generalization abilities of GANs. However, they are still established on the assumption that the training dataset is clean and thus cannot deal with the scenarios with contaminated datasets, which, however, is the main focus of this paper.

Another line of works that are relevant to our paper is PU-learning \cite{kiryo2017positive,du2015convex,kato2018learning}, which aim to classify an unlabelled dataset by leveraging an extra dataset solely composed of one class of instances. Recently, some works have proposed to leverage the generation ability of GANs to perform this task. For examples, GenPU in \cite{hou2017generative} proposed to train an array of generators and discriminators to distinguish between the positive and negative instances in the unlabelled datasets. PAN \cite{hu2021predictive} proposed to train the PU-learning classifier under the GAN framework by viewing instances selected by the classifier as the generated instances. However, these works generally focus on how to obtain a better classifier/discriminator rather than how to generate high-quality desired instances.

\vspace{-0.5mm}
\section{Experiments}
\label{experiments}
\subsection{Experimental Setups}

\begin{figure*}[!tb]
    \centering
    \begin{minipage}{0.2\linewidth}
        \centering
        \includegraphics[width=1\linewidth]{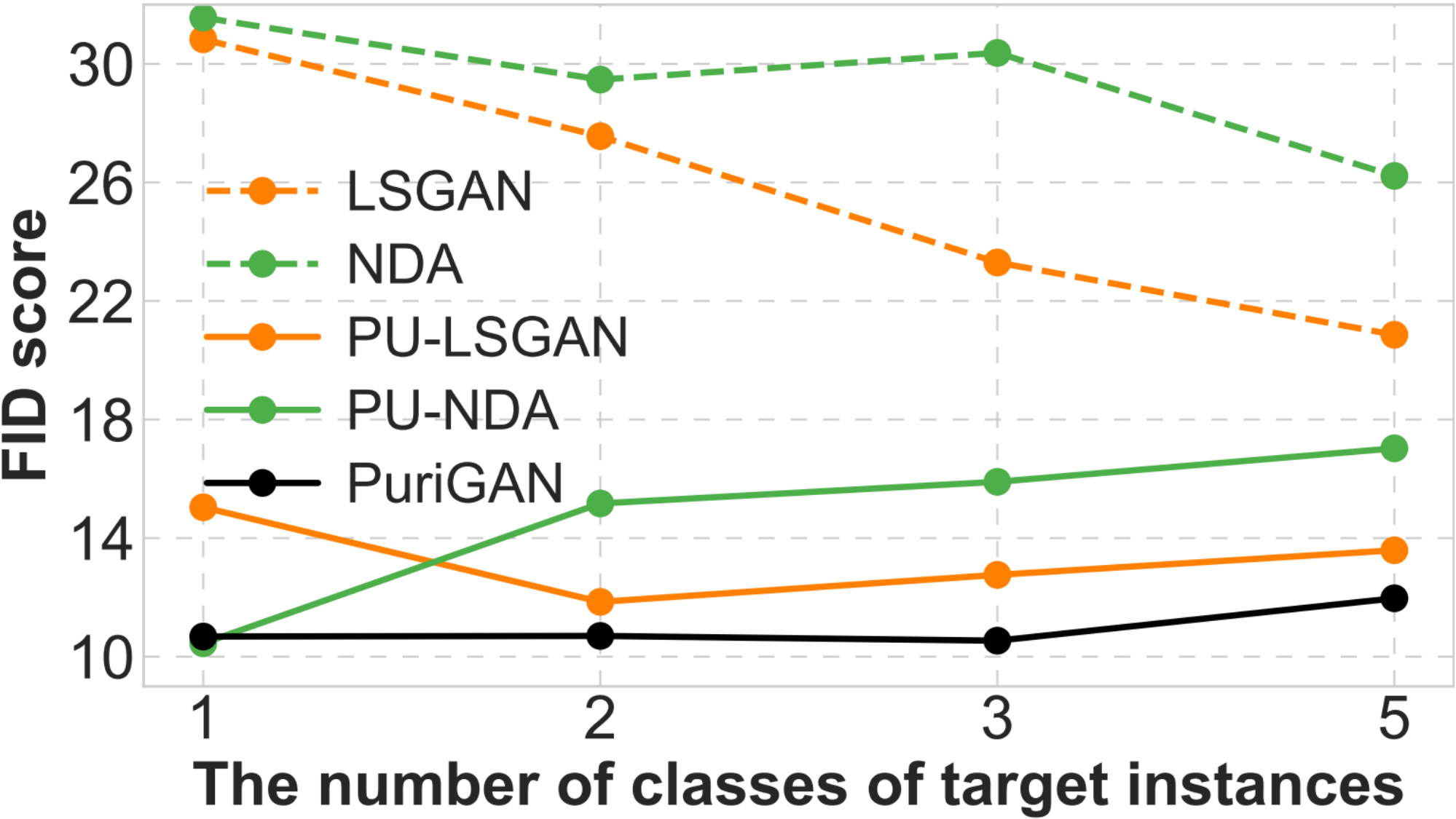}
        \vspace{-6mm}
        \caption*{\scriptsize MNIST}
    \end{minipage}
    \hspace{0.03\linewidth}
    \begin{minipage}{0.2\linewidth}
        \centering
        \includegraphics[width=1\linewidth]{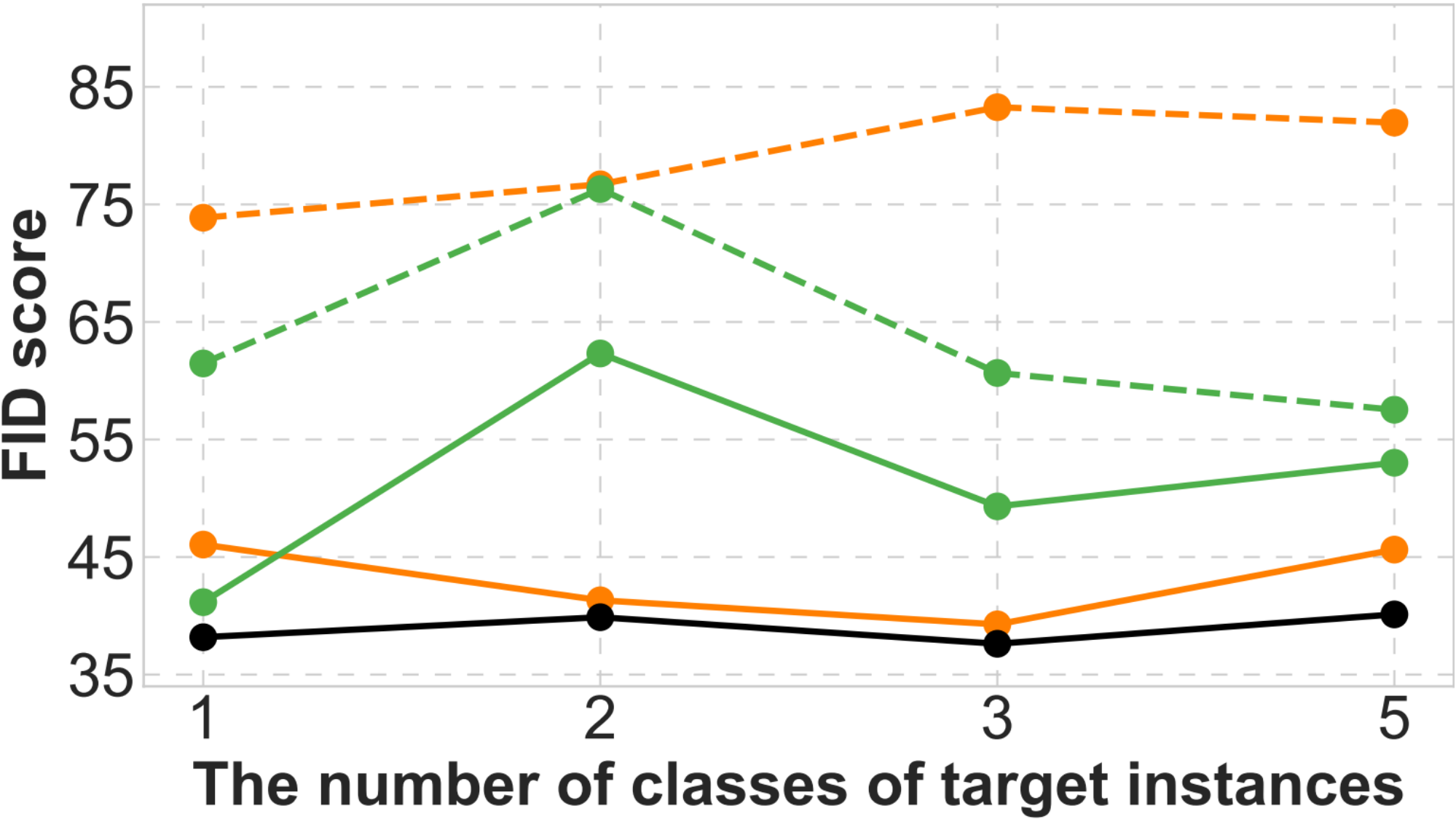}
        \vspace{-6mm}
        \caption*{\scriptsize F-MNIST}
    \end{minipage}
    \hspace{0.03\linewidth}
    \begin{minipage}{0.2\linewidth}
        \centering
        \includegraphics[width=1\linewidth]{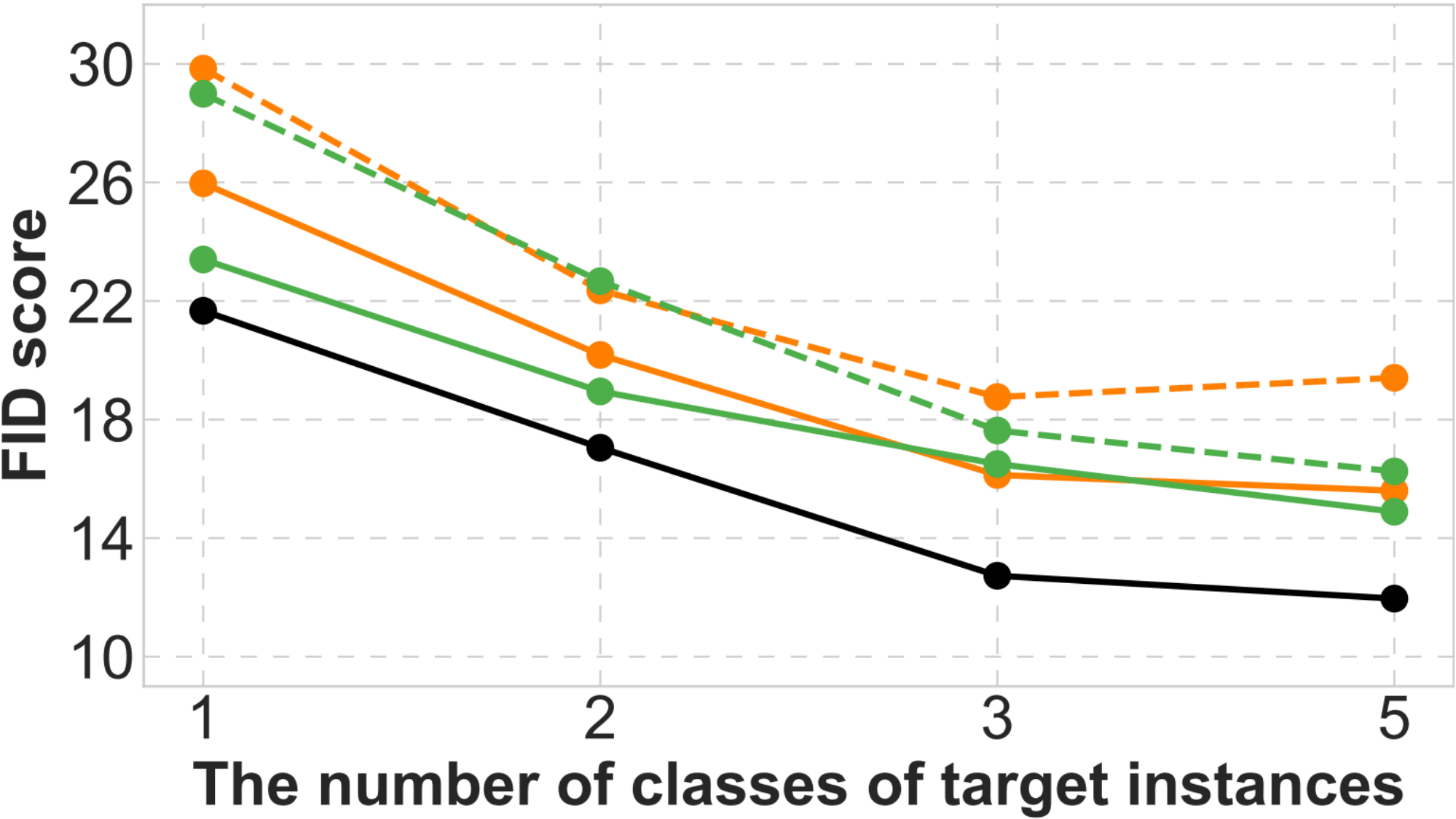}
        \vspace{-6mm}
        \caption*{\scriptsize SVHN}
    \end{minipage}
    \hspace{0.03\linewidth}
    \begin{minipage}{0.2\linewidth}
        \centering
        \includegraphics[width=1\linewidth]{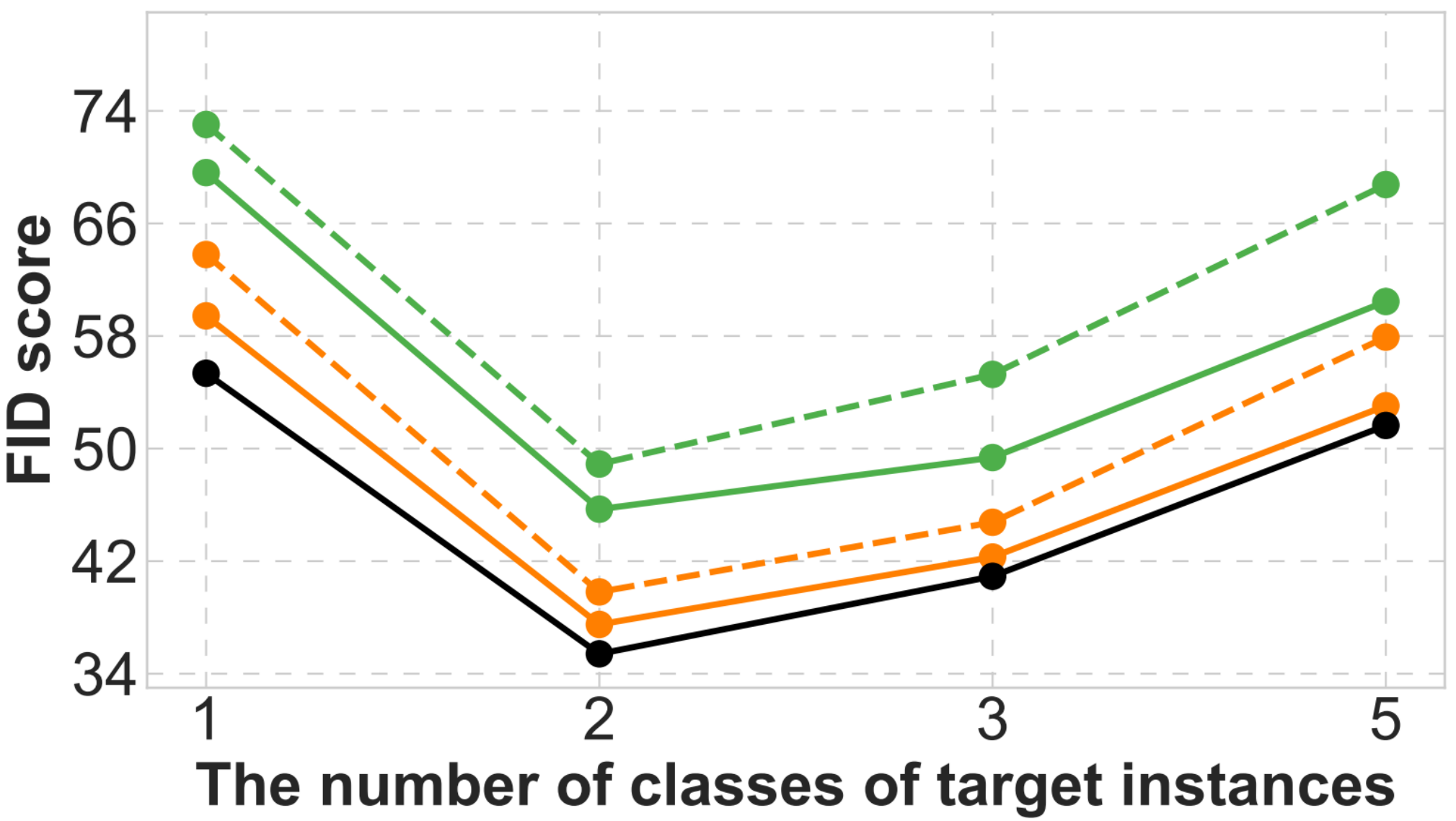}
        \vspace{-6mm}
        \caption*{\scriptsize CIFAR-10}
    \end{minipage}
\vspace{-2mm}
\caption{FID score as a function of the number of classes of target instances under the scenario of $\gamma_p=0.5$ and $\gamma_c=0.2$.}
\label{FID-types-ratio}
\vspace{-2mm}
\end{figure*}

\vspace{-0.5mm}
\paragraph{Evaluation}
To evaluate the generation performance of PuriGAN, for datasets MNIST, F-MNIST, SVHN and CIFAR-10, we randomly select one category and view its instances as the desired instances, while viewing a proportion of instances from another five categories randomly selected from the remaining nine as contamination instances. The two types of instances constitute the final contaminated datasets for training. For dataset CelebA, since it does not have a label, it is partitioned into two subsets according to its attribute value `bald', with images from each subset viewed as desired and contaminated instances, respectively. For each contaminated dataset, the number of desired instances is fixed, while the number of contamination instances is controlled by the contamination ratio $\gamma_p$, which is defined as the ratio between the number of contamination instances and total instances in the training dataset. On the other hand, the number of available contamination instances is controlled by parameter $\gamma_c$, which is defined as the ratio between the number of available contamination instances and total instances 
in the training dataset. The trained models are evaluated on the desired instances in testing dataset with the widely used criteria of Fréchet inception distance (FID) \citep{heusel2017gans}, which is computed by following the protocol in \cite{asokan2020teaching}. For each dataset, we repeat the random selections and training processes for ten times and the averaged results are reported as the final performance. 

\paragraph{Baselines}
We compare PuriGAN with the following baselines. 1) \emph{LSGAN} \cite{mao2017least}: it is developed for clean datasets and is not able to leverage the collected contamination instances; 2) \emph{NDA} \cite{sinha2020negative}: it is able to leverage the collected contamination instances to boost generation quality, but it is also developed to only work on clean datasets; 3) \emph{GenPU} \cite{hou2017generative}: it is a GAN-based PU-learning method that is partially able to work on contaminated datasets. 4) \emph{Rumi-LSGAN} \cite{asokan2020teaching}: It can also leverage the contamination instances but requires the dataset ${\mathcal{X}}$ to be clean. 5)  \emph{PU-LSGAN}: it is a two-stage training method by combining PU-learning and LSGAN, in which the PU-learning employs the recently proposed nnPU method \cite{kiryo2017positive}; 6) \emph{PU-NDA}: it is a combination of PU-learning and NDA;

\subsection{Performance on Image Generation}
\label{main_experiment}

Table \ref{Overall_fidscore} presents the FID scores of the proposed PuriGANs and comparable models on four datasets under the specific contamination ratio $\gamma_p=0.4$ and available contamination ratio $\gamma_c=0.2$. From the table, it can be observed that the proposed PuriGANs perform significantly better than LSGAN and NDA, which demonstrates that the proposed PuriGANs are effective to counter the influence of contamination instances in the training datasets comparing to traditional GANs. Moreover, comparing to the two-stage methods that employ PU-learning to remove contamination instances first, we can see that these two-stage methods still lag behind our proposed PuriGANs, especially on relatively complex datasets F-MNIST and CIFAR-10. This is because PU-learning tends to perform well on simple datasets, but unsatisfactorily on complex ones, resulting in relatively small improvements on complex datasets. 

\begin{figure*}[!tb]
    \centering
    \begin{minipage}{0.25\linewidth}
        \centering
        \includegraphics[width=1\linewidth]{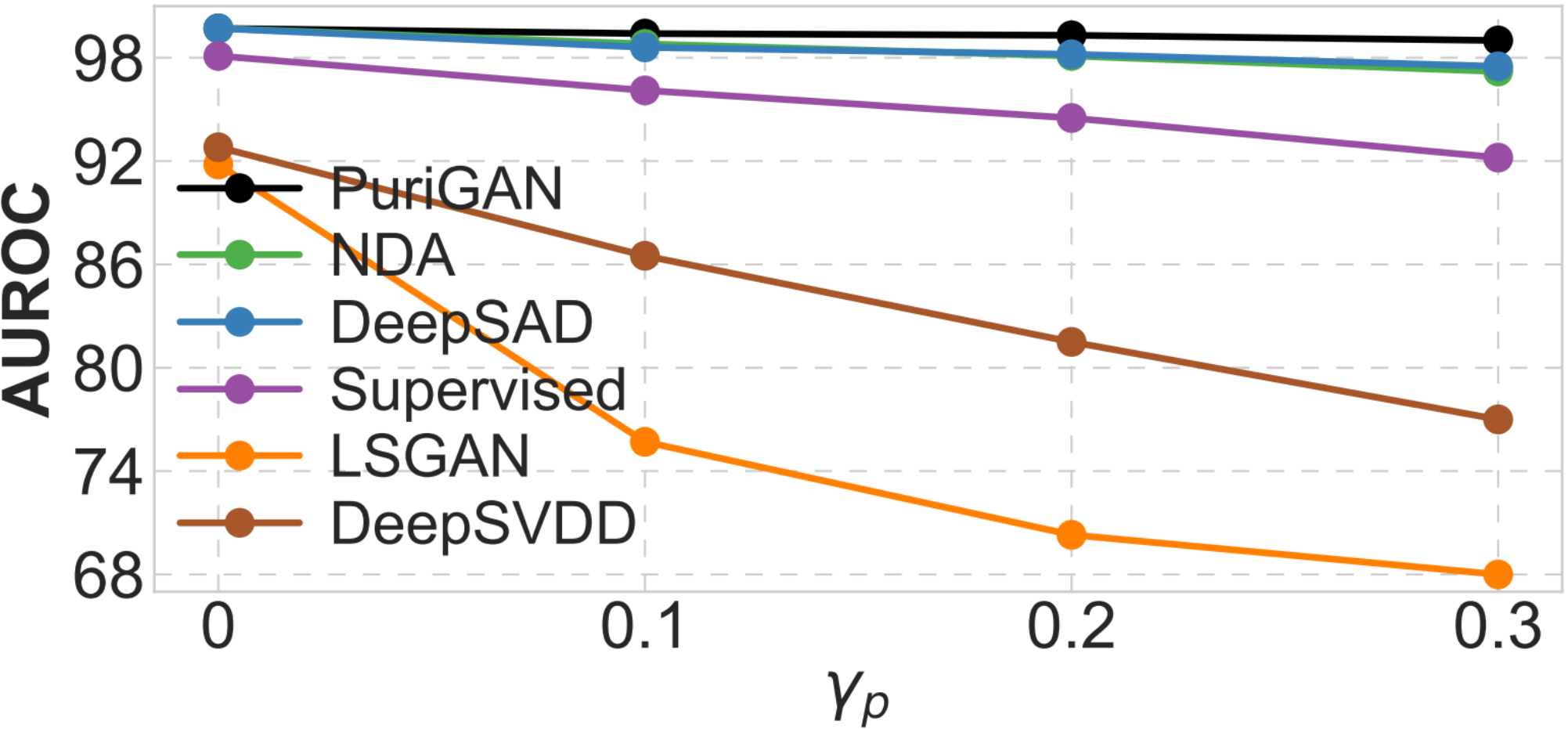}
        \vspace{-6mm}
        \caption*{\scriptsize MNIST}
    \end{minipage}
    \hspace{0.06\linewidth}
    \begin{minipage}{0.25\linewidth}
        \centering
        \includegraphics[width=1\linewidth]{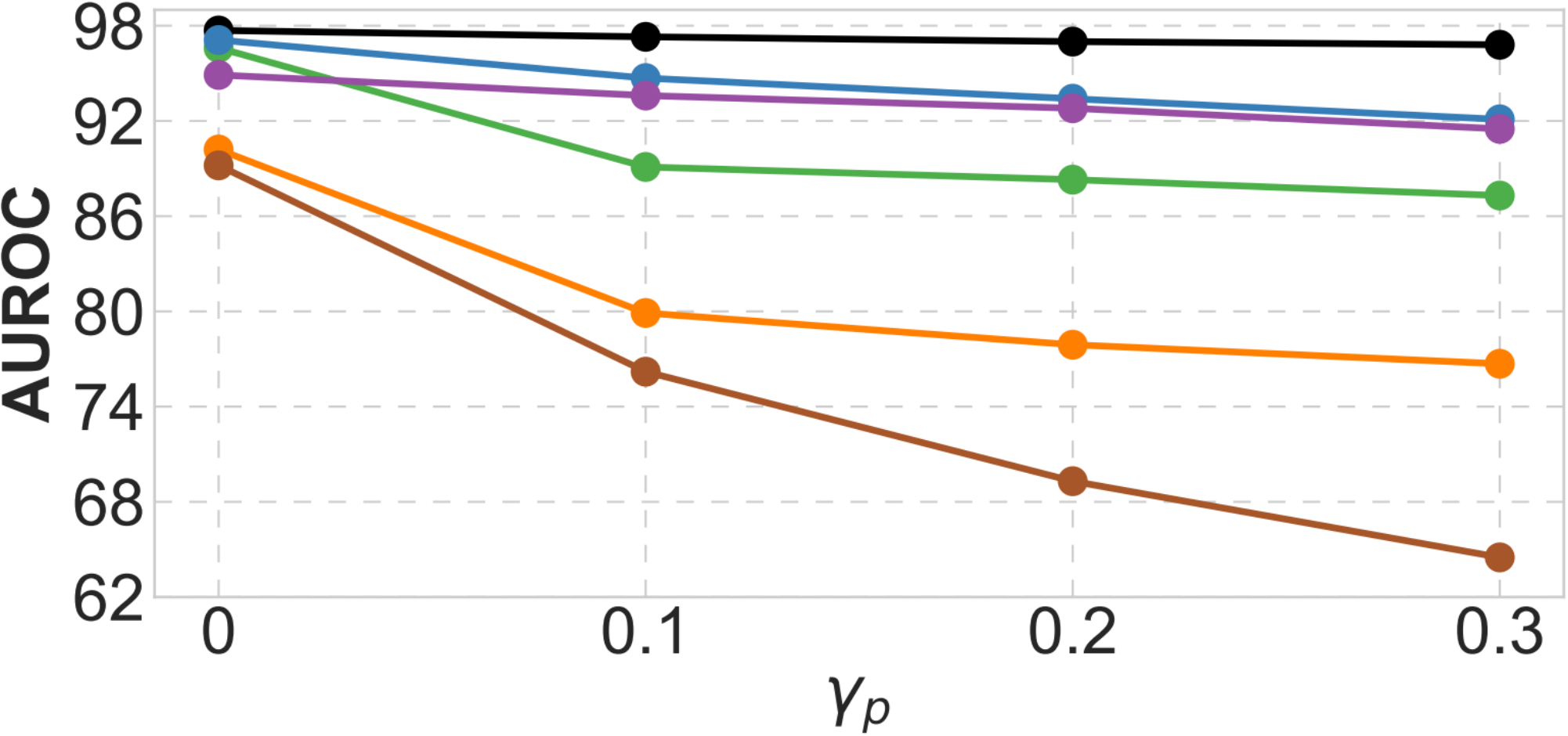}
        \vspace{-6mm}
        \caption*{\scriptsize F-MNIST}
    \end{minipage}
    \hspace{0.06\linewidth}
    \begin{minipage}{0.25\linewidth}
        \centering
        \includegraphics[width=1\linewidth]{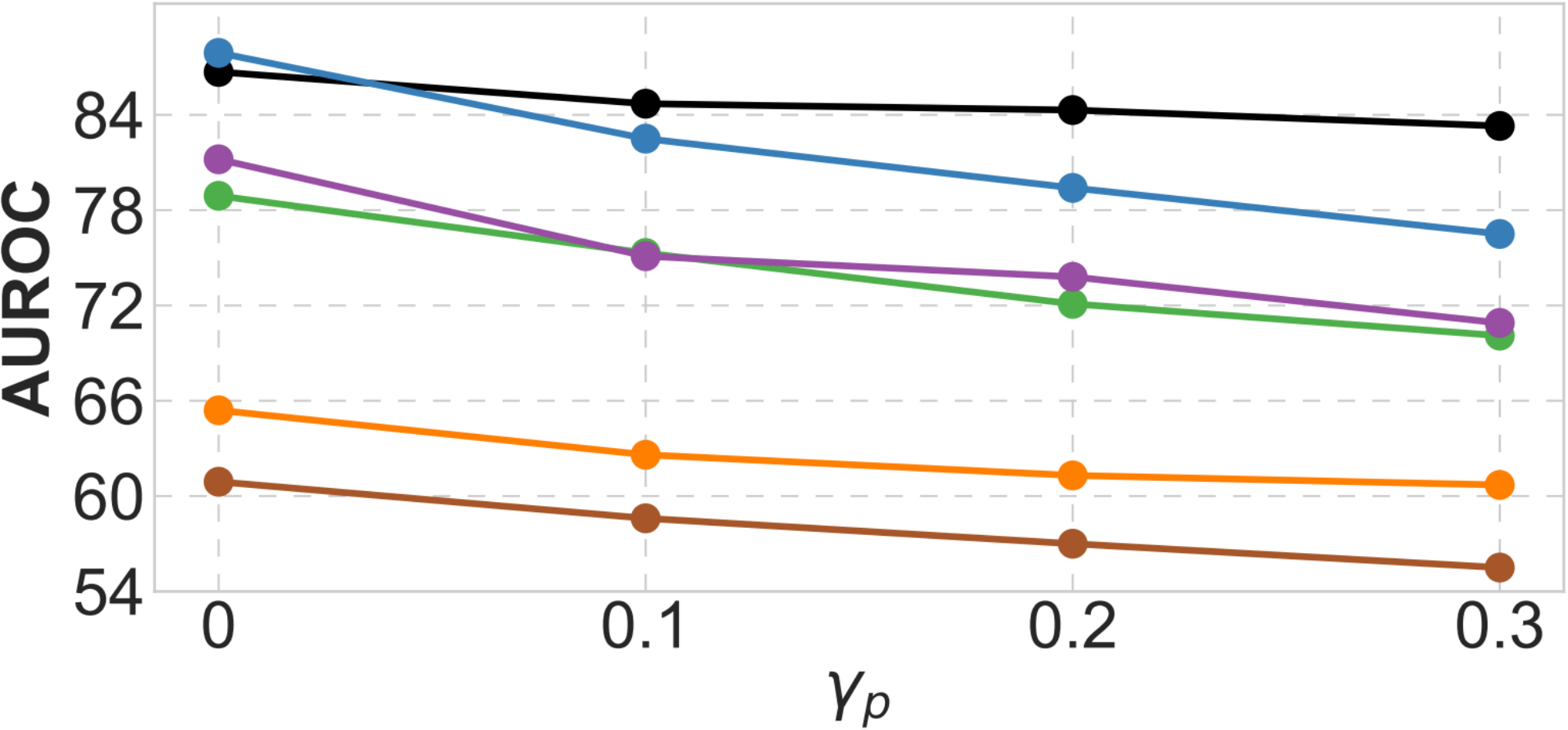}
        \vspace{-6mm}
        \caption*{\scriptsize CIFAR-10}
    \end{minipage}
\vspace{-2mm}
\caption{AUROC of semi-supervised anomaly detection under different contamination ratio of $\gamma_p$ and fixed $\gamma_c=0.05$.}
\label{tablep}
\end{figure*}

\paragraph{Impact of Contamination Ratio $\gamma_p$} 
Fig. \ref{FID-contamination-ratio} shows how FID scores vary as a function of the contamination ratio $\gamma_p$ when the available contamination ratio $\gamma_c$ is fixed to $0.2$. It can be seen that as $\gamma_p$ increases, the FID scores for LSGAN and NDA steadily increases on all four datasets, too, which implies the decrease of generated images' quality. This is easy to understand because the two GANs do not have any ability to counter the influences of contamination instances. By incorporating the PU-learning, we can see that the FID scores of both PU-LSGAN and PU-NDA increase much slowly, which indicates the effectiveness of the two-stage training methods in countering the influence of contamination. Comparing the two-stage methods to our PuriGANs, we can see that our models still yield better FID scores, which may be partially attributed to the joint consideration of purification and generation in our proposed PuriGAN. Lastly, we can see from Fig. \ref{FID-contamination-ratio} that PuriGAN with three-level discriminator overall performs better than PuriGAN with two-level discriminator. This is in consistent with our expectation because PuriGAN with three-level discriminator does not rely on the subtle disjoint support condition. In addition, we also see that the two PuriGANs have the same FID score at $\gamma_p=0.5$. That is because the updating rules of the two models become the same at this special case. 

\paragraph{Impact of Available Contamination Ratio $\gamma_c$}
\label{collected_anomalies}
Fig. \ref{FID-collected-ratio} shows how FID scores vary as a function of $\gamma_c$ when the contamination ratio $\gamma_p$ is fixed to $0.5$. Due to $\gamma_p=0.5$, the two PuriGANs become equivalent, thus we only illustrate the performance of one PuriGAN. From Fig. \ref{FID-collected-ratio}, it can be seen that the performance of PuriGAN can be steadily improved as more contamination instances are collected on all considered datasets, demonstrating the effectiveness of PuriGAN in leveraging the extra contamination instances to counter the influences of contamination. Although the two-stage methods PU-LSGAN and PU-NDA can also benefit from more available contamination instances, they are not as effective as the proposed PuriGAN on relatively complex datasets F-MNIST and CIFAR-10.  

\paragraph{Impact of the Number of Classes of Target Instances}
\label{normal_cls}
In the previous experiments, we only use instances from one class as the target instances. In this section, we investigate how the generation performance is affected when the target instances are composed of more classes of instances. Fig. \ref{FID-types-ratio} shows how the FID scores vary as the number of classes of target instances under the scenario of $\gamma_p=0.5$ and $\gamma_c=0.2$. It can be seen from Fig. \ref{FID-types-ratio} that the proposed PuriGAN still performs the best among all comparable baselines under the scenario with more classes of target instances.

\subsection{Performance on Downstream Tasks}

\paragraph{Anomaly Detection with Contaminated Datasets} For evaluation, we select instances from one category and a proportion of instances from the remaining categories to construct a contaminated dataset ${\mathcal{X}}$. Moreover, we also assume the availability of another dataset ${\mathcal{X}}^-$ that is only composed of instances from the remaining categories. The goal of this task is to detect the anomalies, which are those belonging to the remaining categories, by leveraging a model trained on ${\mathcal{X}}$ and ${\mathcal{X}}^-$. As discussed in Applications section, the output value of PuriGAN's discriminator can be used to detect anomalies. To better evaluate PuriGAN, we compare it with several representative unsupervised methods, \emph{Deep SVDD} \cite{ruff2018deep},  \emph{LSGAN} \cite{mao2017least}, and semi-supervised methods, the recently developed \emph{Deep SAD} \cite{DeepSAD} and \emph{NDA} \cite{sinha2020negative}, which can leverage the collected dataset ${\mathcal{X}}^-$ for detection.  In addition, we also train a binary classifier by treating instances from ${\mathcal{X}}$ and ${\mathcal{X}}^-$ as 1 and 0, respectively, and then use it to detect anomalies. The area under receiver operating characteristic curve (AUROC) is used as the evaluation criteria. For each setting, the experiments were run ten times, and their average is reported as the final performance. From Figure \ref{tablep}, it can be observed that the PuriGAN-based method has a very stable performance across different ratios of contamination in the training dataset. By contrast, the performances of all compared unsupervised and semi-supervised methods deteriorate steadily as $\gamma_p$ increases. This further corroborates the advantages of our proposed PuriGAN in countering the influences of contamination in the training datasets.

\textbf{\begin{table}[!t]
\centering
\setlength\tabcolsep{2.2pt} 
\renewcommand\arraystretch{0.95}
{
\begin{tabular}{l|cc|cc|cc|cc} 
\toprule
\multirow{2}{*}{Dataset} & \multicolumn{2}{c|}{20News} & \multicolumn{2}{c|}{IMDB} & \multicolumn{2}{c|}{MNIST} & \multicolumn{2}{c}{CIFAR-10}  \\ 
\cmidrule{2-9}
                         & F1    & Acc                & F1    & Acc              & F1    & Acc               & F1    & Acc                 \\ 
\midrule
a-GAN                   & 63.5 & 68.7              & 73.0 & 70.6            & 94.7 & 95.0            & 76.2  & 83.1  \\
UPU                     & 59.1 & 53.0             & 70.4 & 69.9            & 94.2 & 94.3             & 86.2 & 89.0        \\
nnPU                    & 78.5 & 78.1              & 76.2 & 74.6            & 95.4 & 95.4             & 86.1 & 88.4              \\
nnPUSB                 & 75.9 & 75.6              & 74.2 & 71.9            & 95.6 & 95.6             & 86.6 & 88.6   \\
PAN                 & 81.1 & 81.1              & 77.1 & 78.8            & 96.5 & 96.4             & 87.2 & 89.7                  \\
PuriGAN                 & \textbf{85.7} & \textbf{84.6}              & \textbf{79.7} & \textbf{78.9}            & \textbf{96.8}& \textbf{96.9} & \textbf{88.7} & \textbf{90.9}    \\
\bottomrule
\end{tabular}}
\caption{F1 score and accuracy of different PU-learning methods.}
\label{pu-results}
\end{table}
}

 
\paragraph{PU-learning}
Following the setups in the paper PAN \cite{hu2021predictive}, the proposed PuriGAN is evaluated on two text and two image datasets on this task. F1-score and accuracy are employed as the performance criteria. From Table \ref{pu-results}, it can be seen that PuriGAN outperforms all baselines on the considered datasets. This is probably because the generator in PuriGAN approximately plays a role of data augmentation by generating new training samples continuously, thereby leading to a more competitive performance comparing with traditional PU-learning methods that directly train a classifier.

\section{Conclusion}
In this paper, we studied the problem of how to train GANs to only generate target instances when a contaminated training dataset is presented. To this end, with the introduction of another extra dataset composed of only contamination instances, a purified generative adversarial network framework (PuriGAN) is proposed, which is achieved by augmenting the discriminator in traditional GANs to endow it with the ability to distinguish between the desired and undesired instances. We prove that the proposed PuriGANs are guaranteed to converge to the target distribution under some mild conditions. Extensive experiments are conducted to demonstrate the superior performance of the proposed PuriGANs in generating images only from desired categories. Moreover, we also apply it to the downstream tasks of semi-supervised anomaly detection and PU-learning, which shows that PuriGAN can deliver the best performance over comparable baselines on both tasks.

\section*{Acknowledgements}
This work is supported by the National Natural Science Foundation of China (No. 62276280, U1811264, 62276279), Key R\&D Program of Guangdong Province (No. 2018B010107005), Natural Science Foundation of Guangdong Province (No. 2021A1515012299), Science and Technology Program of Guangzhou (No. 202102021205), and CAAI-Huawei MindSpore Open Fund.

{\small
\bibliography{aaai23}

\begin{thebibliography}{36}
\providecommand{\natexlab}[1]{#1}

\bibitem[{Akcay, Atapour-Abarghouei, and Breckon(2018)}]{akcay2018ganomaly}
Akcay, S.; Atapour-Abarghouei, A.; and Breckon, T.~P. 2018.
\newblock Ganomaly: Semi-supervised anomaly detection via adversarial training.
\newblock In \emph{Asian conference on computer vision}, 622--637.

\bibitem[{Arjovsky, Chintala, and Bottou(2017)}]{arjovsky2017wasserstein}
Arjovsky, M.; Chintala, S.; and Bottou, L. 2017.
\newblock Wasserstein generative adversarial networks.
\newblock In \emph{International conference on machine learning}, 214--223.

\bibitem[{Asokan and Seelamantula(2020)}]{asokan2020teaching}
Asokan, S.; and Seelamantula, C. 2020.
\newblock Teaching a gan what not to learn.
\newblock \emph{Advances in Neural Information Processing Systems}, 33:
  3964--3975.

\bibitem[{Bekker and Davis(2020)}]{bekker2020learning}
Bekker, J.; and Davis, J. 2020.
\newblock Learning from positive and unlabeled data: A survey.
\newblock \emph{Machine Learning}, 109(4): 719--760.

\bibitem[{Brock, Donahue, and Simonyan(2018)}]{brock2018large}
Brock, A.; Donahue, J.; and Simonyan, K. 2018.
\newblock Large Scale GAN Training for High Fidelity Natural Image Synthesis.
\newblock In \emph{International Conference on Learning Representations}.

\bibitem[{Christoffel, Niu, and Sugiyama(2016)}]{christoffel2016class}
Christoffel, M.; Niu, G.; and Sugiyama, M. 2016.
\newblock Class-prior estimation for learning from positive and unlabeled data.
\newblock In \emph{Asian Conference on Machine Learning}, 221--236.

\bibitem[{Du~Plessis, Niu, and Sugiyama(2015)}]{du2015convex}
Du~Plessis, M.; Niu, G.; and Sugiyama, M. 2015.
\newblock Convex formulation for learning from positive and unlabeled data.
\newblock In \emph{International conference on machine learning}, 1386--1394.

\bibitem[{Elkan and Noto(2008)}]{elkan2008learning}
Elkan, C.; and Noto, K. 2008.
\newblock Learning classifiers from only positive and unlabeled data.
\newblock In \emph{Proceedings of the 14th ACM SIGKDD international conference
  on Knowledge discovery and data mining}, 213--220.

\bibitem[{Goodfellow et~al.(2014)Goodfellow, Pouget-Abadie, Mirza, Xu,
  Warde-Farley, Ozair, Courville, and Bengio}]{goodfellow2014generative}
Goodfellow, I.; Pouget-Abadie, J.; Mirza, M.; Xu, B.; Warde-Farley, D.; Ozair,
  S.; Courville, A.; and Bengio, Y. 2014.
\newblock Generative adversarial nets.
\newblock \emph{Advances in neural information processing systems}, 27.

\bibitem[{Heusel et~al.(2017)Heusel, Ramsauer, Unterthiner, Nessler, and
  Hochreiter}]{heusel2017gans}
Heusel, M.; Ramsauer, H.; Unterthiner, T.; Nessler, B.; and Hochreiter, S.
  2017.
\newblock Gans trained by a two time-scale update rule converge to a local nash
  equilibrium.
\newblock \emph{Advances in neural information processing systems}, 30.

\bibitem[{Hou et~al.(2018)Hou, Chaib-Draa, Li, and Zhao}]{hou2017generative}
Hou, M.; Chaib-Draa, B.; Li, C.; and Zhao, Q. 2018.
\newblock Generative adversarial positive-unlabeled learning.
\newblock In \emph{Proceedings of the 27th International Joint Conference on
  Artificial Intelligence}, 2255--2261.

\bibitem[{Hu et~al.(2021)Hu, Le, Liu, Ji, Ma, Zhao, and Yan}]{hu2021predictive}
Hu, W.; Le, R.; Liu, B.; Ji, F.; Ma, J.; Zhao, D.; and Yan, R. 2021.
\newblock Predictive adversarial learning from positive and unlabeled data.
\newblock In \emph{Proceedings of the AAAI Conference on Artificial
  Intelligence}, volume~35, 7806--7814.

\bibitem[{Isola et~al.(2017)Isola, Zhu, Zhou, and Efros}]{isola2017image}
Isola, P.; Zhu, J.-Y.; Zhou, T.; and Efros, A.~A. 2017.
\newblock Image-to-image translation with conditional adversarial networks.
\newblock In \emph{Proceedings of the IEEE conference on computer vision and
  pattern recognition}, 1125--1134.

\bibitem[{Izadi et~al.(2018)Izadi, Mirikharaji, Kawahara, and
  Hamarneh}]{izadi2018generative}
Izadi, S.; Mirikharaji, Z.; Kawahara, J.; and Hamarneh, G. 2018.
\newblock Generative adversarial networks to segment skin lesions.
\newblock In \emph{2018 IEEE 15th International Symposium on Biomedical Imaging
  (ISBI 2018)}, 881--884. IEEE.

\bibitem[{Jain, White, and Radivojac(2016)}]{jain2016estimating}
Jain, S.; White, M.; and Radivojac, P. 2016.
\newblock Estimating the class prior and posterior from noisy positives and
  unlabeled data.
\newblock \emph{Advances in neural information processing systems}, 29.

\bibitem[{Karras et~al.(2018)Karras, Aila, Laine, and
  Lehtinen}]{karras2018progressive}
Karras, T.; Aila, T.; Laine, S.; and Lehtinen, J. 2018.
\newblock Progressive Growing of GANs for Improved Quality, Stability, and
  Variation.
\newblock In \emph{International Conference on Learning Representations}.

\bibitem[{Karras, Laine, and Aila(2019)}]{karras2019style}
Karras, T.; Laine, S.; and Aila, T. 2019.
\newblock A style-based generator architecture for generative adversarial
  networks.
\newblock In \emph{Proceedings of the IEEE/CVF conference on computer vision
  and pattern recognition}, 4401--4410.

\bibitem[{Kato, Teshima, and Honda(2018)}]{kato2018learning}
Kato, M.; Teshima, T.; and Honda, J. 2018.
\newblock Learning from positive and unlabeled data with a selection bias.
\newblock In \emph{International conference on learning representations}.

\bibitem[{Kazeminia et~al.(2020)Kazeminia, Baur, Kuijper, van Ginneken, Navab,
  Albarqouni, and Mukhopadhyay}]{kazeminia2020gans}
Kazeminia, S.; Baur, C.; Kuijper, A.; van Ginneken, B.; Navab, N.; Albarqouni,
  S.; and Mukhopadhyay, A. 2020.
\newblock GANs for medical image analysis.
\newblock \emph{Artificial Intelligence in Medicine}, 109: 101938.

\bibitem[{Kingma and Welling(2014)}]{VAE2014}
Kingma, D.~P.; and Welling, M. 2014.
\newblock Auto-Encoding Variational Bayes.
\newblock In \emph{International Conference on Learning}.

\bibitem[{Kiryo et~al.(2017)Kiryo, Niu, Du~Plessis, and
  Sugiyama}]{kiryo2017positive}
Kiryo, R.; Niu, G.; Du~Plessis, M.~C.; and Sugiyama, M. 2017.
\newblock Positive-unlabeled learning with non-negative risk estimator.
\newblock \emph{Advances in neural information processing systems}, 30.

\bibitem[{Liu et~al.(2021)Liu, Wan, Huang, Song, Han, and Liao}]{liu2021pd}
Liu, H.; Wan, Z.; Huang, W.; Song, Y.; Han, X.; and Liao, J. 2021.
\newblock Pd-gan: Probabilistic diverse gan for image inpainting.
\newblock In \emph{Proceedings of the IEEE/CVF Conference on Computer Vision
  and Pattern Recognition}, 9371--9381.

\bibitem[{Liu, Breuel, and Kautz(2017)}]{liu2017unsupervised}
Liu, M.-Y.; Breuel, T.; and Kautz, J. 2017.
\newblock Unsupervised image-to-image translation networks.
\newblock \emph{Advances in neural information processing systems}, 30.

\bibitem[{Mao et~al.(2017)Mao, Li, Xie, Lau, Wang, and
  Paul~Smolley}]{mao2017least}
Mao, X.; Li, Q.; Xie, H.; Lau, R.~Y.; Wang, Z.; and Paul~Smolley, S. 2017.
\newblock Least squares generative adversarial networks.
\newblock In \emph{Proceedings of the IEEE international conference on computer
  vision}, 2794--2802.

\bibitem[{Miyato et~al.(2018)Miyato, Kataoka, Koyama, and
  Yoshida}]{miyato2018spectral}
Miyato, T.; Kataoka, T.; Koyama, M.; and Yoshida, Y. 2018.
\newblock Spectral Normalization for Generative Adversarial Networks.
\newblock In \emph{International Conference on Learning Representations}.

\bibitem[{Northcutt, Jiang, and Chuang(2021)}]{northcutt2021confident}
Northcutt, C.; Jiang, L.; and Chuang, I. 2021.
\newblock Confident learning: Estimating uncertainty in dataset labels.
\newblock \emph{Journal of Artificial Intelligence Research}, 70: 1373--1411.

\bibitem[{Nowozin, Cseke, and Tomioka(2016)}]{nowozin2016f}
Nowozin, S.; Cseke, B.; and Tomioka, R. 2016.
\newblock f-gan: Training generative neural samplers using variational
  divergence minimization.
\newblock \emph{Advances in neural information processing systems}, 29.

\bibitem[{Radford, Metz, and Chintala(2016)}]{radford2015unsupervised}
Radford, A.; Metz, L.; and Chintala, S. 2016.
\newblock Unsupervised Representation Learning with Deep Convolutional
  Generative Adversarial Networks.
\newblock In \emph{International Conference on Learning Representations}.

\bibitem[{Ramaswamy, Scott, and Tewari(2016)}]{ramaswamy2016mixture}
Ramaswamy, H.; Scott, C.; and Tewari, A. 2016.
\newblock Mixture proportion estimation via kernel embeddings of distributions.
\newblock In \emph{International conference on machine learning}, 2052--2060.

\bibitem[{Ruff et~al.(2018)Ruff, Vandermeulen, Goernitz, Deecke, Siddiqui,
  Binder, M{\"u}ller, and Kloft}]{ruff2018deep}
Ruff, L.; Vandermeulen, R.; Goernitz, N.; Deecke, L.; Siddiqui, S.~A.; Binder,
  A.; M{\"u}ller, E.; and Kloft, M. 2018.
\newblock Deep one-class classification.
\newblock In \emph{International conference on machine learning}, 4393--4402.

\bibitem[{Ruff et~al.(2019)Ruff, Vandermeulen, G{\"o}rnitz, Binder, M{\"u}ller,
  M{\"u}ller, and Kloft}]{DeepSAD}
Ruff, L.; Vandermeulen, R.~A.; G{\"o}rnitz, N.; Binder, A.; M{\"u}ller, E.;
  M{\"u}ller, K.-R.; and Kloft, M. 2019.
\newblock Deep Semi-Supervised Anomaly Detection.
\newblock In \emph{International Conference on Learning Representations}.

\bibitem[{Schlegl et~al.(2017)Schlegl, Seeb{\"o}ck, Waldstein, Schmidt-Erfurth,
  and Langs}]{AnoGAN17}
Schlegl, T.; Seeb{\"o}ck, P.; Waldstein, S.~M.; Schmidt-Erfurth, U.; and Langs,
  G. 2017.
\newblock Unsupervised anomaly detection with generative adversarial networks
  to guide marker discovery.
\newblock In \emph{IPMI}, 146--157.

\bibitem[{Sinha et~al.(2020)Sinha, Ayush, Song, Uzkent, Jin, and
  Ermon}]{sinha2020negative}
Sinha, A.; Ayush, K.; Song, J.; Uzkent, B.; Jin, H.; and Ermon, S. 2020.
\newblock Negative Data Augmentation.
\newblock In \emph{International Conference on Learning Representations}.

\bibitem[{Wu et~al.(2021)Wu, Shuai, Tam, and Chiu}]{wu2021gradient}
Wu, Y.-L.; Shuai, H.-H.; Tam, Z.-R.; and Chiu, H.-Y. 2021.
\newblock Gradient normalization for generative adversarial networks.
\newblock In \emph{Proceedings of the IEEE/CVF International Conference on
  Computer Vision}, 6373--6382.

\bibitem[{Yu et~al.(2018)Yu, Lin, Yang, Shen, Lu, and Huang}]{yu2018generative}
Yu, J.; Lin, Z.; Yang, J.; Shen, X.; Lu, X.; and Huang, T.~S. 2018.
\newblock Generative image inpainting with contextual attention.
\newblock In \emph{Proceedings of the IEEE conference on computer vision and
  pattern recognition}, 5505--5514.

\bibitem[{Zenati et~al.(2018)Zenati, Romain, Foo, Lecouat, and
  Chandrasekhar}]{zenati2018adversarially}
Zenati, H.; Romain, M.; Foo, C.-S.; Lecouat, B.; and Chandrasekhar, V. 2018.
\newblock Adversarially learned anomaly detection.
\newblock In \emph{2018 IEEE International conference on data mining (ICDM)},
  727--736.

\end{thebibliography}
}

\clearpage
\appendix
\twocolumn[
\begin{@twocolumnfalse}
\section*{\centering{\Large Supplementary Materials of Leveraging Contaminated Datasets to Learn Clean-Data Distribution with Purified Generative Adversarial Networks}}
\vspace{15mm}
\end{@twocolumnfalse}
]

\theoremstyle{definition}
\setcounter{tocdepth}{4}
\setcounter{secnumdepth}{0}
\setcounter{theorem}{0} 

\section{Proof of Theorem 1}
\begin{theorem}
	\label{theorem_disjoint_appendix}
	When $D(\cdot)$ and $G(\cdot)$ are updated according to

\begin{align} 
	\min_D V(D) = & {\mathbb{E}}_{{\mathbf{x}}\sim p_{d}({\mathbf{x}})}\!\!\left[(D({\mathbf{x}}) \!-\! 1)^2\right]\nonumber\\
	&\!+\!  {\mathbb{E}}_{{\mathbf{x}} \sim p_g({\mathbf{x}})}\!\! \left[(D({\mathbf{x}}) \!-\! 0)^2\right] \nonumber \\
	&\!+ \lambda {\mathbb{E}}_{{\mathbf{x}}\sim p^-({\mathbf{x}})}\!\! \left[(D(\mathbf{x}) \!-\! 0)^2\right],
	\tag{\ref{Discrim_pol}}
\end{align}

\begin{align} 
	\min_G V(G) =& {\mathbb{E}}_{{\mathbf{x}}\sim p_{d}({\mathbf{x}})}\!\! \left[(D({\mathbf{x}}) \!-\! c)^2\right]\nonumber \\
	&\!+\! {\mathbb{E}}_{{\mathbf{z}} \sim p({\mathbf{z}})}\!\! \left[(D(G({\mathbf{z}})) \!-\! c)^2\right]\nonumber \\
	&\!+\!   {\mathbb{E}}_{{\mathbf{x}}\sim p^-({\mathbf{x}})}\!\! \left[(D(\mathbf{x}) \!-\! c)^2\right], \tag{\ref{Genera_pol}}
\end{align}

	the optimal discriminator is
	\begin{equation}
		D^*({\mathbf{x}}) = \frac{p_d(\mathbf{x})}{p_d({\mathbf{x}}) + p_g({\mathbf{x}}) + \lambda p^-({\mathbf{x}})};
		\tag{\ref{Opt_Discrim_Disjoint}}
	\end{equation}
Moreover, by supposing the support of target and contamination distributions are disjoint, {\it i.e.}, $Supp(p^+({\mathbf{x}})) \cap Supp(p^-({\mathbf{x}})) = \emptyset$, and letting $\lambda \to + \infty$, the generator distribution $p_g({\mathbf{x}})$ will converge to the target distribution $p^+({\mathbf{x}})$.
\end{theorem}

\begin{proof}
    For convenience, we use $p_d$, $p^+$, $p^-$, $p_g$ to replace the $p_d(\mathbf{x})$, $p^+(\mathbf{x})$, $p^-(\mathbf{x})$, $p_g(\mathbf{x})$. Converting expectations into integral form, \eqref{Discrim_pol} can be rewritten as:
\begin{align}
\begin{split}
    V(D)=&
    \int_{\mathcal{S}}\bigg \{\left(D({\mathbf{x}})-1\right)^{2} p_d(\mathbf{x})+\left(D({\mathbf{x}})-0\right)^{2} p_{g}(\mathbf{x})
     \nonumber \\
     & +\lambda\left(D({\mathbf{x}})-0\right)^{2} p^{-}(\mathbf{x}) \bigg \} \mathrm{d} \boldsymbol{x},
\end{split}
\end{align}
where $\mathcal{S}$ represents the entire integration space. To derive the optimal discriminator, we only need to find a discriminator $D(\mathbf{x})$ that can minimize the integrand in $V(D)$ at every value of $\mathbf{x}$. Thus, by setting the derivative of integrand in $V(D)$ w.r.t. $D(\mathbf{x})$ to $0$, that is, $2\!\left(D({\mathbf{x}})-1\right) p_d(\mathbf{x})+2\left(D({\mathbf{x}})-0\right) p_{g}(\mathbf{x}) +2\lambda\left(D({\mathbf{x}})-0\right) p^{-}(\mathbf{x})=0$, we obtain the optimal discriminator as

\begin{equation}
		D^*({\mathbf{x}}) = \frac{p_d(\mathbf{x})}{p_d({\mathbf{x}}) + p_g({\mathbf{x}}) + \lambda p^-({\mathbf{x}})}.
		\tag{\ref{Opt_Discrim_Disjoint}}
\end{equation}

Because $p_g(\mathbf{x})$ denotes the distribution of generated samples, that is, $G({\mathbf{z}}) \sim p_g({\mathbf{x}})$, \eqref{Genera_pol} can be equivalently written as:
\begin{align}
\begin{split}
    V(G)=&
    \int_{\mathcal{S}}\bigg \{ \left(D({\mathbf{x}})-c\right)^{2} p_{d}(\mathbf{x})
    +\left(D({\mathbf{x}})-c\right)^{2} p_{g}(\mathbf{x})
   \\
   & +\left(D({\mathbf{x}})-c\right)^{2} p^{-}(\mathbf{x})\bigg \} \mathrm{d} \boldsymbol{x}.
    \label{proof1_app:Gobjective1}
\end{split}
\end{align}
Substituting $p_d({\mathbf{x}}) = \pi p^+({\mathbf{x}}) + (1-\pi) p^-({\mathbf{x}})$ into $D^*({\mathbf{x}})$ gives

 \begin{equation}\label{app_Opt_Discrim_Disjoint_p-}
    D^*({\mathbf{x}}) = \frac{\pi p^+({\mathbf{x}}) + (1-\pi)p^-({\mathbf{x}})}{\pi p^+({\mathbf{x}}) + (1-\pi)p^-({\mathbf{x}}) +p_g({\mathbf{x}}) +\lambda p^-({\mathbf{x}})} .
\end{equation}
Then, by substituting $p_d({\mathbf{x}}) = \pi p^+({\mathbf{x}}) + (1-\pi) p^-({\mathbf{x}})$ and the optimal discriminator  \eqref{app_Opt_Discrim_Disjoint_p-} into  \eqref{proof1_app:Gobjective1}, the objective function of generator can be rewritten as:
\begin{align}
\begin{split}
    V(G)&=
    \int_{\mathcal{S}}\!\bigg \{ \pi\left(\frac{\pi p^+ + (1-\pi)p^-}{\pi p^+ + (1-\pi)p^-+\lambda p^- +p_g}-c\right)^{2}\!\!p^{+}
    \\&\quad+\left(1-\pi\right)\left(\frac{\pi p^+ + (1-\pi)p^-}{\pi p^+ + (1-\pi)p^-+\lambda p^- +p_g}-c\right)^{2}\!\! p^{-}
    \\&\quad +\left(\frac{\pi p^+ + (1-\pi)p^-}{\pi p^+ + (1-\pi)p^-+\lambda p^- +p_g}-c\right)^{2}\!\! p_{g} 
    \\&\quad +\left(\frac{\pi p^+ + (1-\pi)p^-}{\pi p^+ + (1-\pi)p^-+\lambda p^- +p_g}-c\right)^{2}\!\! p^{-}\bigg \} \mathrm{d} \boldsymbol{x}.
    \label{proof2_app:GEobjective1}
\end{split}
\end{align}

From the assumption that $p^+(\mathbf{x})$ and $p^-(\mathbf{x})$ are disjoint, we have $p^-(\mathbf{x})=0$ when $p^+(\mathbf{x})>0$, and $p^+(\mathbf{x})=0$ when $p^-(\mathbf{x})>0$. We thus can divide the whole space $\mathcal{S}$ into $\mathcal{S}_1$ and $\mathcal{S}_2$ two subspaces. In $\mathcal{S}_1$, $p^+(\mathbf{x})\ge0$ and $p^-(\mathbf{x})=0$, in $\mathcal{S}_2$, $p^-(\mathbf{x})\ge0$ and $p^+(\mathbf{x})=0$. Based on this assumption, \eqref{proof2_app:GEobjective1} can be further written as: 
\begin{align}
\begin{split}
    &V(G)=\\
    &\int_{\mathcal{S}_1}\bigg \{\!\pi\!\left(\frac{\pi p^+}{\pi p^+ +p_g}-c\right)^{2}\!\! p^{+} 
    +\left(\frac{\pi p^+}{\pi p^+ +p_g}-c\right)^{2}\!\!p_{g} \bigg \}\mathrm{d} \boldsymbol{x}
    \\&\quad+\int_{\mathcal{S}_2}\bigg \{ \left(1-\pi\right)\!\left(\frac{(1-\pi)p^-}{(1-\pi)p^-+\lambda p^- +p_g}-c\right)^{2}\!\!p^{-}
    \\&\quad +\left(\frac{(1-\pi)p^-}{(1-\pi)p^-+\lambda p^- +p_g}-c\right)^{2}\!\!p^{-}
    \\&\quad +\left(\frac{(1-\pi)p^-}{(1-\pi)p^-+\lambda p^- +p_g}-c\right)^{2} \!\!p_{g} \bigg \} \mathrm{d} \boldsymbol{x}.
    \label{proof1_app:Gobjective2}
\end{split}
\end{align}
When $\lambda$ is set to be a very lager number, we can see that $\frac{(1-\pi)p^-(\mathbf{x})}{(1-\pi)p^-(\mathbf{x})+\lambda p^-(\mathbf{x}) +p_g(\mathbf{x})}$ converges to 0. Based on this, \eqref{proof1_app:Gobjective2} can be written as:
\begin{align}
\begin{split}
    V(G)\!=\!&
    \int_{\mathcal{S}_1}\!\!\!\bigg \{ \pi\!\left(\frac{\pi p^+}{\pi p^+\!+\!p_g}\! - \!c\right)^{2}\!\!\! p^{+} 
     \!\!+\!\!\left(\frac{\pi p^+}{\pi p^+ \!+\!p_g}\!-\!c\right)^{2}\!\!\! p_{g} \bigg \}\mathrm{d} \boldsymbol{x}\\
   & \!+\!\int_{\mathcal{S}_2}\!\!\bigg \{ c^2(2-\pi) p^-\!+\!c^2p_g \bigg \} \mathrm{d} \boldsymbol{x}.
    \label{proof1_app:Gobjective3}
\end{split}
\end{align}
Let $\varphi(x)=(x-c)^2$. Because $\varphi(x)$ is a convex function, it has a property that $\varphi \!\left(\int_{-\infty}^{\infty} g(x)f(x) \mathrm{d} \boldsymbol{x}\right)\!\!\le\!\! \int_{-\infty}^{\infty} \varphi\left(g(x)\right)\!f(x)\ \mathrm{d} \boldsymbol{x} $, where $f(x)\ge 0$ and $\int_{-\infty}^{\infty} f(x)  \mathrm{d} \boldsymbol{x}\!=\!1$.
Without loss of generality, we denote  $\int_{\mathcal{S}_1}p_g \mathrm{d} \boldsymbol{x}$ as $\alpha$, then it can be easily seen that $\int_{\mathcal{S}_2}p_g \mathrm{d} \boldsymbol{x} = 1 - \alpha$. Now, from \eqref{proof1_app:Gobjective3}, we can derive the following relation
\begin{align}
\begin{split}
    &V(G)
    \\&=(\pi+\alpha)\!\!\int_{\mathcal{S}_1} \frac{1}{\pi+\alpha} \bigg \{\varphi\!\left(\frac{\pi p^+}{\pi p^+ +p_g}\right) \!\left(\pi p^{+}+ p_g\right) \bigg \}\mathrm{d} \boldsymbol{x}\\
   & \quad+\int_{\mathcal{S}_2}\!\! \bigg \{ c^2(2-\pi) p^-\!+\!c^2p_g  \bigg \} \mathrm{d} \boldsymbol{x}
    \\&\ge (\pi+\alpha) \ \varphi\left(\frac{1}{\pi+\alpha} \int_{\mathcal{S}_1} \pi p^+ \mathrm{d} \boldsymbol{x}\right)
   \\&\quad+\int_{\mathcal{S}_2}\!\! \bigg \{ c^2(2-\pi) p^-\!+\!c^2p_g  \bigg \} \mathrm{d} \boldsymbol{x}
    \\&=  (\pi+\alpha) \ \varphi\left(\frac{\pi}{\pi+\alpha}\right)+c^2(1-\alpha) + c^2 \left(2-\pi\right)
    \\&= c\left[(\pi+1)c-2\pi\right]+\frac{\pi^2}{\pi+\alpha}+c^2 \left(2-\pi\right). \label{proof1_app:Gobjective6}
\end{split}
\end{align}
Obviously, \eqref{proof1_app:Gobjective6} is monotonically decreasing function w.r.t. $\alpha \in [0, 1]$ and its minimum value can be obtained as $\alpha=1$, with the minimum value equal to

\begin{align}
    (1+\pi)\ \varphi\left(\frac{\pi}{1+\pi}\right) +c^2 \left(2-\pi\right) .
    \label{proof1_app:optimalg1}
\end{align}

On the other hand, if we substitute $p_g=p^+$ into \eqref{proof1_app:Gobjective3}, we obtain
\begin{align}
\begin{split}
    \widetilde V(G)
    &=(1+\pi)\! \int_{\mathcal{S}_1} \!\! \frac{1}{1+\pi} \bigg \{ \varphi\!\left(\frac{\pi p^+}{\pi p^+ \!+\!p_g}\right) \!\!\left(\pi p^{+}\!+\!p_g\right)\!\!\bigg \}\mathrm{d} \boldsymbol{x}\!
    \\
    &\quad+\!\int_{\mathcal{S}_2}\!\!\bigg\{c^2(2\!-\!\pi) p^-\!+\!c^2p_g\bigg \} \mathrm{d} \boldsymbol{x}
    \\&= (1+\pi)\!\int_{\mathcal{S}_1}\!\!\frac{1}{1+\pi} \bigg \{ \varphi\!\left(\frac{\pi p^+}{\pi p^+ \!+\!p_g}\right)\!\! \left(\pi p^{+}\!+\!p_g\right) \!\!\bigg \}\mathrm{d} \boldsymbol{x} \!
   \\& \quad+\!\int_{\mathcal{S}_2}\!\!\bigg\{c^2(2\!-\!\pi) p^-\!+\!c^2\!\cdot0\!\bigg \}  \mathrm{d} \boldsymbol{x}
    \\&= (1+\pi)\ \varphi\!\left(\frac{\pi}{1+\pi}\right) \!+\!c^2(2-\pi).
    \label{proof1_app:optimal_g2}
\end{split}
\end{align}
From \eqref{proof1_app:optimalg1}, it is known that $(1+\pi) \ \varphi\left(\frac{\pi}{1+\pi}\right) +c^2 \left(2-\pi\right)$ is the minimum value of $V(G)$. From \eqref{proof1_app:optimal_g2}, we can see that $V(G)$ can attain this minimum value by setting $p_g(\mathbf{x}) = p^+(\mathbf{x})$. Therefore, we can conclude that the minimum value of $V(G)$ can be obtained at $p_g(\mathbf{x})=p^+(\mathbf{x})$.
\end{proof}

\section{Proof of Theorem 2}
\begin{theorem}
	\label{theorem_contaminated_appendix}
	When $D(\cdot)$ and $G(\cdot)$ are updated according to
	\begin{align}
	\min_D V(D)  = & {\mathbb{E}}_{{\mathbf{x}}\sim p_d({\mathbf{x}})}\!\left[(D({\mathbf{x}}) - 1)^2\right] \nonumber \\
	& +  {\mathbb{E}}_{\mathbf{x} \sim p_g({\mathbf{x}})}\!\left[(D({\mathbf{x}})-0)^2\right]\nonumber \\
	& +  {\mathbb{E}}_{{\mathbf{x}} \sim p^-({\mathbf{x}})} \! \left[(D({\mathbf{x}})- d)^2\right] ,
    \tag{\ref{Discrim_Relax}}
\end{align}

and 
\begin{align}
	\min_G V(G) =& {\mathbb{E}}_{{\mathbf{x}}\sim p_{d}({\mathbf{x}})}\!\! \left[(D({\mathbf{x}}) \!-\! c)^2\right]\nonumber  \!  \\
	& +\! {\mathbb{E}}_{{\mathbf{z}} \sim p({\mathbf{z}})}\!\! \left[(D(G({\mathbf{z}})) \!-\! c)^2\right]\nonumber \! \\
	& +\!   {\mathbb{E}}_{{\mathbf{x}}\sim p^-({\mathbf{x}})}\!\! \left[(D(\mathbf{x}) \!-\! c)^2\right],
	\tag{\ref{Generator_Relax}}
\end{align}
the optimal discriminator is
\begin{equation}
		D^*({\mathbf{x}}) = \frac{p_d({\mathbf{x}}) + d \cdot p^-({\mathbf{x}})}{p_d({\mathbf{x}}) + p_g({\mathbf{x}}) + p^-({\mathbf{x}})};
		\tag{\ref{optimal_three_level_discri}}
\end{equation}
 Moreover, if $d$ is set as 
 \begin{align}
      d=\frac{2\pi-1}{\pi+1},\nonumber
 \end{align}
the generator distribution $p_g({\mathbf{x}})$ will converge to the target distribution $p^+({\mathbf{x}})$.
\end{theorem}

\begin{proof}
    For convenience, we use $p_d$, $p^+$, $p^-$, $p_g$ to replace the $p_d(\mathbf{x})$, $p^+(\mathbf{x})$, $p^-(\mathbf{x})$, $p_g(\mathbf{x})$, converting expectations into integral form, \eqref{Discrim_Relax} and  \eqref{Generator_Relax} can be rewritten as:
    \begin{align}
    \begin{split}
    V(D) =& 
    \int_{\mathcal{S}}\bigg \{\left(D({\mathbf{x}})-1\right)^{2} p^{+}+\left(D({\mathbf{x}})-0\right)^{2} p_{g}
   \\& +\left(D({\mathbf{x}})-d\right)^{2} p^{-} \bigg \} \mathrm{d} \boldsymbol{x},
    \label{proof2:Dobjective}
\end{split}
\end{align}
where $\mathcal{S}$ represents the entire integration space. To derive the optimal discriminator, we only need to find a discriminator $D(x)$ that can minimize the integrand in $V(D)$ at every value of $x$. Thus, by setting the derivative of integrand in $V(D)$ w.r.t. $D(x)$ to $0$, that is, $2\!\left(D({\mathbf{x}})-1\right) p_d(\mathbf{x})+2\left(D({\mathbf{x}})-0\right) p_{g}(\mathbf{x}) +2\left(D({\mathbf{x}})-d\right) p^{-}(\mathbf{x})=0$, we obtain the optimal discriminator as
\begin{equation}
		D^*({\mathbf{x}}) = \frac{p_d({\mathbf{x}}) + d \cdot p^-({\mathbf{x}})}{p_d({\mathbf{x}}) + p_g({\mathbf{x}}) + p^-({\mathbf{x}})}.
		\tag{\ref{optimal_three_level_discri}}
\end{equation}

Because $p_g(\mathbf{x})$ denotes the distribution of generated samples, that is, $G({\mathbf{z}}) \sim p_g({\mathbf{x}})$, \eqref{Generator_Relax} can be equivalently written as:
\begin{align}
\begin{split}
    V(G)=&
    \int_{\mathcal{S}}\bigg \{  \left(D({\mathbf{x}})-c\right)^{2} p^{+}({\mathbf{x}})
    +\left( D({\mathbf{x}})-c\right)^{2} p_g({\mathbf{x}})
    \\&
    +\left(D({\mathbf{x}})-c\right)^{2} p^{-}({\mathbf{x}})\bigg \} \mathrm{d} \boldsymbol{x}.
    \label{proof2:Gobjective}
\end{split}
\end{align}
Substituting $p_d({\mathbf{x}}) = \pi p^+({\mathbf{x}}) + (1-\pi) p^-({\mathbf{x}})$ into \eqref{optimal_three_level_discri} gives
\begin{align}
    D^*({\mathbf{x}}) = \frac{\pi p^+({\mathbf{x}}) +\left(1-\pi\right)p^-({\mathbf{x}}) +d p^-({\mathbf{x}})}{\pi p^+({\mathbf{x}})+\left(1-\pi\right)p^-({\mathbf{x}})+p^-({\mathbf{x}})+p_g({\mathbf{x}})}.
    \label{proof2:Doptimal}
\end{align}
Then, by substituting $p_d({\mathbf{x}}) = \pi p^+({\mathbf{x}}) + (1-\pi) p^-({\mathbf{x}})$ and the optimal discriminator \eqref{proof2:Doptimal} into \eqref{proof2:Gobjective}, we can write the objective function of the generator as
\begin{align}
\begin{split}
    V(G) & =\int_{\mathcal{S}}\bigg \{\pi\left(\frac{\pi p^+ \!+\!(1\!-\!\pi)p^- \!+\!d p^-}{\pi p^+ +(1\!-\!\pi)p^- \!+\!p^- \!+\!p_g }\!-\!c\right)^{2}p^{+}\\
    &\!\! \quad +\!\left(1\!-\!\pi\right)\left(\frac{\pi p^+ \!+\!(1\!-\!\pi)p^- \!+\!d p^-}{\pi p^+ \!+\!(1\!-\!\pi)p^- \!+\!p^- \!+\!p_g }\!-\!c\right)^{2}p^{-}\\
    &\!\! \quad +\!\left(\frac{\pi p^+ \!+\!(1\!-\!\pi)p^- \!\!d p^-}{\pi p^+ \!+\!(1\!-\!\pi)p^- \!+\!p^- \!+\!p_g }\!-\!c\right)^{2}\! p^{-}\\
    &\!\!\quad +\!\left(\frac{\pi p^+ \!+\!(1\!-\!\pi)p^-\!+\!d p^-}{\pi p^+ \!+\!(1\!-\!\pi)p^- \!+\!p^- \!\!p_g }\!-\!c\right)^{2}\! p_{g} \bigg \} \mathrm{d} \boldsymbol{x}\\
    &\!\!=\int_{\mathcal{S}}\!\!\left(\frac{\pi p^+ \!+\!(1\!-\!\pi)p^- \!+\!d p^-}{\pi p^+ \!+\!(1\!-\!\pi)p^- \!+\!p^- \!+\!p_g }\!-\!c\right)^{2}  \\& \quad \cdot\left(\pi p^+\! +\!(1\!-\!\pi)p^- \!+\!p^- \!+\!p_g\right) \mathrm{d} \boldsymbol{x}.
    \label{proof3:Geq1}
\end{split}
\end{align}

Let $\!\varphi(x) \!=\! \left(x-c\right)^2$. Because $\varphi(x)$ is a convex function, it has a property that $\varphi \!\left(\int_{-\infty}^{\infty} g(x)f(x) \mathrm{d} \boldsymbol{x}\right)\!\!\le\!\! \int_{-\infty}^{\infty} \varphi\left(g(x)\right)\!f(x)\ \mathrm{d} \boldsymbol{x} $, where $\int_{-\infty}^{\infty} f(x)  \mathrm{d} \boldsymbol{x}\!=\!1$. Now, we can infer from \eqref{proof3:Geq1} the following relation
\begin{align}
\begin{split}
    V(G) &\!= \! 3\int_{\mathcal{S}}\! \frac{1}{3}\varphi\left(\!\frac{\pi p^+ \!+\!(1\!-\!\pi)p^- \!+\!d p^-}{\pi p^+ \!+\!(1\!-\!\pi)p^- \!+\!p^- \!+\!p_g }\right)\!\!
    \\& \quad \cdot \left(\pi p^+\! +\!(1\!-\!\pi)p^- \!+\!p^- \!+\!p_g\right)\mathrm{d} \boldsymbol{x}\\
    &\!\!\ge\!3\ \varphi \bigg \{\! \int_{\mathcal{S}}\! \frac{1}{3}\!\left(\!\frac{\pi p^+ \!+\!(1\!-\!\pi)p^- \!+\!d p^-}{\pi p^+ \!+\!(1\!-\!\pi)p^- \!+\!p^- \!+\!p_g }\right)\!\!
    \\& \quad \cdot \left(\pi p^+\! +\!(1\!-\!\pi)p^- \!+\!p^- \!+\!p_g\right)\! \mathrm{d} \boldsymbol{x} \bigg \}\\
    &\!\!= \!3\  \varphi\bigg \{\! \int_{\mathcal{S}}\! \frac{1}{3} 
    \left(\pi p^+ \!+\!(1\!-\!\pi)p^- \!+\! dp^-  \!\right)\!\mathrm{d} \boldsymbol{x} \bigg\}.
    \label{proof2:Geq3}
\end{split}
\end{align}
Because $p^{+},p^{-},p_{g}$ are all probability distribution, their integrations over ${\mathcal{S}}$ should be equal to 1. Thus,  \eqref{proof2:Geq3} can be further written as
\begin{align}
\begin{split}
    V(G) &\ge 3\  \varphi\bigg \{\! \int_{\mathcal{S}}\! \frac{1}{3} \left(\pi p^+ \!+\!(1\!-\!\pi)p^- \!+\! dp^-  \!\right)\!\mathrm{d} \boldsymbol{x} \bigg\}\\
    &= 3\ \varphi\left(\frac{\pi}{\pi+1}\right),
    \label{proof2:Geq4}
\end{split}
\end{align}
which indicates the minimum value of $V(G)$ is $3\ \varphi\left(\frac{\pi}{\pi+1}\right)$.

On the other hand, if we substitute $p_{g} = p^{+}$ into \eqref{proof3:Geq1}, we can obtain 
\begin{align}
\begin{split}
    \label{proof2:Geq5}
    \widetilde V(G) & = \int_{\mathcal{S}}\!\varphi\left(\!\frac{\pi p^+ \!+\!(1\!-\!\pi)p^- \!+\!d p^-}{\pi p^+ \!+\!(1\!-\!\pi)p^- \!+\!p^- \!+\!p^+ }\right)\!\!
    \\& \quad \cdot \left(\pi p^+\! +\!(1\!-\!\pi)p^- \!+\!p^- \!+\!p^+\right)\mathrm{d} \boldsymbol{x}.\\
\end{split}
\end{align}
Substituting $d=\frac{2\pi-1}{\pi+1}$ into \eqref{proof2:Geq5} gives
\begin{align}
\begin{split}
    \widetilde V(G) & = \int_{\mathcal{S}}\!\varphi\left(\!\frac{\pi p^+ \!+\!(1\!-\!\pi)p^- \!+\!d p^-}{\pi p^+ \!+\!(1\!-\!\pi)p^- \!+\!p^- \!+\!p^+ }\right)\!\!
    \\& \quad \cdot \left(\pi p^+\! +\!(1\!-\!\pi)p^- \!+\!p^- \!+\!p^+\right)\mathrm{d} \boldsymbol{x}\\
    & = \!\int_{\mathcal{S}}\!\!\bigg \{\!\varphi\!\left(\frac{\pi}{\pi+1}\right)\!\left(\left(1+\pi\right) p^+ +\left(2-\pi\right)p^-\right)\!\!\bigg \}\mathrm{d} \boldsymbol{x}\\
    &  =\! 3\ \varphi\left(\frac{\pi}{\pi+1}\right).
    \label{proof2:Geq6}
\end{split}
\end{align}
From \eqref{proof2:Geq4}, it is known that $3\ \varphi\left(\frac{\pi}{\pi+1}\right)$ is the minimum value of $V(G)$. From \eqref{proof2:Geq6}, we can see that $V(G)$ can attain this minimum value by setting $p_g(\mathbf{x}) = p^+(\mathbf{x})$. Therefore, we can conclude that the minimum value of $V(G)$ can be obtained at $p_g(\mathbf{x}) = p^+(\mathbf{x})$.
\end{proof}

\vspace{5mm}
\section{Additional Experimental Results and Training Details}
\paragraph{Dataset}
In the task of image generation, our model is evaluated on five image datasets.   i) \emph{MNIST}: consisting of 60000 training and 10000 testing handwritten digit images of 28$\times$28 from 10 classes; ii) \emph{F-MNIST}: consisting of 60000 training and 10000 testing fashion images of 28$\times$28 from 10 classes; iii) \emph{SVHN}: containing 73257 training and 26032 testing house-number images obtained from Google Street View; iv) \emph{CIFAR10}: A dataset consisting of 50000 training and 10000 testing images from ten classes; v) \emph{CelebA}: a large-scale face dataset consisting of 200K celebrity images, with each associated with 40 attribute annotations. In semi-supervised anomaly detection and PU-learning, our model is evaluated on commonly used datasets for the two tasks.
\paragraph{Sensitivity Analysis to Estimation Error of $\pi$}

\begin{figure*}[h]
\centering
    \subfigure[MNIST]{
	\begin{minipage}{0.31\linewidth}
		\centering
		\includegraphics[width=0.95\linewidth]{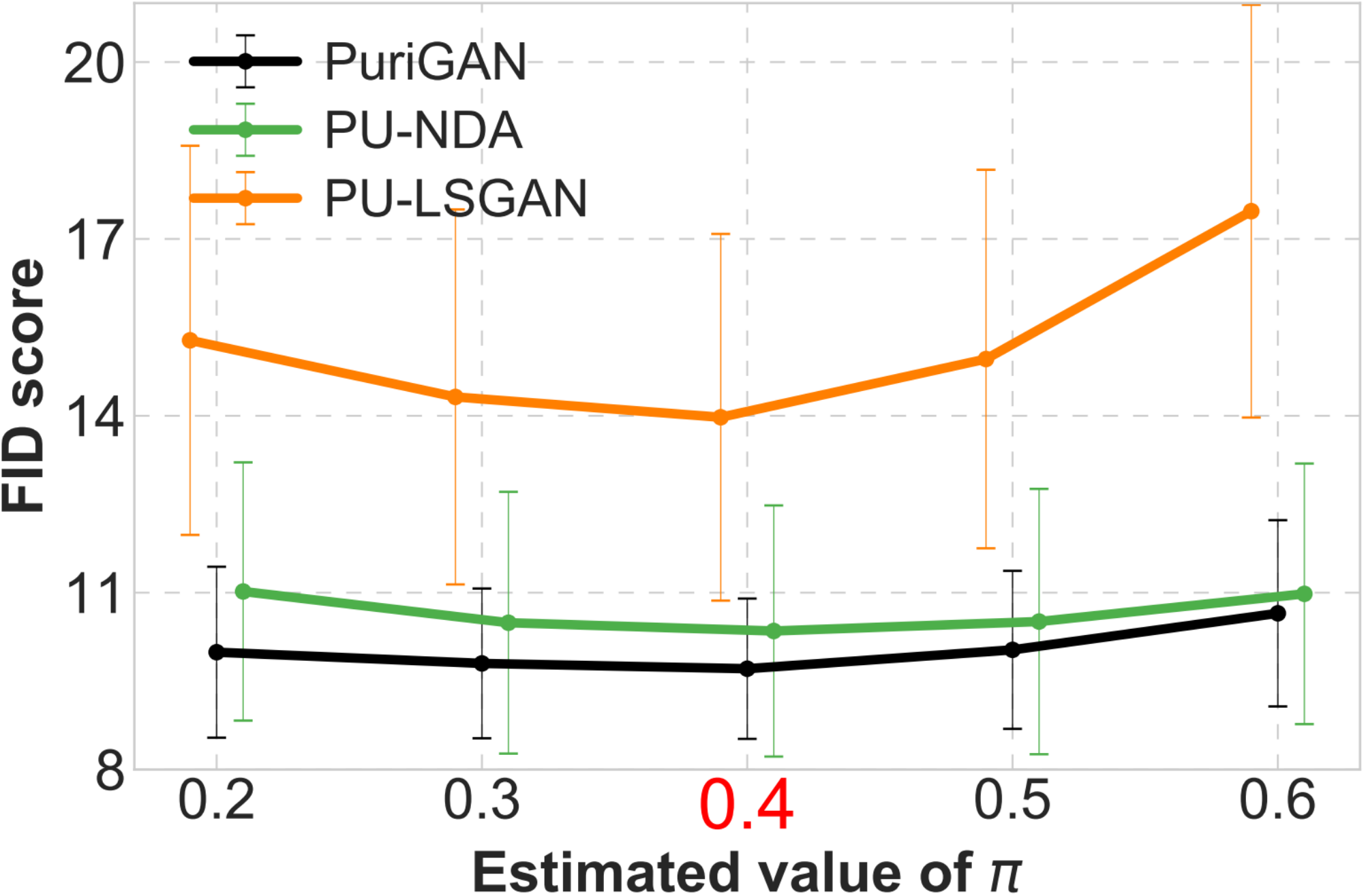}\\
	    \vspace{2mm}
		\includegraphics[width=0.95\linewidth]{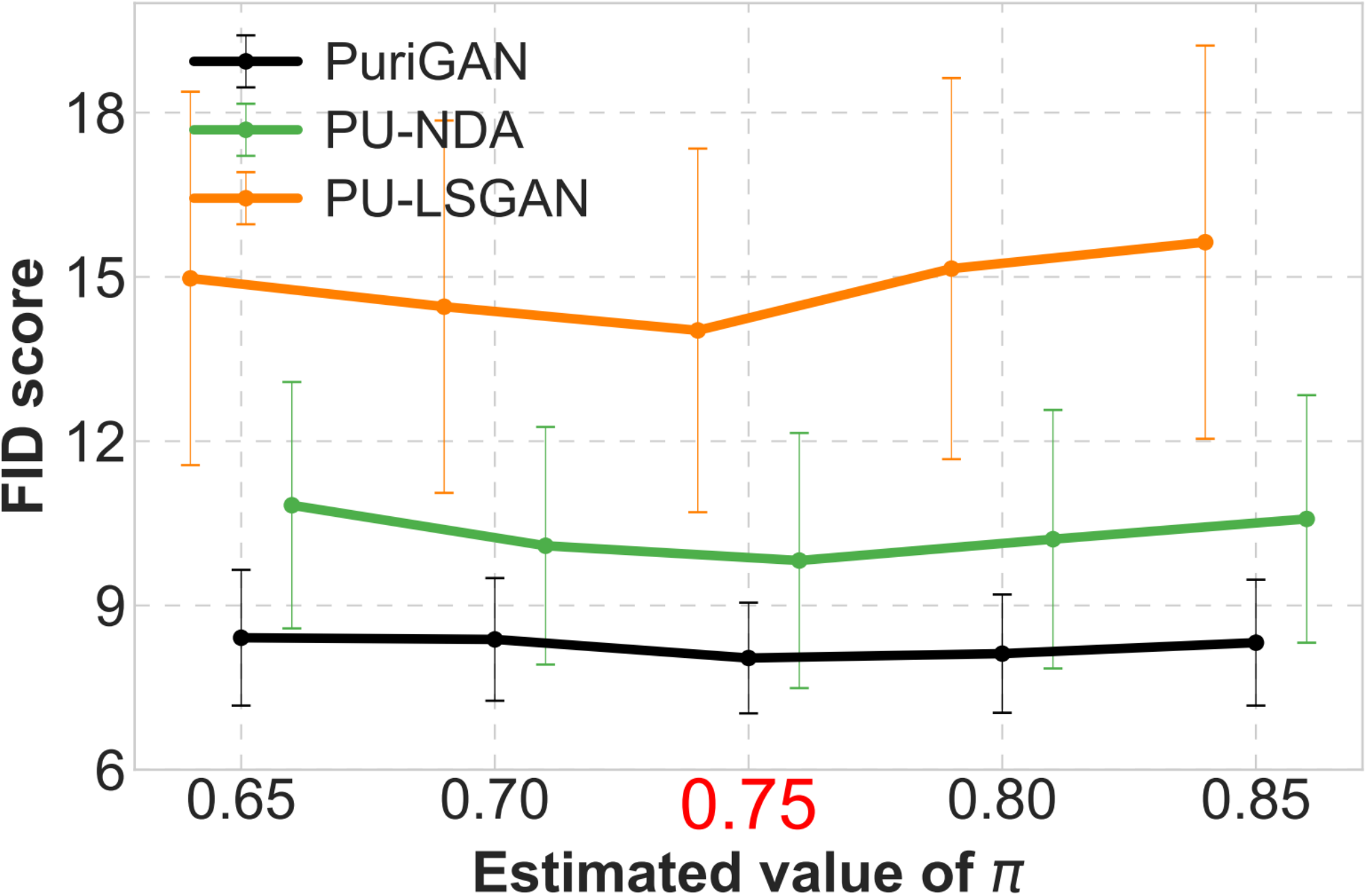}\\
	\vspace{2mm}
	\end{minipage}}
	\subfigure[FMNIST]{
	\begin{minipage}{0.31\linewidth}
		\centering
		\includegraphics[width=0.95\linewidth]{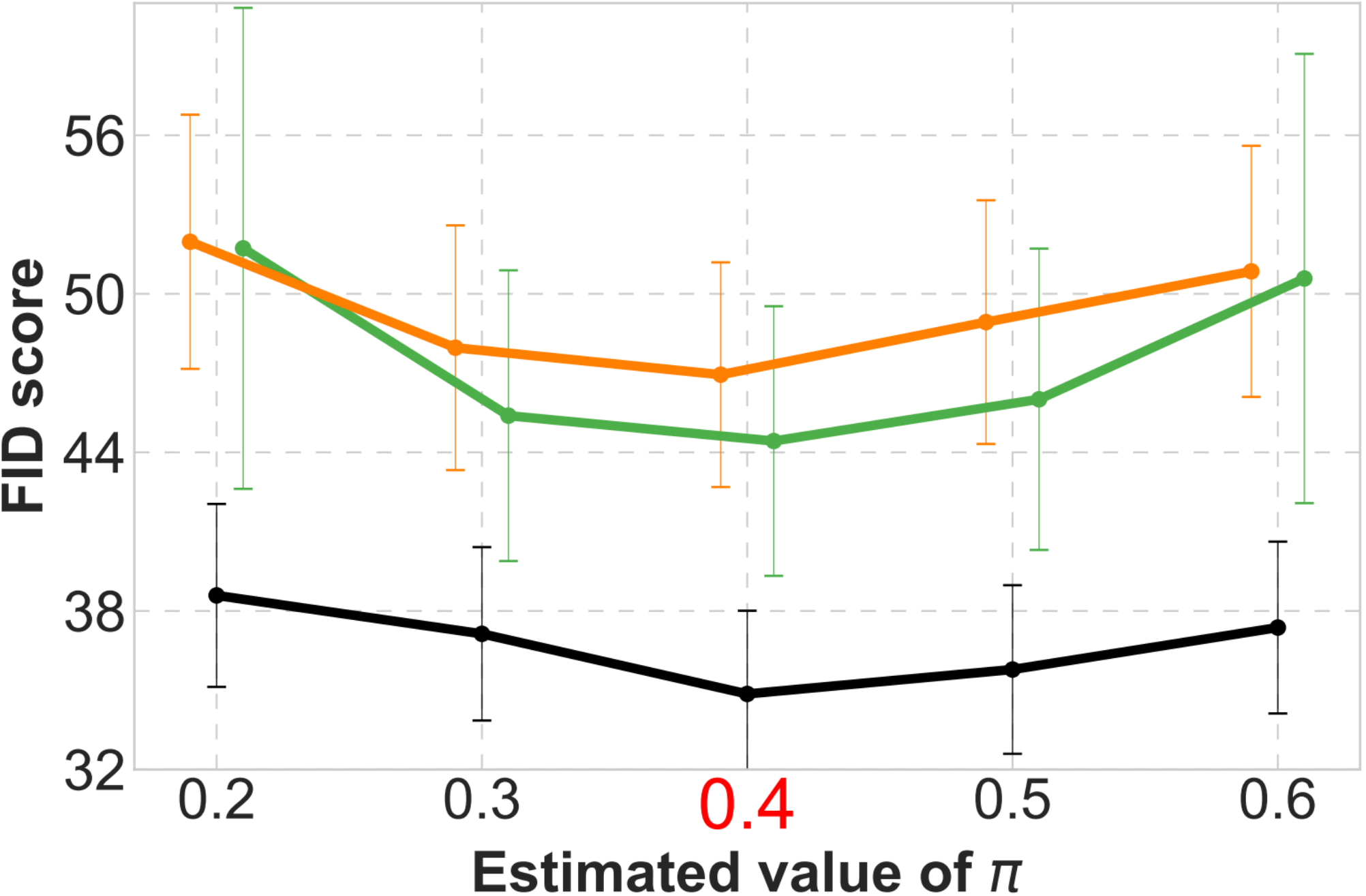}\\
		\vspace{2mm}
		\includegraphics[width=0.95\linewidth]{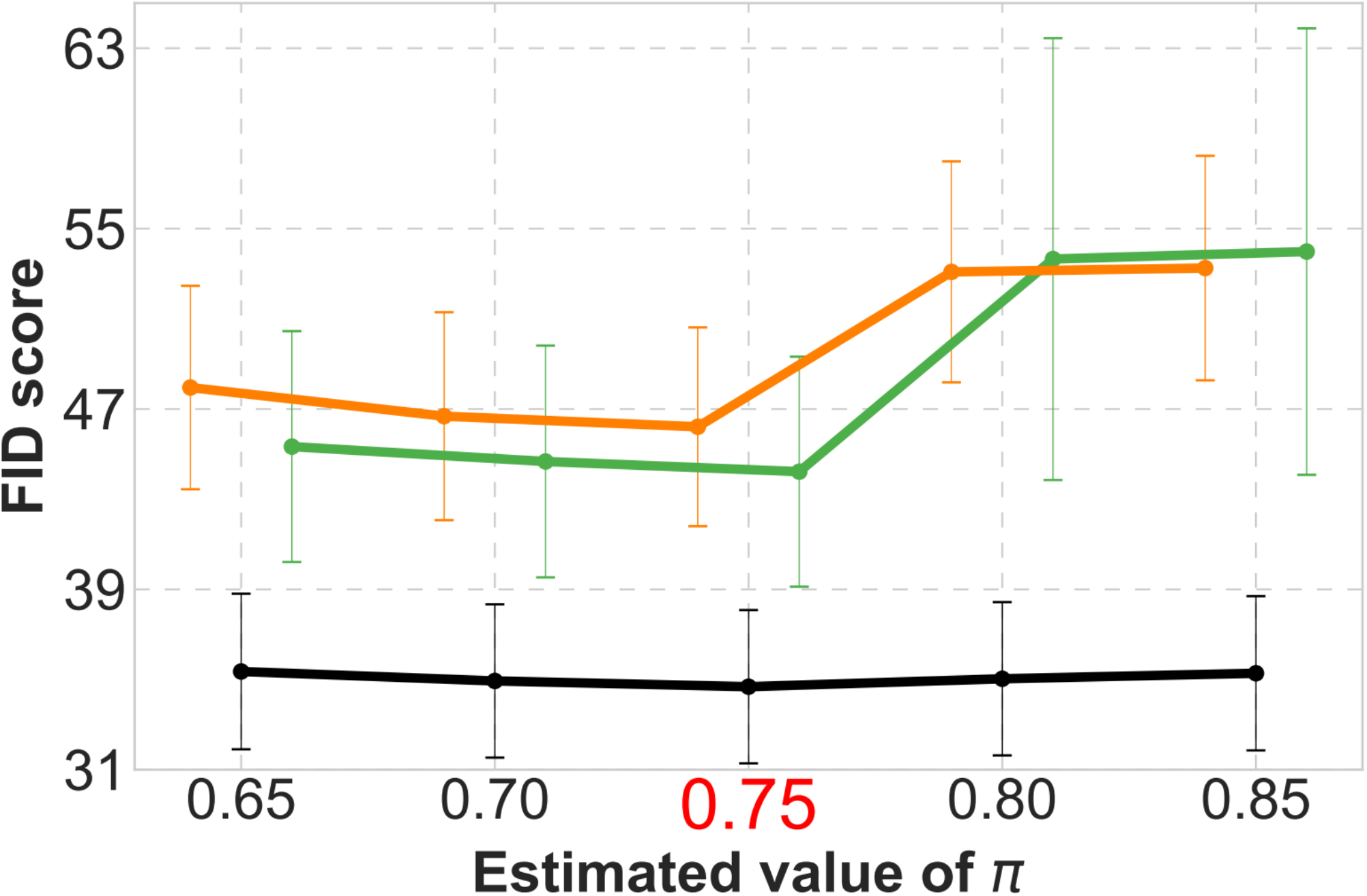}\\
	\vspace{2mm}
	\end{minipage}}
	\subfigure[SVHN]{
	\begin{minipage}{0.31\linewidth}
		\centering
		\includegraphics[width=0.95\linewidth]{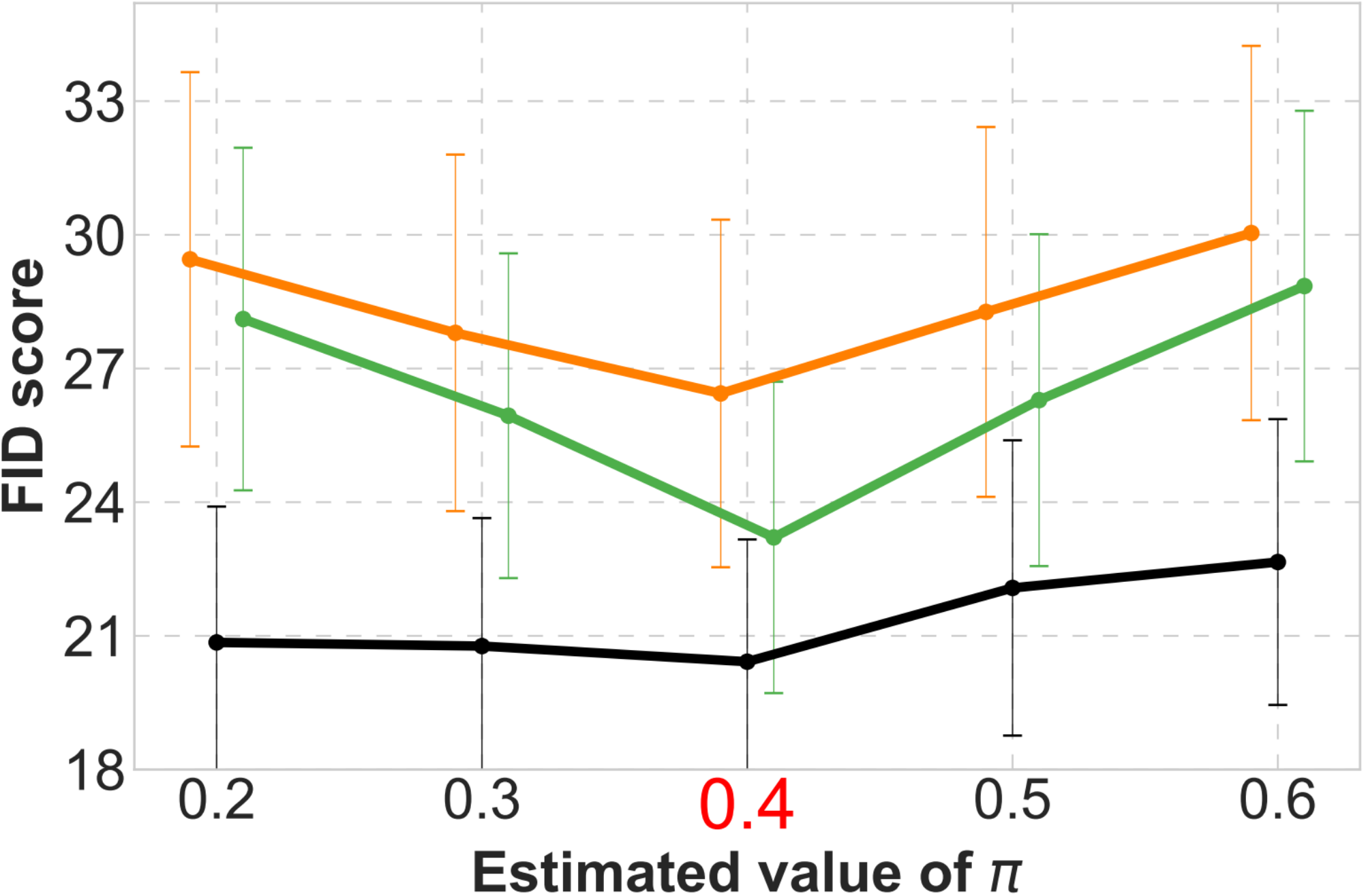}\\
		\vspace{2mm}
		\includegraphics[width=0.95\linewidth]{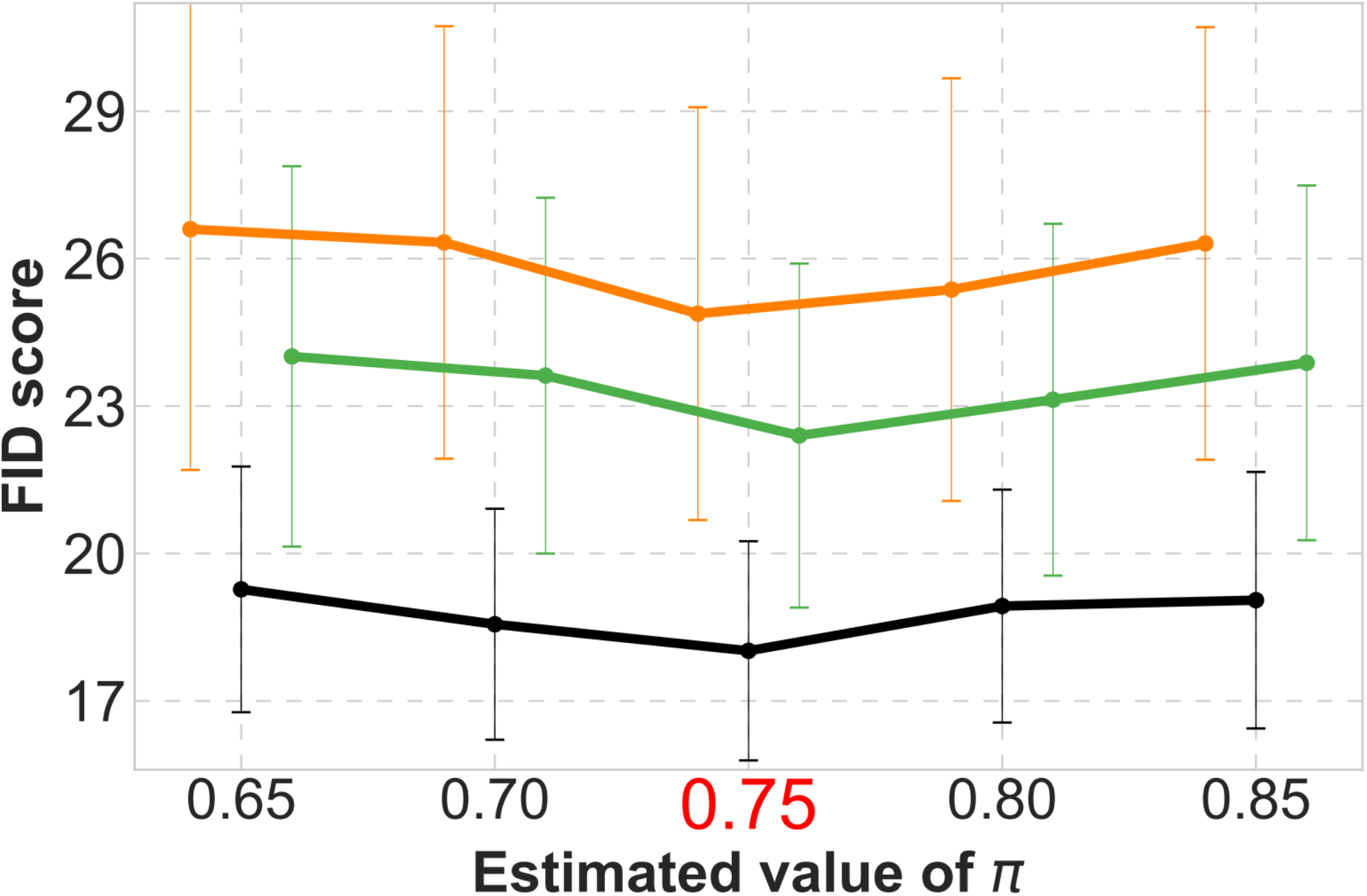}\\
	\vspace{2mm}
	\end{minipage}}
    \caption{Sensitivity analysis to the estimation error of $\pi$ under different scenarios. (The correct $\pi$ value is marked in red.)}
    \label{est_pi}
\end{figure*}

We conduct experiments to investigate the impact of the estimation error of $\pi$ under different scenarios. We compare our proposed PuriGAN with two-stage methods,  which also need to estimate the value of $\pi$ in the first PU-learning stage. The final results are shown in Figure \ref{est_pi}, in which the ground-truth value of $\pi$ is marked in red. From the figure, it can be observed that the the performance of our proposed model is not sensitive to the estimation error of $\pi$, with only a slight performance drop observed on inaccurately estimated $\pi$. In contrast, two-stage methods seem to depend more on the estimation accuracy of $\pi$, especially on complex datasets F-MNIST and SVHN. This is due to the wrong value of $\pi$ causes the PU-learning classifier to learn the wrong boundaries, which will lead to a poor discrimination between the contaminated and target instances. The bad performance of the PU-learning stage will be passed down to the second stage, causing the even worse performance of the whole model. Our end-to-end model simultaneously performs purification and generation, on the other hand, is more robust to the estimation error of $\pi$.

\begin{table}[t]
\centering
\small
\setlength\tabcolsep{4pt}
{\begin{tabular}{ccccccccc}
\toprule
\makecell[c]{Data} & \makecell[c]{$\gamma_p$}  & \makecell[c]{$\lambda=1$}  & \makecell[c]{$\lambda=3$} &  \makecell[c]{$\lambda=5$} \\
\cmidrule(lr){1-1} \cmidrule(lr){2-2} \cmidrule(lr){3-3} \cmidrule(lr){4-4} \cmidrule(lr){5-5} 
& 0.1 & \textbf{9.04 $\pm$ 1.5} & 9.16 $\pm$ 1.7 & 9.41 $\pm$ 1.6 \\
MNIST & 0.3 & 9.51 $\pm$ 1.5 & \textbf{9.38 $\pm$ 1.5} & 9.75 $\pm$ 1.5 \\
& 0.5 &  \textbf{10.68 $\pm$ 1.6} & 12.03 $\pm$ 1.8 & 12.89 $\pm$ 2.0 \\
\cmidrule(lr){1-1} \cmidrule(lr){2-2} \cmidrule(lr){3-3} \cmidrule(lr){4-4} \cmidrule(lr){5-5} 
& 0.1 & \textbf{36.24 $\pm$ 3.2} & 37.57 $\pm$ 3.7 & 37.76 $\pm$ 3.6 \\
F-MNIST & 0.3 & \textbf{37.17 $\pm$ 3.2} & 38.05 $\pm$ 3.6 & 38.22 $\pm$ 3.7 \\
& 0.5 &  \textbf{38.18 $\pm$ 3.1} & 39.60 $\pm$ 3.3 & 40.97 $\pm$ 3.6 \\
\cmidrule(lr){1-1} \cmidrule(lr){2-2} \cmidrule(lr){3-3} \cmidrule(lr){4-4} \cmidrule(lr){5-5} 
& 0.1 & 17.79 $\pm$ 2.1 & 17.41 $\pm$ 2.0 & \textbf{16.90 $\pm$ 1.9} \\
SVHN & 0.3 & 19.30 $\pm$ 2.5 & \textbf{18.08 $\pm$ 2.0} & 18.14 $\pm$ 2.1 \\
& 0.5 &  21.67 $\pm$ 2.9 & 20.42 $\pm$ 2.7 & \textbf{19.58 $\pm$ 2.6} \\

\bottomrule
\end{tabular}
\caption{Sensitivity analysis w.r.t $\lambda$ under different contamination ratio scenarios.}

\label{lambda_re}}
\vspace{5mm}
\end{table}

\begin{table}[t]
\centering
\small
\setlength\tabcolsep{4pt}
{\begin{tabular}{ccccccccc}
\toprule
\makecell[c]{Data} & \makecell[c]{$\gamma_p$}  & \makecell[c]{$c=0.5$}  & \makecell[c]{$c=0.6$} &  \makecell[c]{$c=0.7$} \\
\cmidrule(lr){1-1} \cmidrule(lr){2-2} \cmidrule(lr){3-3} \cmidrule(lr){4-4} \cmidrule(lr){5-5} 
& 0.1 & \textbf{9.04 $\pm$ 1.5} & 9.23 $\pm$ 1.5 & 9.34 $\pm$ 1.5 \\
MNIST & 0.3 & \textbf{9.51 $\pm$ 1.5} & 10.23 $\pm$ 1.7 & 9.69 $\pm$ 1.5 \\
& 0.5 &  \textbf{10.68 $\pm$ 1.6} & 12.25 $\pm$ 2.0 & 11.67 $\pm$ 1.8 \\
\cmidrule(lr){1-1} \cmidrule(lr){2-2} \cmidrule(lr){3-3} \cmidrule(lr){4-4} \cmidrule(lr){5-5} 
& 0.1 & \textbf{36.24 $\pm$ 3.2} & 36.45 $\pm$ 3.2 & 37.13 $\pm$ 3.5 \\
F-MNIST & 0.3 & 37.17 $\pm$ 3.2 & \textbf{36.85 $\pm$ 3.1} & 38.44 $\pm$ 3.6 \\
& 0.5 &  38.18 $\pm$ 3.1 &  \textbf{37.89 $\pm$ 3.0} & 38.79 $\pm$ 3.2 \\
\cmidrule(lr){1-1} \cmidrule(lr){2-2} \cmidrule(lr){3-3} \cmidrule(lr){4-4} \cmidrule(lr){5-5} 
& 0.1 & \textbf{17.79 $\pm$ 2.1} & 18.01 $\pm$ 2.2 & 17.90 $\pm$ 2.1 \\
SVHN & 0.3 & 19.30 $\pm$ 2.5 & 19.79 $\pm$ 2.5 & \textbf{18.83 $\pm$ 2.1} \\
& 0.5 & 21.67 $\pm$ 2.9 & 22.57 $\pm$ 3.1 & \textbf{20.98 $\pm$ 2.6} \\

\bottomrule
\end{tabular}
\caption{Sensitivity analysis w.r.t $c$ under different contamination ratio scenarios.}

\label{c_re}}
\vspace{5mm}
\end{table}

\begin{table*}[!t]
\centering
\small
\setlength\tabcolsep{2.2pt}
{\begin{tabular}{ccccccccc}
\toprule
\makecell[c]{Data} & \makecell[c]{$\gamma_p$}  & \makecell[c]{LSGAN}  & \makecell[c]{NDA}  &\makecell[c]{GenPU}  &\makecell[c]{PU-LSGAN} &  \makecell[c]{PU-NDA} & \makecell[c]{PuriGAN$_1$} & \makecell[c]{PuriGAN$_2$} \\
\cmidrule(lr){1-1} \cmidrule(lr){2-2} \cmidrule(lr){3-3} \cmidrule(lr){4-4} \cmidrule(lr){5-5} \cmidrule(lr){6-6} \cmidrule(lr){7-7} \cmidrule(lr){8-8} \cmidrule(lr){9-9}
MNIST& 0 & 13.98 $\pm$ 3.1  & 9.63 $\pm$ 2.4 &  9.31 $\pm$ 3.8 & 13.98 $\pm$ 3.1 & 9.63 $\pm$ 2.4 & 8.67 $\pm$ 1.8 & \textbf{7.77 $\pm$ 1.2} \\ 
& 0.1 & 17.90 $\pm$ 3.5 & 10.99 $\pm$ 2.1 & 11.42 $\pm$ 3.9 & 14.93 $\pm$ 3.1 & 10.01 $\pm$ 3.0 & 9.16 $\pm$ 1.7 & \textbf{8.13 $\pm$ 1.4} \\
& 0.2 &  21.42 $\pm$ 2.9 & 12.47 $\pm$ 1.9 & 13.67 $\pm$ 4.3 & 14.70 $\pm$ 3.0 & 9.90 $\pm$ 2.9 & 9.23 $\pm$ 1.7 &\textbf{8.57 $\pm$ 1.5} \\
& 0.3 & 23.05 $\pm$ 4.8 & 18.61 $\pm$ 7.4 & 17.17 $\pm$ 5.5 & 15.17 $\pm$ 3.2  & 10.21 $\pm$ 3.1 & 9.51 $\pm$ 1.5 & \textbf{8.33 $\pm$ 1.0}\\
& 0.4 & 26.93 $\pm$ 5.7 & 21.94 $\pm$ 8.9 &  21.28 $\pm$ 6.0 & 15.21 $\pm$ 3.2 & 10.35 $\pm$ 3.1 & 9.71 $\pm$ 1.9 &\textbf{9.51 $\pm$ 1.0} \\
& 0.5 & 30.83 $\pm$ 9.0 & 31.56 $\pm$ 11.6 & 23.10 $\pm$ 8.2 & 15.04 $\pm$ 3.1 & \textbf{10.45 $\pm$ 3.0} & 10.68 $\pm$ 1.6 &10.68 $\pm$ 1.6\\
\cmidrule(lr){1-1} \cmidrule(lr){2-2} \cmidrule(lr){3-3} \cmidrule(lr){4-4} \cmidrule(lr){5-5} \cmidrule(lr){6-6} \cmidrule(lr){7-7} \cmidrule(lr){8-8} \cmidrule(lr){9-9}
F-MNIST & 0 & 41.00 $\pm$ 4.9 & 39.43 $\pm$ 3.6 &  40.75 $\pm$ 3.4 & 41.00 $\pm$ 4.9 & 39.43 $\pm$ 3.6 & \textbf{35.98 $\pm$ 3.3} &36.12 $\pm$ 3.1  \\
& 0.1 & 46.38 $\pm$ 4.6 & 44.54 $\pm$ 5.1 &  39.41 $\pm$ 4.8 & 43.42 $\pm$ 5.3 & 42.25 $\pm$ 5.2 & 36.24 $\pm$ 3.2&\textbf{35.39 $\pm$ 3.4} \\
& 0.2 & 51.47 $\pm$ 4.4 & 51.72 $\pm$ 9.1 & 45.23 $\pm$ 7.7 & 44.83 $\pm$ 5.0 & 43.07 $\pm$ 5.3 & 36.80 $\pm$ 3.4 &\textbf{34.19 $\pm$ 3.3} \\
& 0.3 & 56.74 $\pm$ 4.9 &56.93 $\pm$ 11.0 & 51.12 $\pm$ 6.5 & 45.26 $\pm$ 4.9  & 42.86 $\pm$ 5.2 & 37.17 $\pm$ 3.2 & \textbf{34.97 $\pm$ 3.5}\\
& 0.4 & 62.44 $\pm$ 6.8 & 58.88 $\pm$ 14.3 &  58.73 $\pm$ 8.0 & 46.94 $\pm$ 5.1 &44.43 $\pm$ 5.4 & 37.61 $\pm$ 3.4 &\textbf{34.86 $\pm$ 3.2} \\
&0.5 & 73.87 $\pm$ 9.0 & 61.47 $\pm$ 11.6 & 63.04 $\pm$ 8.5 & 46.06 $\pm$ 4.8 & 43.55 $\pm$ 5.2 & \textbf{38.18 $\pm$ 3.1} &\textbf{38.18 $\pm$ 3.1}\\
\cmidrule(lr){1-1} \cmidrule(lr){2-2} \cmidrule(lr){3-3} \cmidrule(lr){4-4} \cmidrule(lr){5-5} \cmidrule(lr){6-6} \cmidrule(lr){7-7} \cmidrule(lr){8-8} \cmidrule(lr){9-9}
SVHN & 0 &  22.40 $\pm$ 2.6 & 20.29 $\pm$ 3.2 & 19.49 $\pm$ 3.6 & 22.40 $\pm$ 2.6  &20.29 $\pm$ 3.2 & 16.82 $\pm$ 2.0 & \textbf{16.27 $\pm$ 1.9 }\\
& 0.1 &  24.88 $\pm$ 3.9 & 22.83 $\pm$ 3.5 & 22.27 $\pm$ 3.5 & 25.16 $\pm$ 4.3  & 21.88 $\pm$ 3.3 &  17.70 $\pm$ 2.1 & \textbf{17.26 $\pm$ 2.1} \\
& 0.2 & 26.60 $\pm$ 4.2 & 23.67 $\pm$ 3.4 & 23.98 $\pm$ 5.0 & 25.79 $\pm$ 4.0 & 22.93 $\pm$ 3.5 &  18.61 $\pm$ 2.3 & \textbf{17.86 $\pm$ 2.1} \\
& 0.3 & 27.08 $\pm$ 4.9 & 24.67 $\pm$ 3.7 & 25.57 $\pm$ 5.5  & 27.00 $\pm$ 5.2  & 22.67 $\pm$ 3.4 & 19.39 $\pm$ 2.5 & \textbf{18.34 $\pm$ 2.3}\\
& 0.4 & 28.50 $\pm$ 4.9 & 27.58 $\pm$ 4.2 &  26.87 $\pm$ 5.6 & 26.44 $\pm$ 4.1 & 23.21 $\pm$ 3.5 & 20.42 $\pm$ 2.7 & \textbf{19.63 $\pm$ 2.5 }\\
& 0.5 & 30.09 $\pm$ 4.2 & 28.99 $\pm$ 4.3 &  28.43 $\pm$ 5.3  & 27.06 $\pm$ 4.8 & 23.40 $\pm$ 3.5 & \textbf{21.67 $\pm$ 2.9} &\textbf{21.67 $\pm$ 2.9}\\
\cmidrule(lr){1-1} \cmidrule(lr){2-2} \cmidrule(lr){3-3} \cmidrule(lr){4-4} \cmidrule(lr){5-5} \cmidrule(lr){6-6} \cmidrule(lr){7-7} \cmidrule(lr){8-8} \cmidrule(lr){9-9}
CELEBA & 0 &  36.41 $\pm$ 1.5 & 38.64 $\pm$ 1.8 &  34.82 $\pm$ 1.4 & 36.41 $\pm$ 1.5  &38.64 $\pm$ 1.8 & \textbf{31.45 $\pm$ 1.4} & 31.59 $\pm$ 1.4 \\ 
& 0.1 &  39.10 $\pm$ 1.9 & 41.56 $\pm$ 1.8 &  35.57 $\pm$ 1.8 & 40.20 $\pm$ 2.2  &39.58 $\pm$ 2.3  & 32.65 $\pm$ 1.3& \textbf{32.13 $\pm$ 1.5 }\\
& 0.2 &  42.87 $\pm$ 1.9 & 44.24 $\pm$ 2.0 &  37.38 $\pm$ 1.7 & 41.25 $\pm$ 2.3 & 42.66 $\pm$ 2.5  &  \textbf{33.85 $\pm$ 1.4}  & 34.16 $\pm$ 1.4 \\
& 0.3 & 45.11 $\pm$ 1.9 & 47.41 $\pm$ 2.0 &  39.97 $\pm$ 1.9  & 43.19 $\pm$ 2.4  & 45.91 $\pm$ 2.8 & \textbf{34.56 $\pm$ 1.4} & 35.40 $\pm$ 1.5 \\
& 0.4 & 49.29 $\pm$ 1.9 & 52.93 $\pm$ 2.3 &  45.81 $\pm$ 1.8 & 44.75 $\pm$ 2.3 & 46.34 $\pm$ 2.8 & 36.43 $\pm$ 1.5 & \textbf{35.67 $\pm$ 1.5} \\
& 0.5 &  55.25 $\pm$ 2.0  & 53.30 $\pm$ 2.5 & 48.06 $\pm$ 2.0  & 47.06 $\pm$ 2.6 & 47.03 $\pm$ 2.9 & \textbf{39.39 $\pm$ 1.5} &\textbf{ 39.39 $\pm$ 1.6}\\
\cmidrule(lr){1-1} \cmidrule(lr){2-2} \cmidrule(lr){3-3} \cmidrule(lr){4-4} \cmidrule(lr){5-5} \cmidrule(lr){6-6} \cmidrule(lr){7-7} \cmidrule(lr){8-8} \cmidrule(lr){9-9}
CIFAR-10 & 0 &  53.68 $\pm$ 9.5 & 67.77 $\pm$ 11.1 &  53.63 $\pm$ 9.4 & 53.68 $\pm$ 9.5  &67.77 $\pm$ 11.1 & 52.19 $\pm$ 8.8 & \textbf{51.75 $\pm$ 9.1} \\ 
& 0.1 &  57.85 $\pm$ 9.5 & 67.93 $\pm$ 10.8 &  54.47 $\pm$ 9.8 & 58.19 $\pm$ 11.4  & 69.01 $\pm$ 12.0  & 52.64 $\pm$ 8.9& \textbf{51.58 $\pm$ 10.4 }\\
& 0.2 &  59.74 $\pm$ 9.6 & 69.26 $\pm$ 11.0 &  57.68 $\pm$ 10.7 & 58.27 $\pm$ 11.2 & 66.92 $\pm$ 11.3  &  53.57 $\pm$ 9.4  & \textbf{52.83 $\pm$ 9.6} \\
& 0.3 & 60.50 $\pm$ 9.9 & 71.16 $\pm$ 11.4 & 59.84 $\pm$ 10.1  & 59.58 $\pm$ 11.3  & 68.81 $\pm$ 11.6 & 54.19 $\pm$ 9.8 & \textbf{52.53 $\pm$ 10.5} \\
& 0.4 & 61.08 $\pm$ 9.9 & 72.95 $\pm$ 10.7 &  62.59 $\pm$ 10.8 & 59.26 $\pm$ 10.7 & 70.12 $\pm$ 11.8 & 54.70 $\pm$ 10.4 & \textbf{52.70 $\pm$ 10.6} \\
& 0.5 &  63.79 $\pm$ 11.2  & 73.03 $\pm$ 11.0 & 65.06 $\pm$ 10.3  & 59.43 $\pm$ 11.3 & 69.61 $\pm$ 11.5 & \textbf{55.35 $\pm$ 10.5} &\textbf{ 55.35 $\pm$ 10.5}\\
\bottomrule
\end{tabular}
\caption{Complete FID scores$\downarrow$ on different datasets under various $\gamma_p$ scenarios, where PuriGAN$_1$ and PuriGAN$_2$ denote PuriGAN using two- and three-level discriminator, respectively.}

\label{Overall_fidscore_app}}
\end{table*}

\begin{table*}[t]
\centering
\small
\setlength\tabcolsep{2.2pt}
{\begin{tabular}{ccccccccc}
\toprule
\makecell[c]{Data} & \diagbox{model}{$\gamma_c$}& \makecell[c]{0} &\makecell[c]{0.05} & \makecell[c]{0.1} & \makecell[c]{0.2} & \makecell[c]{0.33} & \makecell[c]{0.5} & \makecell[c]{1.0} \\
\cmidrule(lr){1-1} \cmidrule(lr){2-2} \cmidrule(lr){3-3} \cmidrule(lr){4-4} \cmidrule(lr){5-5} \cmidrule(lr){6-6} \cmidrule(lr){7-7} \cmidrule(lr){8-8} \cmidrule(lr){9-9} 
MNIST & PuriGAN & 30.83 $\pm$ 9.0 &\textbf{14.90 $\pm$ 4.0} & \textbf{13.48 $\pm$ 3.0} & \textbf{10.68 $\pm$ 1.6} & \textbf{10.94  $\pm$ 1.3} & \textbf{10.42 $\pm$ 2.1} & \textbf{10.00 $\pm$ 1.9}\\
& PU-LSGAN & 30.83 $\pm$ 9.0 &16.03 $\pm$ 4.2 & 15.60 $\pm$ 4.2 & 15.04 $\pm$ 4.1 & 15.24 $\pm$ 4.2 & 14.53 $\pm$ 3.9 & 14.29 $\pm$ 3.9 \\
& PU-NDA & 29.63 $\pm$ 8.6 & 10.67 $\pm$ 2.4 & 10.55 $\pm$ 2.5 & 10.45 $\pm$ 2.4 & 10.15 $\pm$ 2.3 & 9.91 $\pm$ 2.4   & 10.07 $\pm$ 2.5 \\
& NDA & 29.63 $\pm$ 8.6 & 34.89 $\pm$ 12.6 & 32.46 $\pm$ 12.3 & 31.56 $\pm$ 11.6 & 29.52 $\pm$ 10.1 & 26.78 $\pm$ 7.4   & 23.73 $\pm$ 6.5 \\
& GenPU &29.63 $\pm$ 8.6 & 24.83 $\pm$ 11.1 & 24.19 $\pm$ 10.7 & 23.10 $\pm$ 10.2 & 22.72 $\pm$ 9.0 & 21.98 $\pm$ 8.1   &20.79 $\pm$ 8.0 \\
\cmidrule(lr){1-1} \cmidrule(lr){2-2} \cmidrule(lr){3-3} \cmidrule(lr){4-4} \cmidrule(lr){5-5} \cmidrule(lr){6-6} \cmidrule(lr){7-7} \cmidrule(lr){8-8} \cmidrule(lr){9-9} 
F-MNIST & PuriGAN &73.87 $\pm$ 9.0 & \textbf{44.26 $\pm$ 3.9} & \textbf{43.06 $\pm$ 3.7} & \textbf{38.18 $\pm$ 3.1} & \textbf{35.83 $\pm$ 2.9} & \textbf{34.61 $\pm$ 2.9} & \textbf{31.50 $\pm$ 2.5}\\
& PU-LSGAN & 73.87 $\pm$ 9.0 & 50.93 $\pm$ 5.2 & 50.10 $\pm$ 5.2 & 46.06 $\pm$ 4.8 & 43.61 $\pm$ 4.8 & 42.18 $\pm$ 5.1  & 41.32 $\pm$ 5.0 \\
& PU-NDA & 70.06 $\pm$ 8.8 & 44.95 $\pm$ 5.5 & 44.03 $\pm$ 5.3 & 43.55 $\pm$ 5.2 & 42.76 $\pm$ 4.7  & 40.63 $\pm$ 9.2 & 39.86 $\pm$ 3.6 \\
& NDA & 70.06 $\pm$ 8.8 & 69.40 $\pm$ 13.5 & 64.65 $\pm$ 12.3 & 61.47 $\pm$ 11.6 & 57.75 $\pm$ 10.2  & 53.82 $\pm$ 9.2 & 52.67 $\pm$ 9.4 \\
& GenPU & 70.06 $\pm$ 8.8 & 60.88 $\pm$ 8.7 & 62.31 $\pm$ 8.2 & 63.04 $\pm$ 8.5 & 58.93 $\pm$ 7.3 & 59.31 $\pm$ 8.2  & 61.29 $\pm$ 7.9\\
\cmidrule(lr){1-1} \cmidrule(lr){2-2} \cmidrule(lr){3-3} \cmidrule(lr){4-4} \cmidrule(lr){5-5} \cmidrule(lr){6-6} \cmidrule(lr){7-7} \cmidrule(lr){8-8} \cmidrule(lr){9-9} 
SVHN & PuriGAN &30.09 $\pm$ 4.2 & \textbf{23.98 $\pm$ 3.5} & \textbf{22.80 $\pm$ 3.1} & \textbf{21.67 $\pm$ 3.4} & \textbf{20.85 $\pm$ 3.0} & \textbf{20.51 $\pm$ 3.1} & \textbf{19.17 $\pm$ 3.0}\\
& PU-LSGAN & 30.09 $\pm$ 4.2 & 28.14 $\pm$ 4.9 & 28.26 $\pm$ 5.0 & 27.06 $\pm$ 4.8 & 27.32 $\pm$ 4.6 & 25.38 $\pm$ 4.3  & 23.61 $\pm$ 3.5 \\
& PU-NDA &30.31 $\pm$ 4.2& 24.01 $\pm$ 3.5 & 23.65 $\pm$ 3.6 & 23.40 $\pm$ 3.5 & 23.16 $\pm$ 3.4 & 22.49 $\pm$ 3.7  &21.62 $\pm$ 3.5\\
& NDA & 30.31 $\pm$ 4.2& 29.91 $\pm$ 4.4 & 29.48 $\pm$ 4.5 & 28.99 $\pm$ 4.3 & 29.31 $\pm$ 4.5 & 28.82 $\pm$ 4.2  &28.45 $\pm$ 4.2\\
& GenPU & 30.31 $\pm$ 4.2&  30.05 $\pm$ 5.7 & 29.10 $\pm$ 5.6 & 28.43 $\pm$ 5.3 & 28.75 $\pm$ 5.3 & 27.72 $\pm$ 5.2  &27.45 $\pm$ 3.4\\
\cmidrule(lr){1-1} \cmidrule(lr){2-2} \cmidrule(lr){3-3} \cmidrule(lr){4-4} \cmidrule(lr){5-5} \cmidrule(lr){6-6} \cmidrule(lr){7-7} \cmidrule(lr){8-8} \cmidrule(lr){9-9} 
CIFAR-10 & PuriGAN &63.79 $\pm$ 11.2 & \textbf{57.07 $\pm$ 10.5} & \textbf{56.41 $\pm$ 10.6} & \textbf{55.35 $\pm$ 10.5} & \textbf{54.72 $\pm$ 10.5} & \textbf{53.91 $\pm$ 10.5} & \textbf{53.25 $\pm$ 10.4}\\
& PU-LSGAN & 63.79 $\pm$ 11.2 & 61.96 $\pm$ 11.2 & 59.88 $\pm$ 11.0 & 59.43 $\pm$ 11.3 & 58.32 $\pm$ 10.6 & 56.6 $\pm$ 10.7  & 55.03 $\pm$ 10.6 \\
& PU-NDA &68.35 $\pm$ 11.0& 72.03 $\pm$ 11.8 & 70.64 $\pm$ 11.6 & 69.61 $\pm$ 11.5 & 68.84 $\pm$ 11.4 & 69.01 $\pm$ 11.6  &68.47 $\pm$ 11.5\\
& NDA & 68.35 $\pm$ 11.0& 76.88 $\pm$ 12.4 & 74.52 $\pm$ 11.7 & 73.03 $\pm$ 11.0 & 72.19 $\pm$ 11.2 & 72.82 $\pm$ 11.2  & 71.98 $\pm$ 11.2\\
& GenPU & 68.35 $\pm$ 11.0 &  66.9 $\pm$ 10.7 & 65.72 $\pm$ 10.4 & 65.06 $\pm$ 10.3 & 64.38 $\pm$ 10.4 & 63.65 $\pm$ 10.2  &62.89 $\pm$ 10.2\\

\bottomrule
\end{tabular}
\caption{Complete FID scores$\downarrow$ as a function of $\gamma_c$, the ratio between the collected contamination instances and target instances on different datasets.}
\label{tablel}}
\vspace{5mm}
\end{table*}

\begin{table*}[!t]
\centering
\footnotesize
\setlength\tabcolsep{2.6pt}
{\begin{tabular}{cccccccc}
\toprule
\makecell[c]{Data} & \makecell[c]{$k$}  & \makecell[c]{LSGAN} & \makecell[c]{NDA} & \makecell[c]{GenPU} & \makecell[c]{PU-LSGAN} & \makecell[c]{PU-NDA} & \makecell[c]{PuriGAN} \\
\cmidrule(lr){1-1} \cmidrule{2-2} \cmidrule(lr){3-3} \cmidrule(lr){4-4} \cmidrule(lr){5-5} \cmidrule(lr){6-6} \cmidrule(lr){7-7} \cmidrule(lr){8-8} 
MNIST & 1  & 30.83 $\pm$ 9.0  & 31.56  $\pm$ 11.6 & 23.10 $\pm$ 8.2 & 15.04 $\pm$ 3.1 & \textbf{10.45 $\pm$ 3.0} & 10.68 $\pm$ 1.6\\
& 2  & 27.56 $\pm$ 5.6 & 29.47 $\pm$ 9.6 & 21.32 $\pm$ 8.8 & 11.85 $\pm$ 2.4 &15.17 $\pm$ 4.7 & \textbf{10.70 $\pm$ 1.3} \\
 & 3  & 23.30 $\pm$ 4.7 & 30.37 $\pm$ 6.0 & 19.47 $\pm$ 8.0 & 12.76 $\pm$ 1.5 &15.90 $\pm$ 1.6 &\textbf{10.54 $\pm$ 1.2} \\
& 5  &20.86 $\pm$ 3.1 & 26.22 $\pm$ 3.5  & 16.68 $\pm$ 6.1 & 13.59 $\pm$ 1.8 & 17.03 $\pm$ 1.5 &\textbf{11.97 $\pm$ 1.6} \\
\cmidrule(lr){1-1} \cmidrule{2-2} \cmidrule(lr){3-3} \cmidrule(lr){4-4} \cmidrule(lr){5-5} \cmidrule(lr){6-6} \cmidrule(lr){7-7} \cmidrule(lr){8-8}
F-MNIST & 1  & 73.87 $\pm$ 9.0 & 61.47 $\pm$ 11.6 & 63.04 $\pm$ 8.5 & 46.06 $\pm$ 5.8 & 43.55 $\pm$ 5.1 & \textbf{38.18 $\pm$ 3.1}\\
& 2 & 76.71 $\pm$ 14.8 & 76.31 $\pm$ 15.5 & 73.30 $\pm$ 10.2 & 41.32 $\pm$ 6.3 & 62.32 $\pm$ 8.3 & \textbf{39.88 $\pm$ 6.7} \\
 & 3 & 79.89 $\pm$ 16.5 & 60.66 $\pm$ 13.6 & 85.16 $\pm$ 11.2 & 39.27 $\pm$ 7.9 &49.31 $\pm$ 5.9 & \textbf{37.62 $\pm$ 8.3} \\
& 5  & 68.78 $\pm$ 17.5 & 57.53 $\pm$ 9.7 & 78.68 $\pm$ 9.5 & 45.63 $\pm$ 8.0 & 53.02 $\pm$ 4.9  & \textbf{40.12 $\pm$ 7.2}\\
\cmidrule(lr){1-1} \cmidrule{2-2} \cmidrule(lr){3-3} \cmidrule(lr){4-4} \cmidrule(lr){5-5} \cmidrule(lr){6-6} \cmidrule(lr){7-7} \cmidrule(lr){8-8} 
SVHN & 1  & 30.09 $\pm$ 4.9 & 28.99 $\pm$ 4.3 & 28.43 $\pm$ 5.3 & 27.06 $\pm$ 4.8 & 23.40 $\pm$ 3.5 & \textbf{21.67 $\pm$ 3.4}\\
& 2  & 22.36 $\pm$ 5.8 & 22.67 $\pm$ 5.1 & 23.77 $\pm$ 6.0 & 20.18 $\pm$ 3.0 & 18.95 $\pm$ 3.8 & \textbf{17.05 $\pm$ 3.4} \\
& 3  & 18.76 $\pm$ 2.1 & 17.63 $\pm$ 2.1 & 20.11 $\pm$ 3.1 & 16.13 $\pm$ 1.9 &16.50 $\pm$ 2.4  &\textbf{12.73 $\pm$ 1.5} \\
& 5  &19.41 $\pm$ 2.2 & 16.25 $\pm$ 1.3 &15.68 $\pm$ 1.5 & 15.60 $\pm$ 2.0 & 14.89 $\pm$ 1.4 &\textbf{11.96 $\pm$ 1.5}\\
\cmidrule(lr){1-1} \cmidrule{2-2} \cmidrule(lr){3-3} \cmidrule(lr){4-4} \cmidrule(lr){5-5} \cmidrule(lr){6-6} \cmidrule(lr){7-7} \cmidrule(lr){8-8} 
CIFAR-10 & 1  & 63.79 $\pm$ 11.2 & 73.03 $\pm$ 11.0 & 65.06 $\pm$ 10.3 & 59.43 $\pm$ 11.3 & 69.61 $\pm$ 11.5 & \textbf{55.35 $\pm$ 10.5}\\
& 2  & 39.81 $\pm$ 2.4 & 48.90 $\pm$ 5.0 & 38.60 $\pm$ 3.7 & 37.50 $\pm$ 2.6 & 45.69 $\pm$ 5.4 & \textbf{35.41 $\pm$ 2.4} \\
& 3  & 44.76 $\pm$ 8.5 & 55.27 $\pm$ 11.1 & 43.44 $\pm$ 7.7 & 42.29 $\pm$ 8.9 & 49.36 $\pm$ 9.2  &\textbf{40.92 $\pm$ 7.8} \\
& 5  &57.92 $\pm$ 11.4 & 68.76 $\pm$ 12.5 &53.42 $\pm$ 11.1 & 53.04 $\pm$ 11.0 & 60.45 $\pm$ 11.9 &\textbf{51.64 $\pm$ 10.8}\\
\bottomrule
\end{tabular}
\caption{Complete FID scores$\downarrow$ as a function of the number of classes of target instances on different datasets.}

\label{tablek}}
\end{table*}

\paragraph{Investigate the Sensitivity to $\lambda$}
In this paragraph, we discuss the effect of different $\lambda$ on the model generation results. Theoretically, for the parameter $\lambda$, it should be set very large. But in practice, if $\lambda$ is set too large, the classifier is likely to classify all instances to 0, which, obviously, will weaken the discriminator’s ability to distinguish between the target and contamination instances. Moreover, a large $\lambda$ may affects the convergence speed of the generative adversarial network. Therefore, the value of $\lambda$ should be set by taking into account the influences of these factors. We list the results of different $\lambda$ values under different scenarios in Table \ref{lambda_re}. Overall, it can be seen that under the considered range $1\le \lambda\le 5$, the observed performances do not differ too much, approximately on the same level. Thus, the performance is not very sensitive to the specific value of $\lambda$, as long as it is not set too large or too small. The experiment results reveal that our model has some potential to achieve better performance with the more detailed tuning of the parameter $\lambda$.
 
\paragraph{Investigate the Sensitivity to $c$}
In this paragraph, we discuss the effect of different $c$ on the model generation results. Theoretically, 
the values of parameter $c$ has no influence on the convergence results of the theorem. We conduct the experiments and list the results of different $c$ value under different scenarios in Table \ref{c_re}. Overall, it can be observed that the performances do not differ too much. Therefore, an arbitrary choice of parameter $c$ can achieve satisfactory results.
 
\paragraph{Complete Results for the Impact of Contamination Ratio $\gamma_p$} Table \ref{Overall_fidscore_app} shows the complete FID scores under various $\gamma_p$.  It can be seen our model outperforms other baselines under almost all scenarios. NDA and LSGAN show a dramatic decrease in performance as the contamination ratio increases. This may because they do not have any ability to counter the influences of contamination instances.
Two-stage approaches performs more closely to our model. This phenomenon is because this kind of method employs a PU-learning classifier to remove contamination instances in the first stage. A remarkable performance gap between the PU-NDA and PU-LSGAN can be observed, which can be attributed to the ability of NDA to leverage collected contamination instances to boost generation performance, although it is primarily designed to work on clean datasets. The GAN-based PU-learning method GenPU can partially able to work on a contaminated dataset, but it focuses more on discriminative learning rather than generating learning. Moreover, GenPU incorporates an array of generators and discriminators, which require very careful tuning of their parameters and weights during the training phase.

\paragraph{Complete Results for the Impact of the Ratio of Collected Contamination Instances $\gamma_c$}
To investigate how the number of collected contamination instances affects the generation performance, we conduct experiments under different $\gamma_c$ where $\gamma_c$ is defined as the ratio between the number of available contamination instances and total desirable instances in the training dataset. The contamination ratio $\gamma_p$ is fixed to 0.5 under which the two PuriGAN variants become equivalent. Thus we only list the performance of one PuriGAN in Table \ref{tablel}. It can be seen that our model consistently outperforms the baseline models. When $\gamma_c$ is set to 0, PuriGAN and PU-LSGAN are reduced to unsupervised LSGAN, and other baselines are reduced to unsupervised original GAN, which can not leverage any collected contaminated information for better performance. As long as a small fraction of collected contamination instances are given, a significant performance improvement of PuriGAN can be observed. As the number of collected instances of contamination further increases, the performance of our proposed model can be steadily improved. These phenomena demonstrate the effectiveness of PuriGAN in leveraging the extra contamination instances to counter the influences of contamination.

\begin{table*}[!t]
\setlength\tabcolsep{2.5pt} 
\centering
\small
{
\begin{tabular}{ccccc|ccccc}
\toprule
\makecell[c]{\\Data} & \makecell[c]{\\$\gamma_p$} & \makecell[c]{Deep\\ SVDD} & \makecell[c]{CAE \\} & \makecell[c]{LSGAN \\ } & \makecell[c]{Deep\\ SAD } & \makecell[c]{\makecell[c]{Supervised\\ Classifier}} & \makecell[c]{NDA \\}& \makecell[c]{PuriGAN\\}\\
\cmidrule(lr){1-1} \cmidrule(lr){2-2} \cmidrule(lr){3-3} \cmidrule(lr){4-4} \cmidrule(lr){5-5} \cmidrule(lr){6-6} \cmidrule(lr){7-7} \cmidrule(lr){8-8} \cmidrule(lr){9-9} \cmidrule(lr){10-10} 
& .00 & 92.8 $\pm$ 4.9 & 92.9 $\pm$ 5.7 &91.8 $\pm$ 6.4 & \textbf{99.7 $\pm$ 0.3} & 98.1 $\pm$ 0.8 & \textbf{99.7 $\pm$ 0.2} & \textbf{99.7 $\pm$ 0.3} \\
& .10 & 86.5 $\pm$ 6.8 & 83.7 $\pm$ 8.4 & 75.7 $\pm$ 10.7 & 98.6 $\pm$ 1.1 & 96.1 $\pm$ 1.4 & 98.8 $\pm$ 0.7 & \textbf{99.4 $\pm$ 0.3} \\
MNIST & .20  & 81.5 $\pm$ 8.4 & 78.6 $\pm$ 10.3 & 70.3 $\pm$ 13.7 & 98.2 $\pm$ 1.3 & 94.5 $\pm$ 1.6 & 98.1 $\pm$ 0.9 & \textbf{99.3 $\pm$ 0.3} \\
& .30  & 77.0 $\pm$ 9.7 & 72.3 $\pm$ 11.1 &68.0 $\pm$ 12.2 & 97.5 $\pm$ 1.3 & 92.2 $\pm$ 2.7 & 97.2 $\pm$ 1.1 & \textbf{99.0 $\pm$ 0.4}\\
\cmidrule(lr){1-1} \cmidrule(lr){2-2} \cmidrule(lr){3-3} \cmidrule(lr){4-4} \cmidrule(lr){5-5} \cmidrule(lr){6-6} \cmidrule(lr){7-7} \cmidrule(lr){8-8} \cmidrule(lr){9-9} 
& .00  & 89.2 $\pm$ 6.2 & 90.2 $\pm$ 5.8  &90.2 $\pm$ 7.7 & 97.1 $\pm$ 3.7 & 94.9 $\pm$ 4.0 &96.6 $\pm$ 4.1 & \textbf{97.7 $\pm$ 2.5} \\
& .10  & 76.2 $\pm$ 6.3 & 77.4 $\pm$ 11.1 &79.9 $\pm$ 9.7& 94.7 $\pm$ 4.2 & 93.6 $\pm$ 4.5 &89.1 $\pm$ 8.0 & \textbf{97.3 $\pm$ 3.0} \\
F-MNIST & .20 & 69.3 $\pm$ 7.3 & 72.4 $\pm$ 12.6 & 77.9 $\pm$ 8.3 & 93.4 $\pm$ 5.2 & 92.8 $\pm$ 4.8 & 88.3 $\pm$ 9.0 & \textbf{97.0 $\pm$ 3.1} \\
& .30 & 64.5 $\pm$ 7.9 & 68.1 $\pm$ 13.5 & 76.7 $\pm$ 9.8 & 92.1 $\pm$ 4.0 & 91.5 $\pm$ 5.0 & 87.3 $\pm$ 9.2 & \textbf{96.8 $\pm$ 3.4}\\
\cmidrule(lr){1-1} \cmidrule(lr){2-2} \cmidrule(lr){3-3} \cmidrule(lr){4-4} \cmidrule(lr){5-5} \cmidrule(lr){6-6} \cmidrule(lr){7-7} \cmidrule(lr){8-8} \cmidrule(lr){9-9} 
& .00  & 60.9 $\pm$ 9.4 & 56.2 $\pm$ 13.2 & 65.4 $\pm$ 4.1 & \textbf{87.9 $\pm$ 4.9} & 81.2 $\pm$ 2.8 & 78.9 $\pm$ 5.6 & 86.7 $\pm$ 5.4 \\
& .10  & 58.6 $\pm$ 10.0 & 55.4 $\pm$ 13.3 & 62.6 $\pm$ 5.0 & 82.5 $\pm$ 5.3 & 75.1 $\pm$ 5.2 & 75.3 $\pm$ 5.6 & \textbf{84.7 $\pm$ 5.4} \\
CIFAR-10 & .20  & 57.0 $\pm$ 10.6 & 54.6 $\pm$ 13.3 & 61.3 $\pm$ 6.1 &79.4 $\pm$ 5.2 & 73.8 $\pm$ 4.9  & 72.1 $\pm$ 8.5 & \textbf{84.3 $\pm$ 5.5} \\
& .30  & 55.5 $\pm$ 11.0  & 53.8 $\pm$ 13.3  & 60.7 $\pm$ 6.1 & 76.9 $\pm$ 5.5 & 70.9 $\pm$ 4.2  & 70.1 $\pm$ 8.5 & \textbf{83.3 $\pm$ 5.4} \\
\bottomrule
\end{tabular}
\vspace{-1.2mm}
\caption{AUROC of semi-supervised anomaly detection on different datasets under different contamination ratio of $\gamma_p$ and fixed $\gamma_c=0.05$.}

\label{tableadp}}

\end{table*}

\begin{figure*}[!tb]
    \centering
    \begin{minipage}{0.32\linewidth}
        \centering
        \includegraphics[width=1\linewidth]{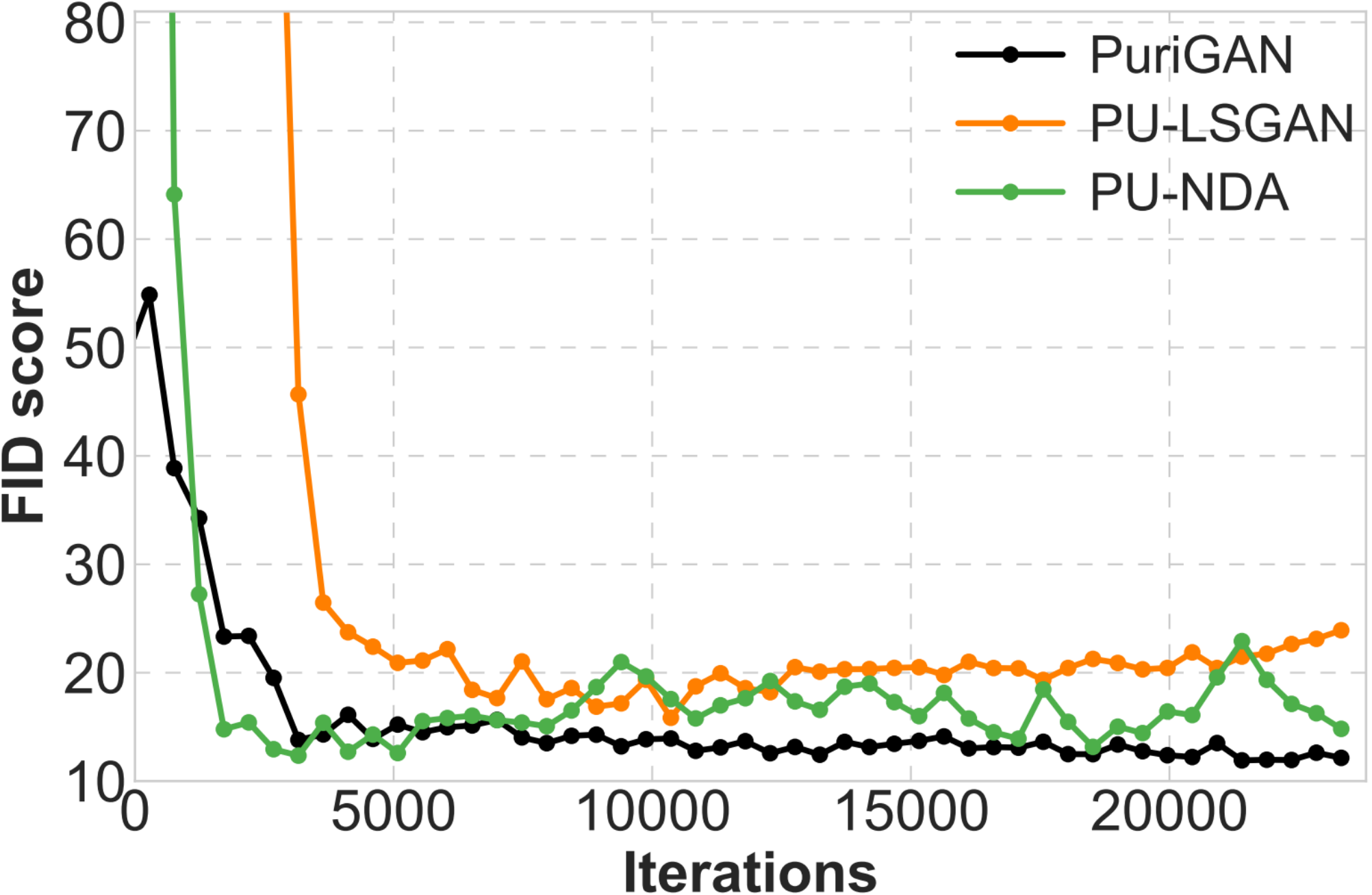}
        \caption*{\scriptsize MNIST}
    \end{minipage}
    \begin{minipage}{0.32\linewidth}
        \centering
        \includegraphics[width=1\linewidth]{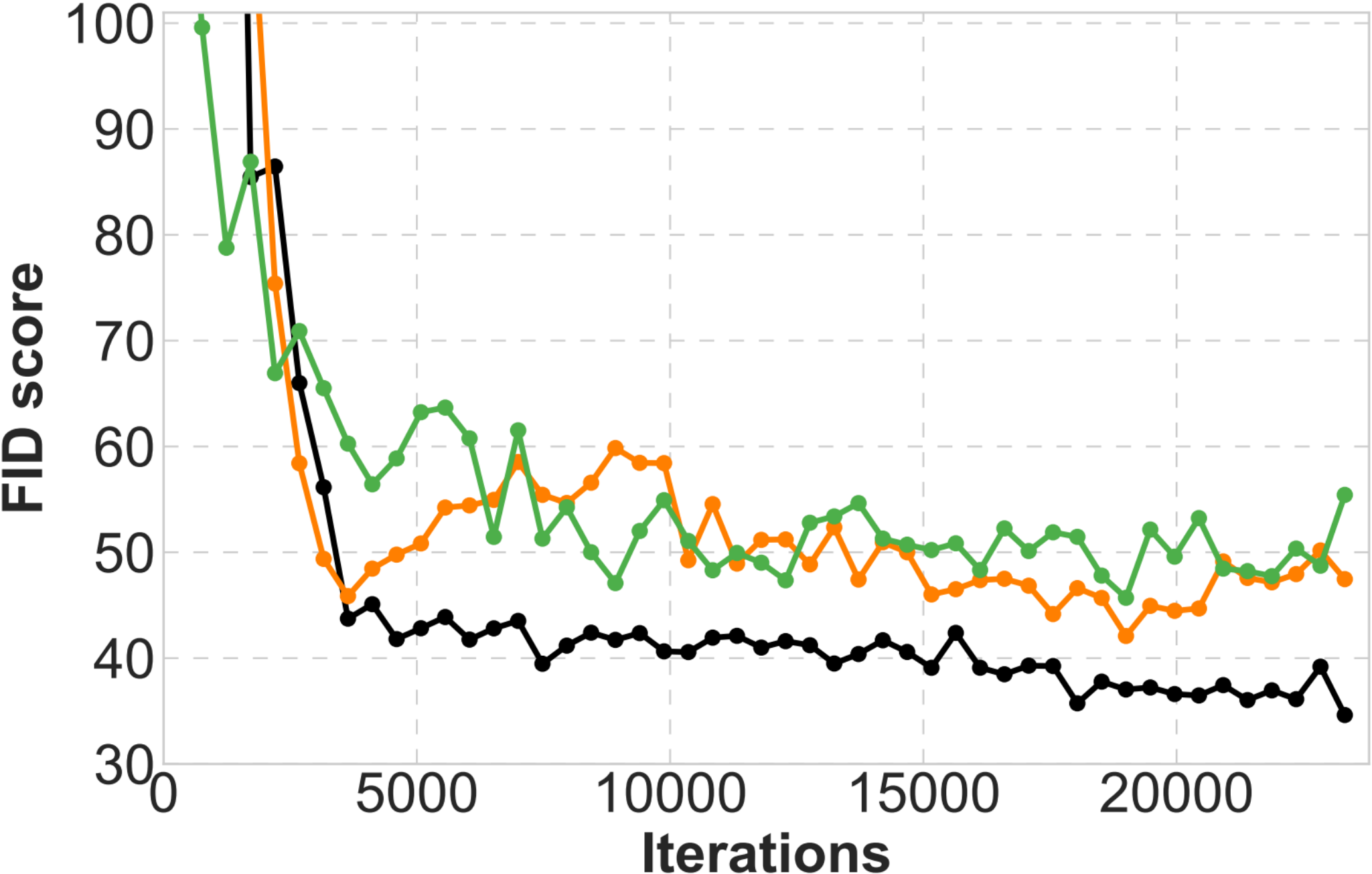}
        \caption*{\scriptsize F-MNIST}
    \end{minipage}
    \begin{minipage}{0.32\linewidth}
        \centering
        \includegraphics[width=1\linewidth]{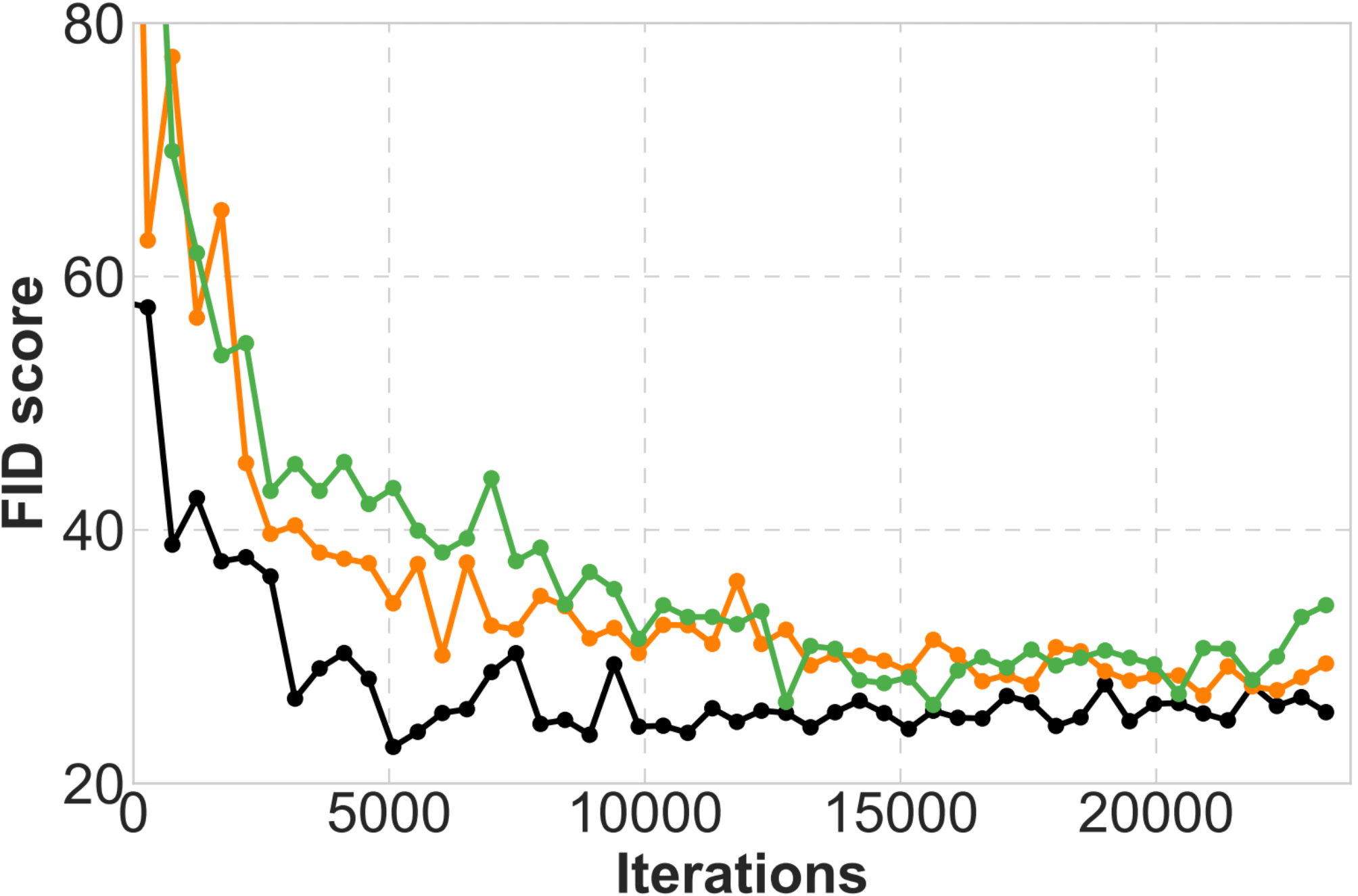}
        \caption*{\scriptsize SVHN}
    \end{minipage}
\caption{Illustration of FID variation as iterations progress on under the scenario of $\gamma_c=0.2$ and $\gamma_p=0.5$.}
\label{FID-iterations}
\vspace{8mm}
\end{figure*}

\begin{figure*}[t]
\begin{flushleft}
\setlength\tabcolsep{1.9pt}
  \begin{tabular}[b]{p{.02\linewidth}|p{.15\linewidth}|p{.15\linewidth}|p{.15\linewidth}|p{.15\linewidth}|p{.15\linewidth}|p{.15\linewidth}}
  &\makecell[c]{LSGAN} & \makecell[c]{NDA} &\makecell[c]{GenPU}& \makecell[c]{PU-LSGAN}& \makecell[c]{PU-NDA}& \makecell[c]{PuriGAN}  \\
  \multirow{2}{*}[4.3em]{\rotatebox{90}{MNIST}} &
    \includegraphics[width=0.99\linewidth]{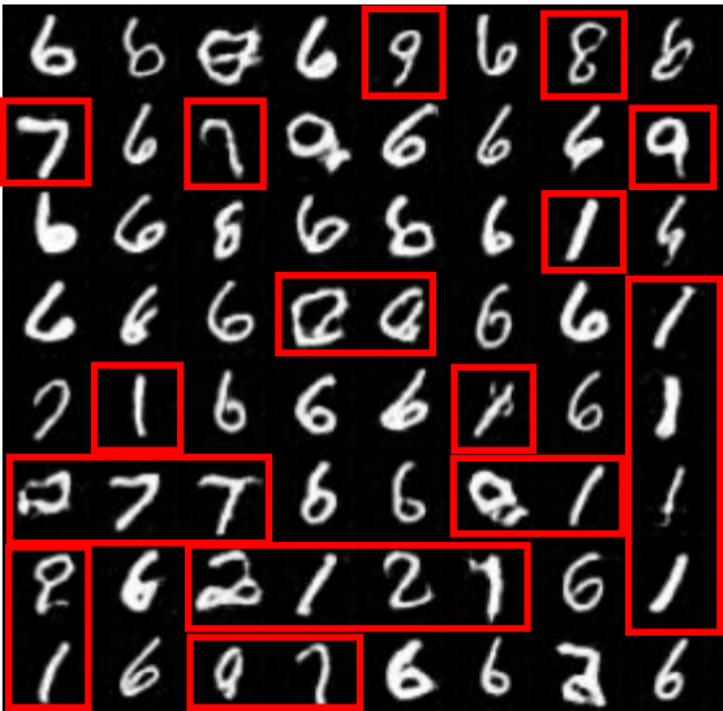} &
     \includegraphics[width=0.99\linewidth]{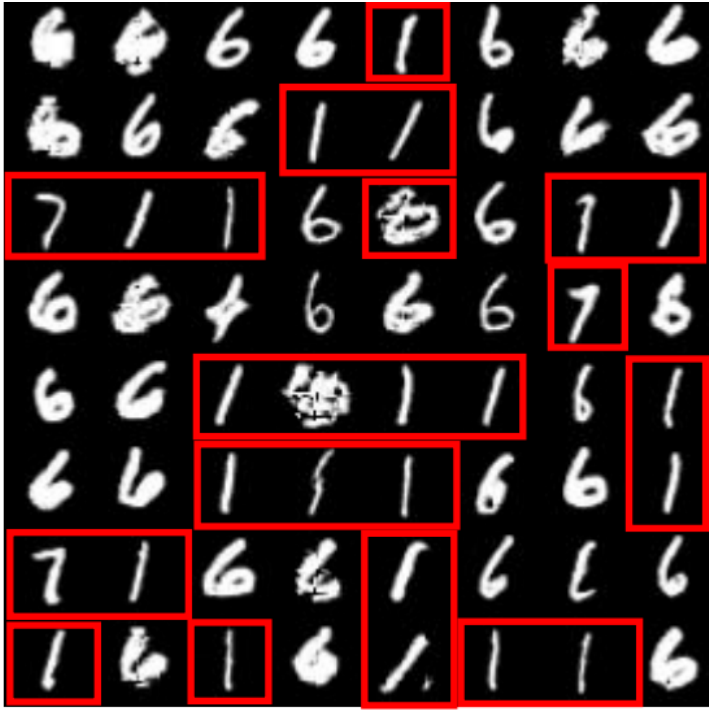} &
    \includegraphics[width=0.99\linewidth]{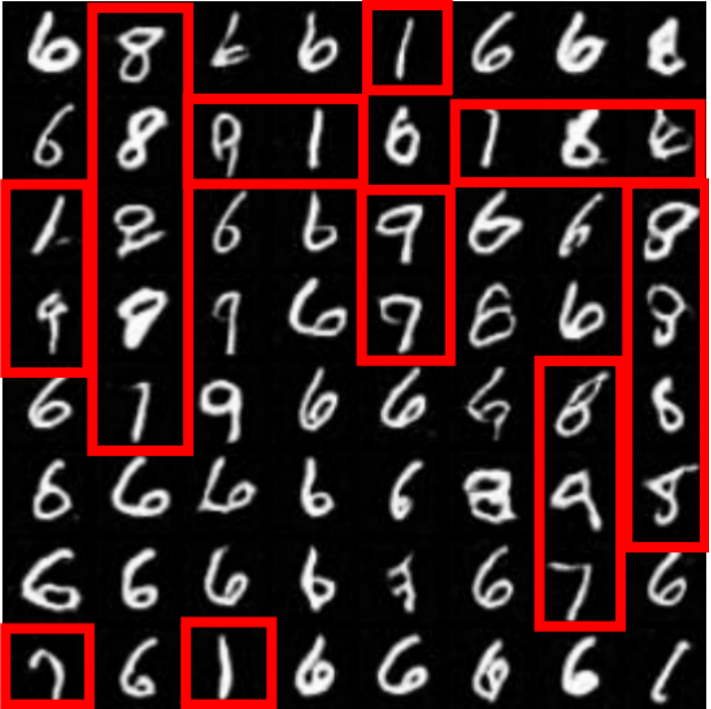} &

      \includegraphics[width=0.99\linewidth]{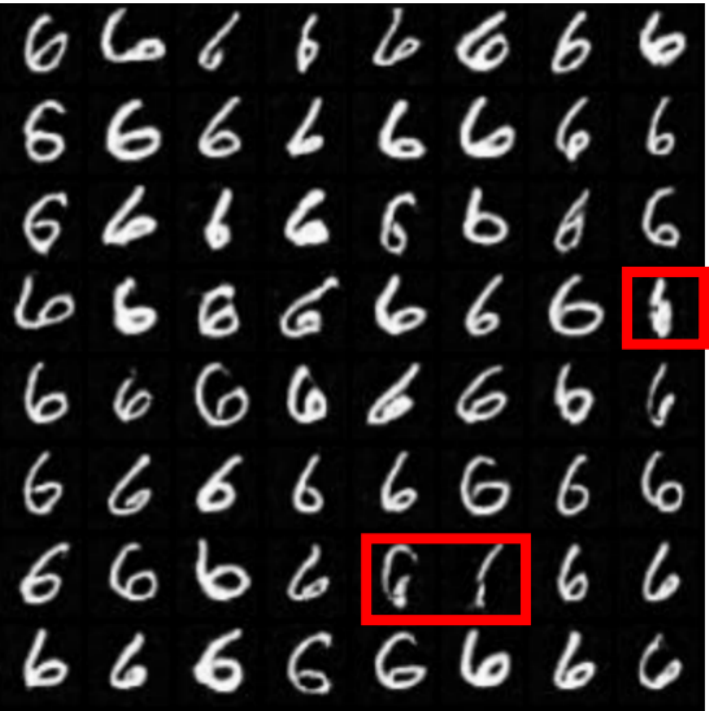}&
      \includegraphics[width=0.99\linewidth]{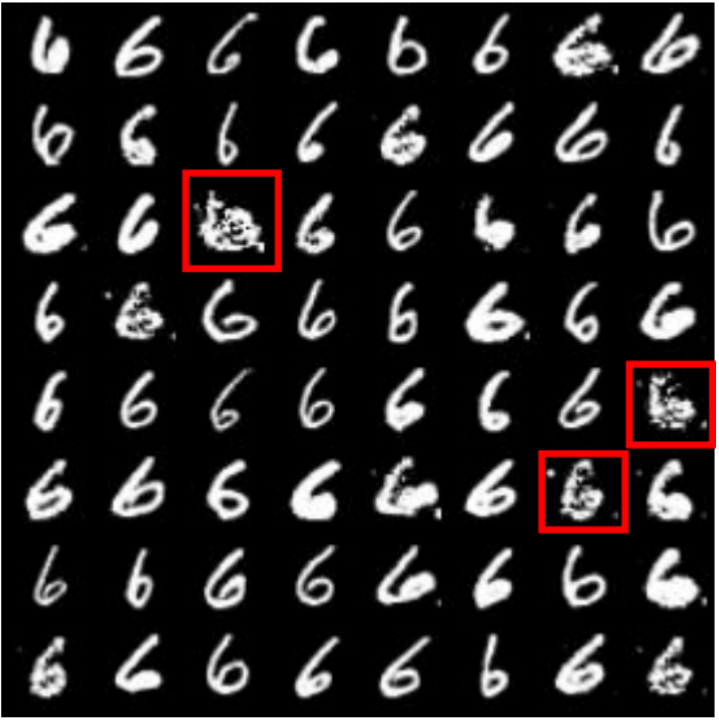}&
      \includegraphics[width=0.99\linewidth]{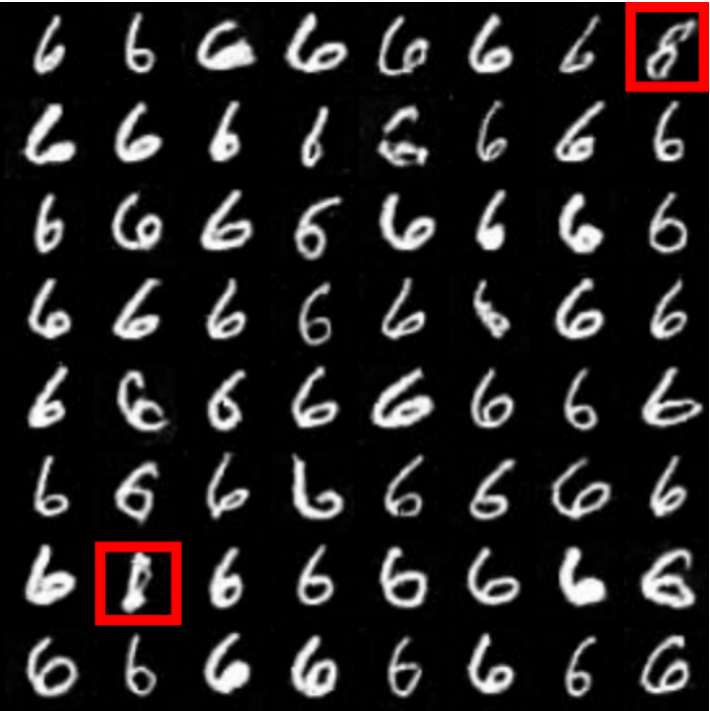}
      \\
    &\makecell[c]{FID = 26.98}  & \makecell[c]{FID = 46.76} & \makecell[c]{FID =28.21 } & \makecell[c]{FID =13.83 } &
    \makecell[c]{FID = 13.72} &
    \makecell[c]{FID = 11.71}\\[1pt]
    \multirow{2}{*}[4.3em]{\rotatebox{90}{F-MNIST}} &
    \includegraphics[width=0.99\linewidth]{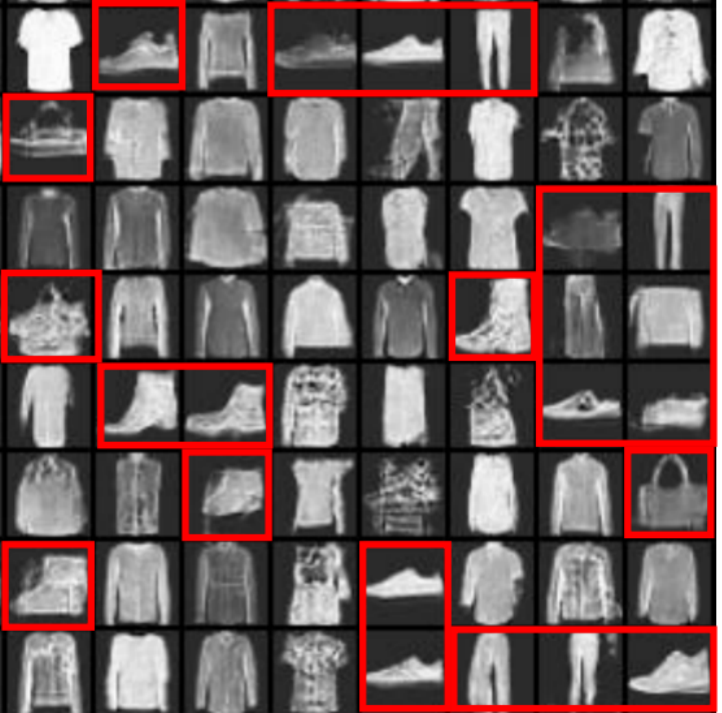} & 
    \includegraphics[width=0.99\linewidth]{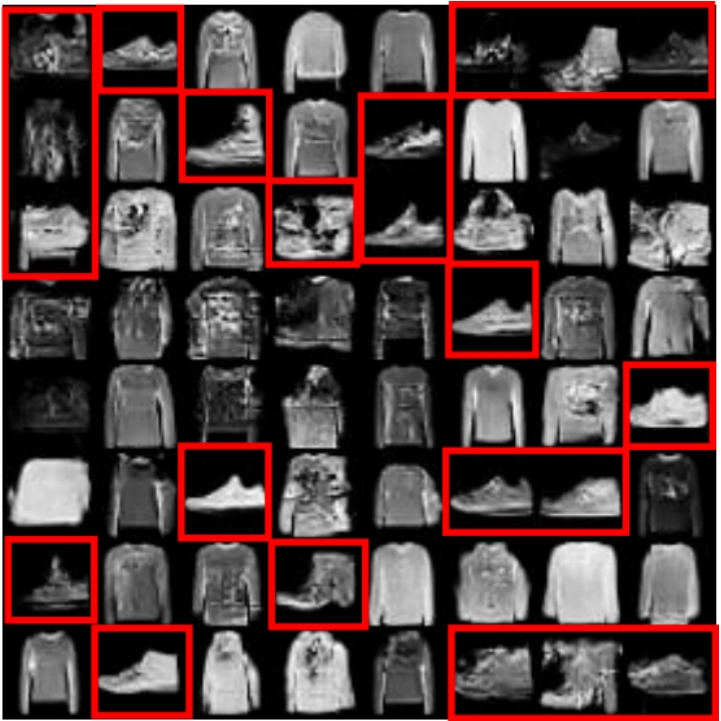}&
    \includegraphics[width=0.99\linewidth]{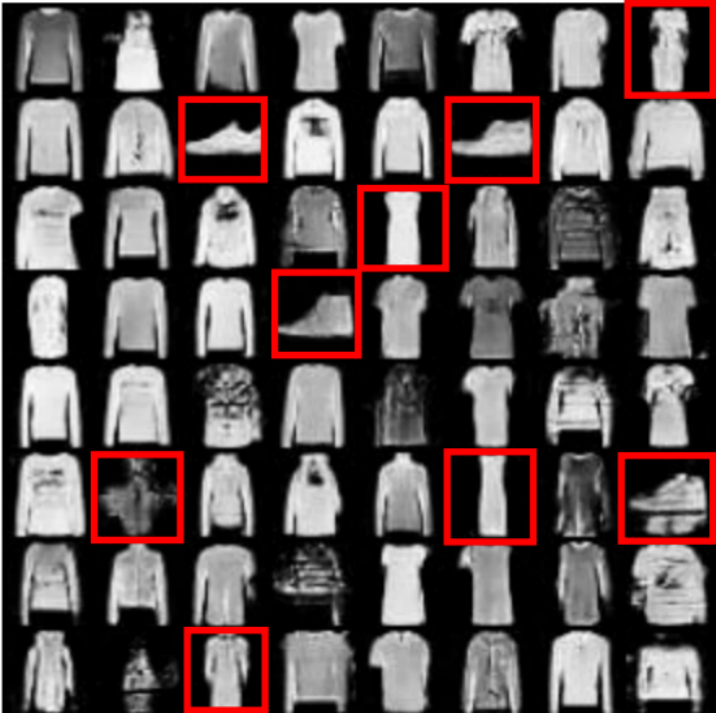}
    &\includegraphics[width=0.99\linewidth]{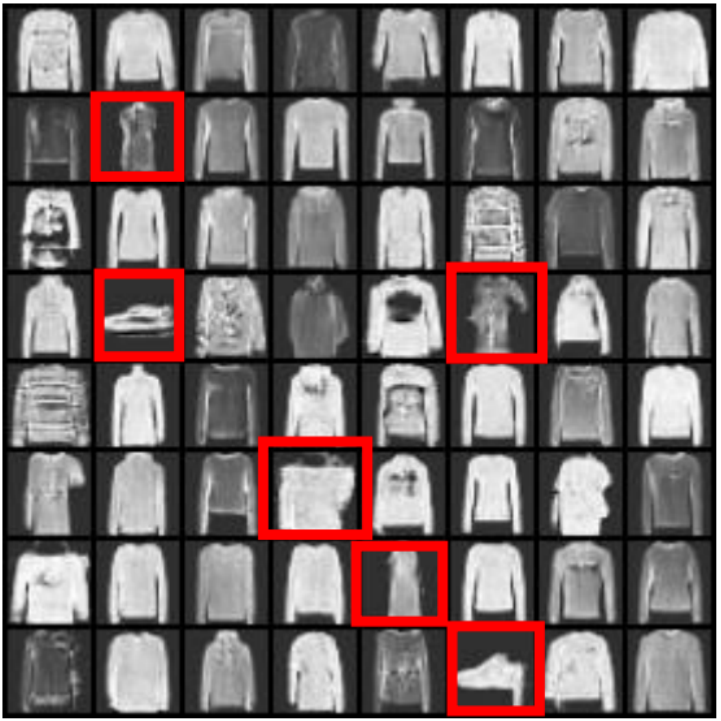} &
    \includegraphics[width=0.99\linewidth]{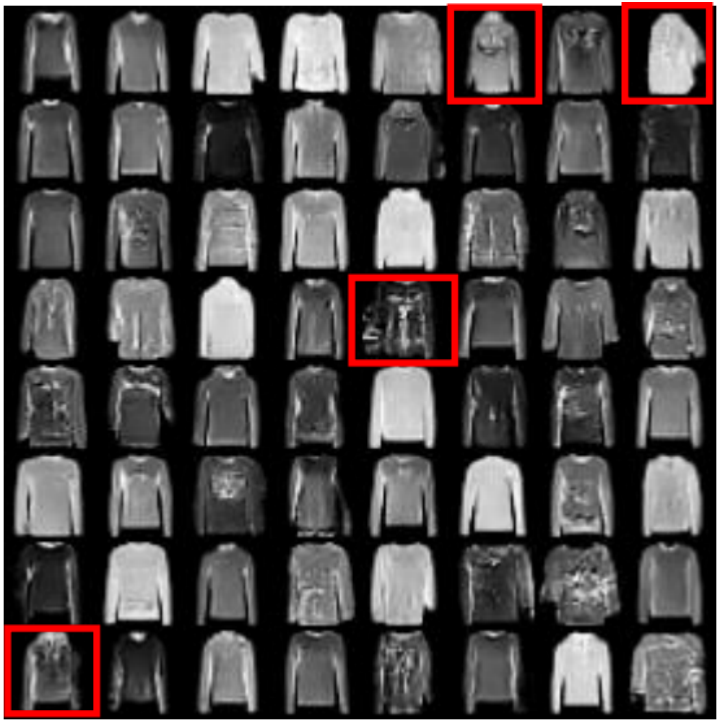}&
    \includegraphics[width=0.99\linewidth]{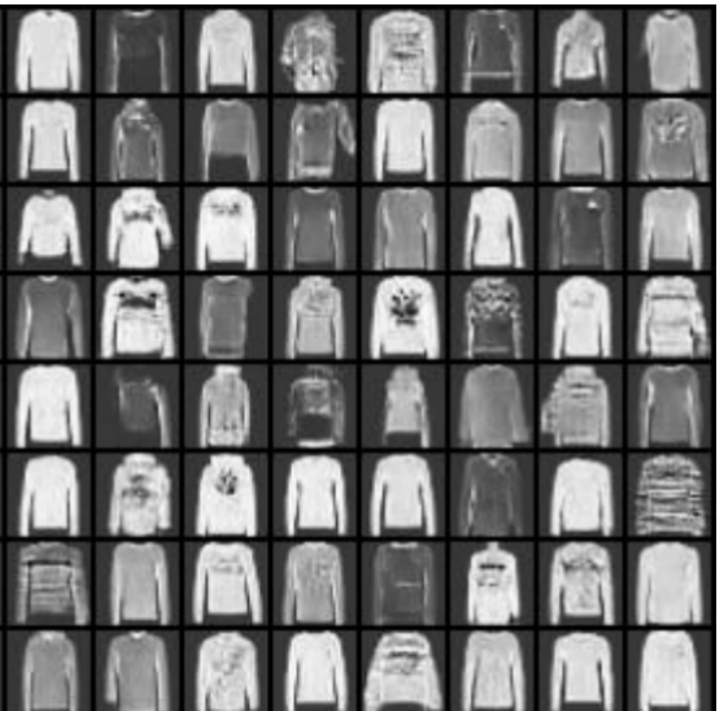}\\
    &\makecell[c]{FID = 65.84} 
    & \makecell[c]{FID = 68.99} 
    & \makecell[c]{FID = 61.80} 
    & \makecell[c]{FID = 44.41}
    & \makecell[c]{FID = 42.38}
    &\makecell[c]{FID = 34.56} 
     \\[2pt]
     \multirow{2}{*}[4.3em]{\rotatebox{90}{SVHN}}  &
     \includegraphics[width=0.99\linewidth]{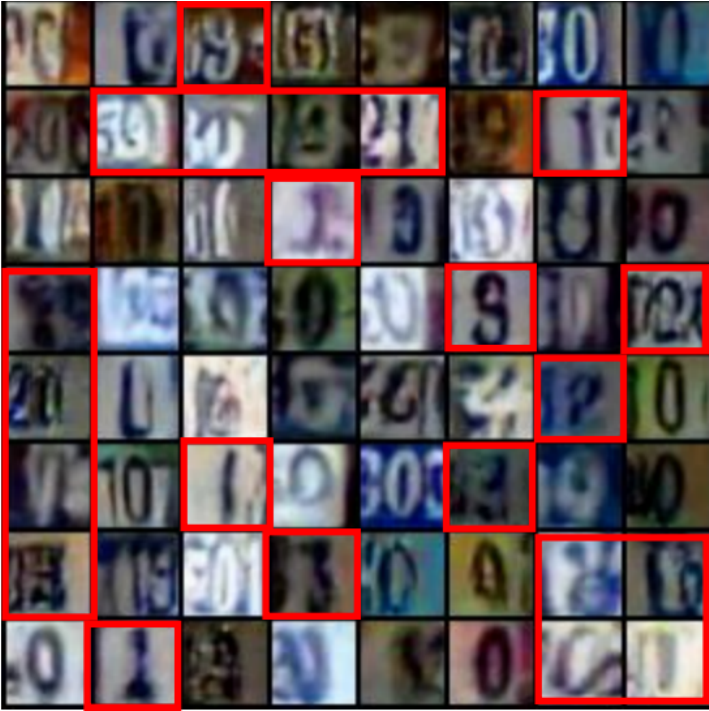} & 
     \includegraphics[width=0.99\linewidth]{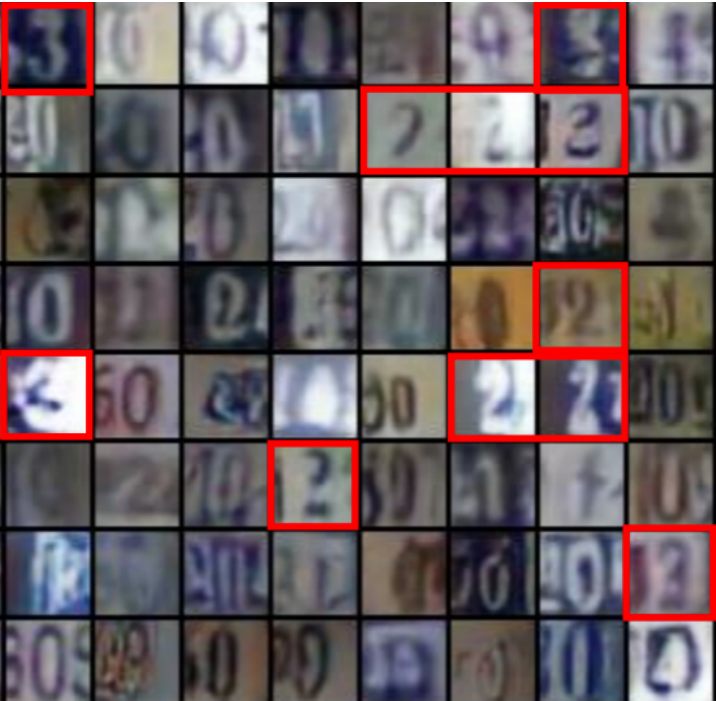} &
     \includegraphics[width=0.99\linewidth]{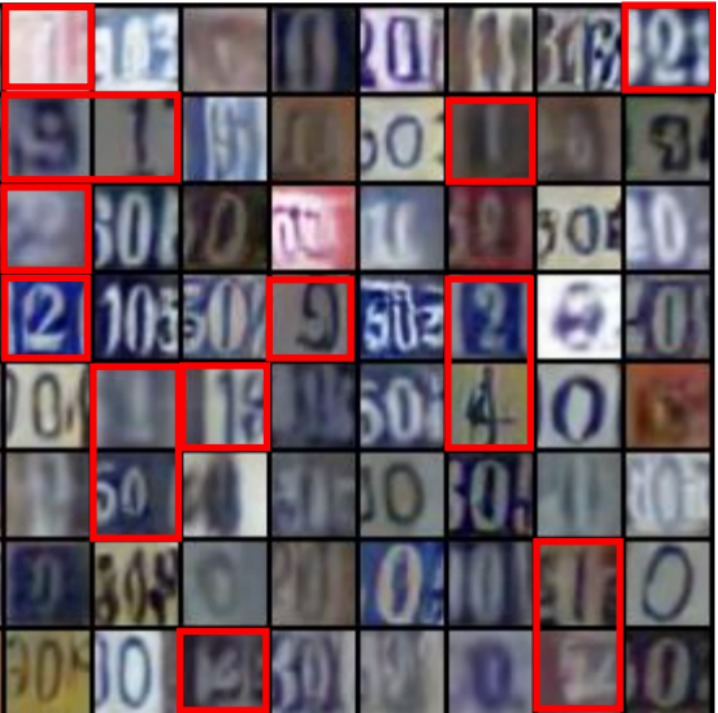}&
     \includegraphics[width=0.99\linewidth]{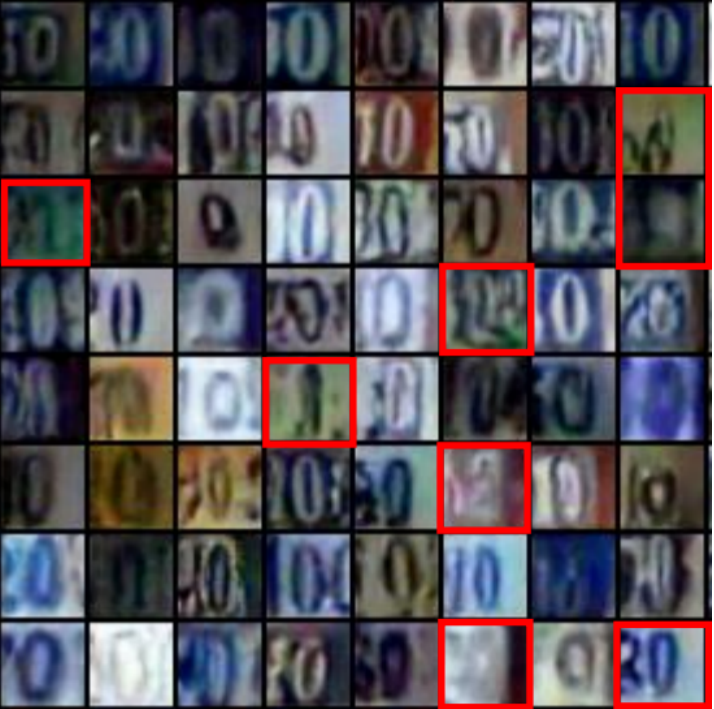} &
      \includegraphics[width=0.99\linewidth]{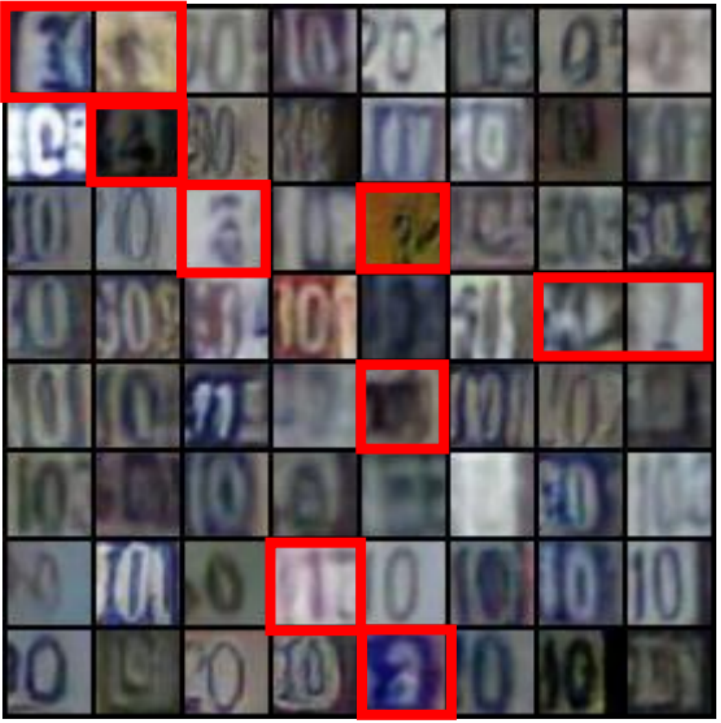} &
      \includegraphics[width=0.99\linewidth]{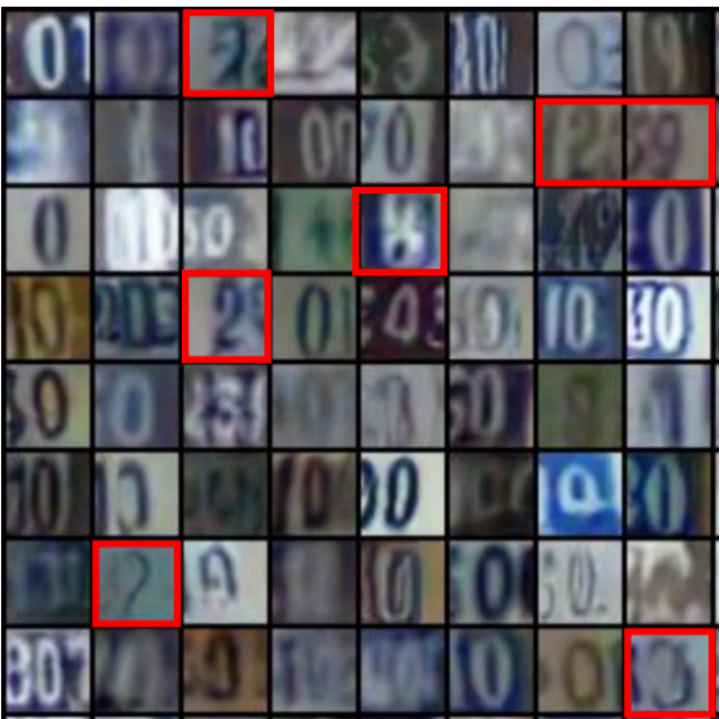}\\
    &\makecell[c]{FID = 28.98} 
    &\makecell[c]{FID = 26.74}
     &\makecell[c]{FID = 25.23} 
    &\makecell[c]{FID = 23.84}
    &\makecell[c]{FID = 22.68}
    &\makecell[c]{FID = 20.09}\\
    \multirow{2}{*}[4.3em]{\rotatebox{90}{CelebA}}  &
    \includegraphics[width=0.99\linewidth]{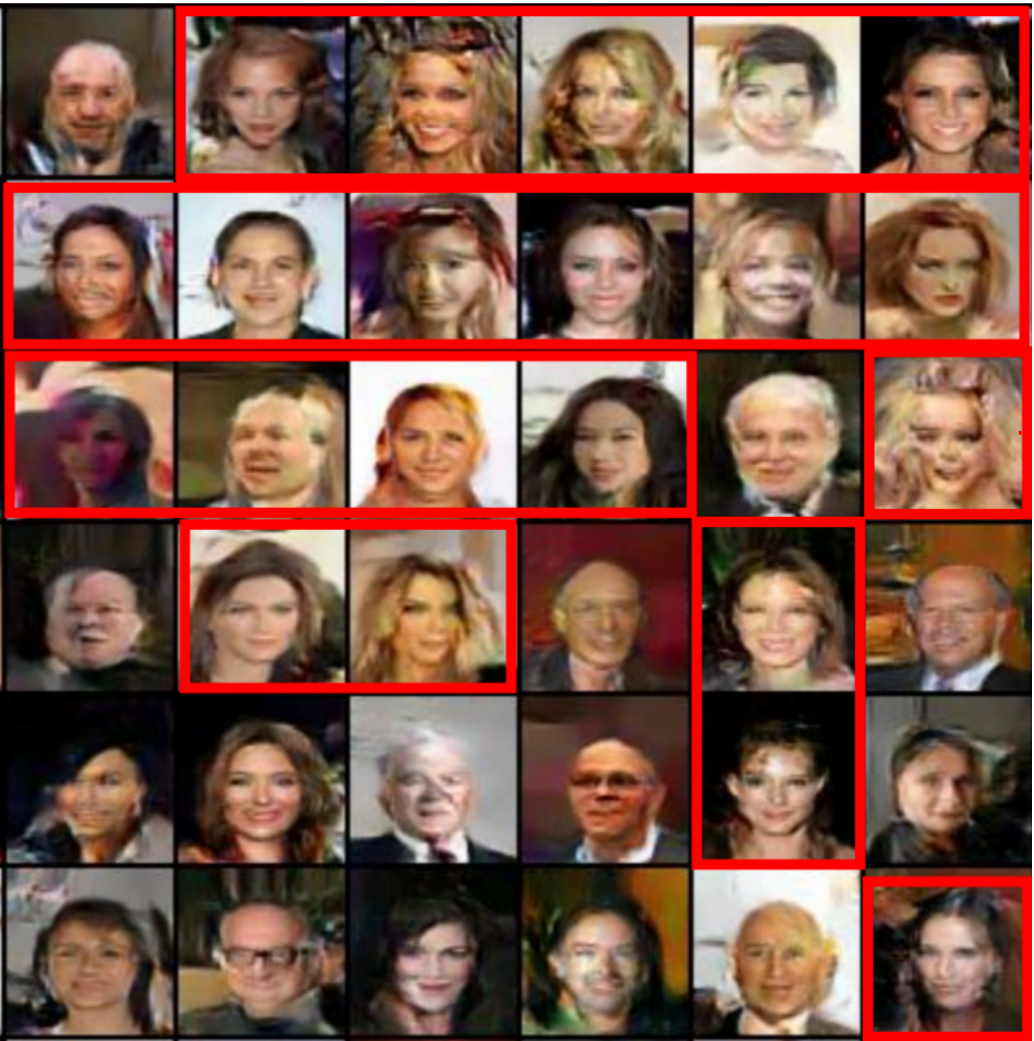} & 
    \includegraphics[width=0.99\linewidth]{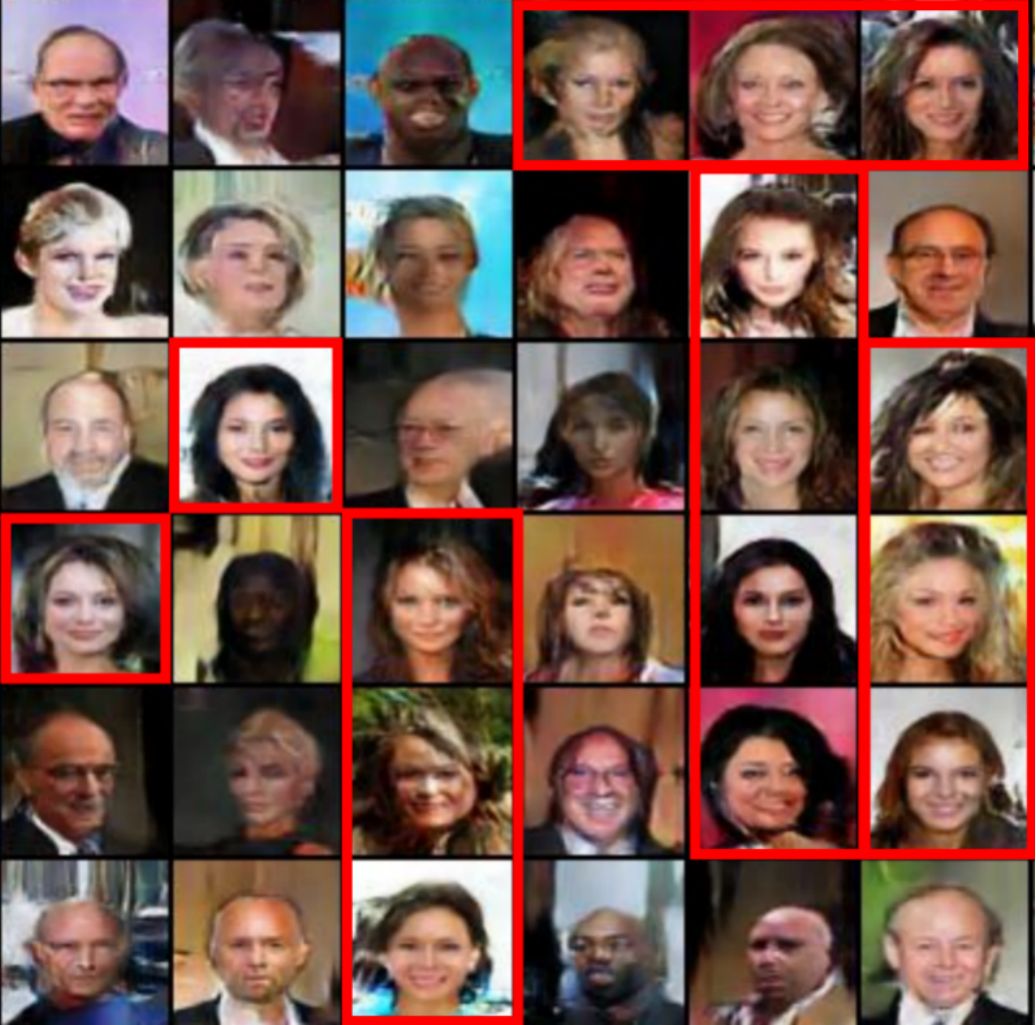}  &
    \includegraphics[width=0.99\linewidth]{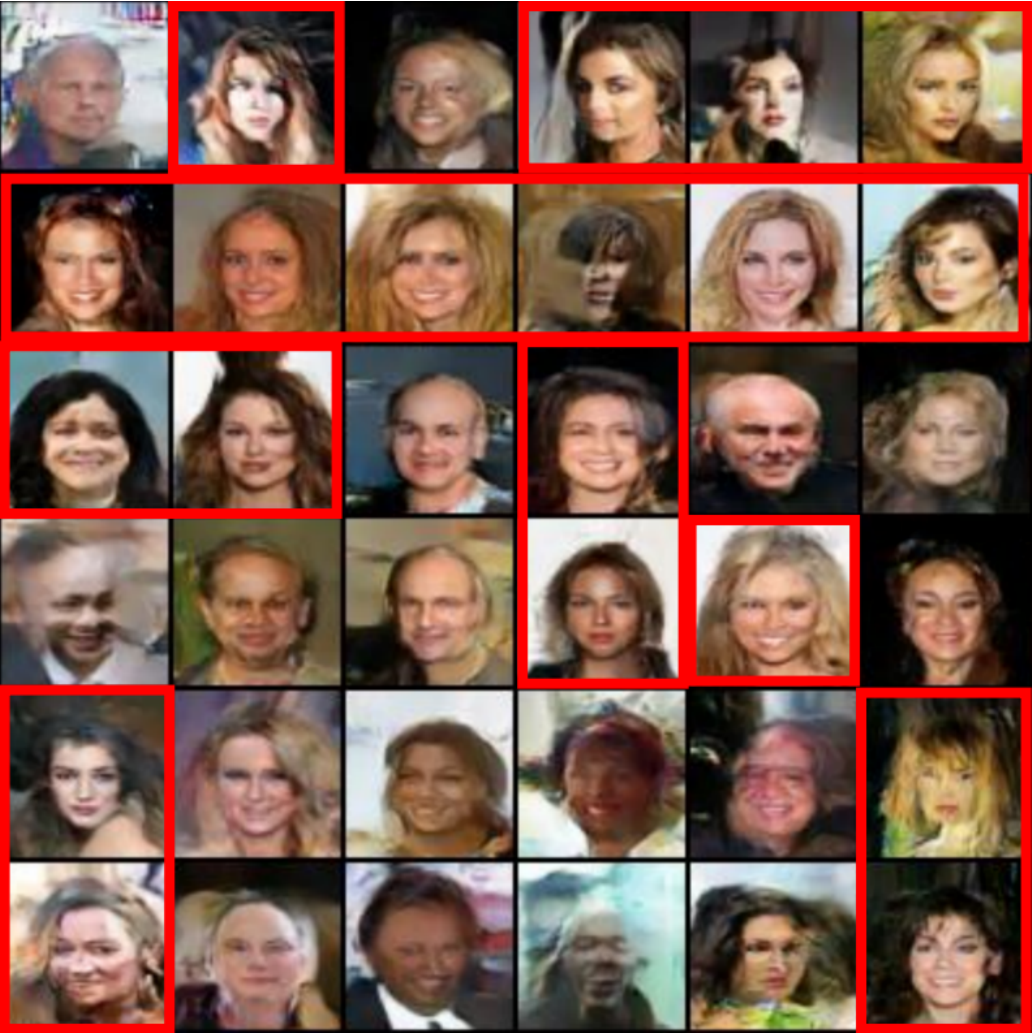}&
    \includegraphics[width=0.99\linewidth]{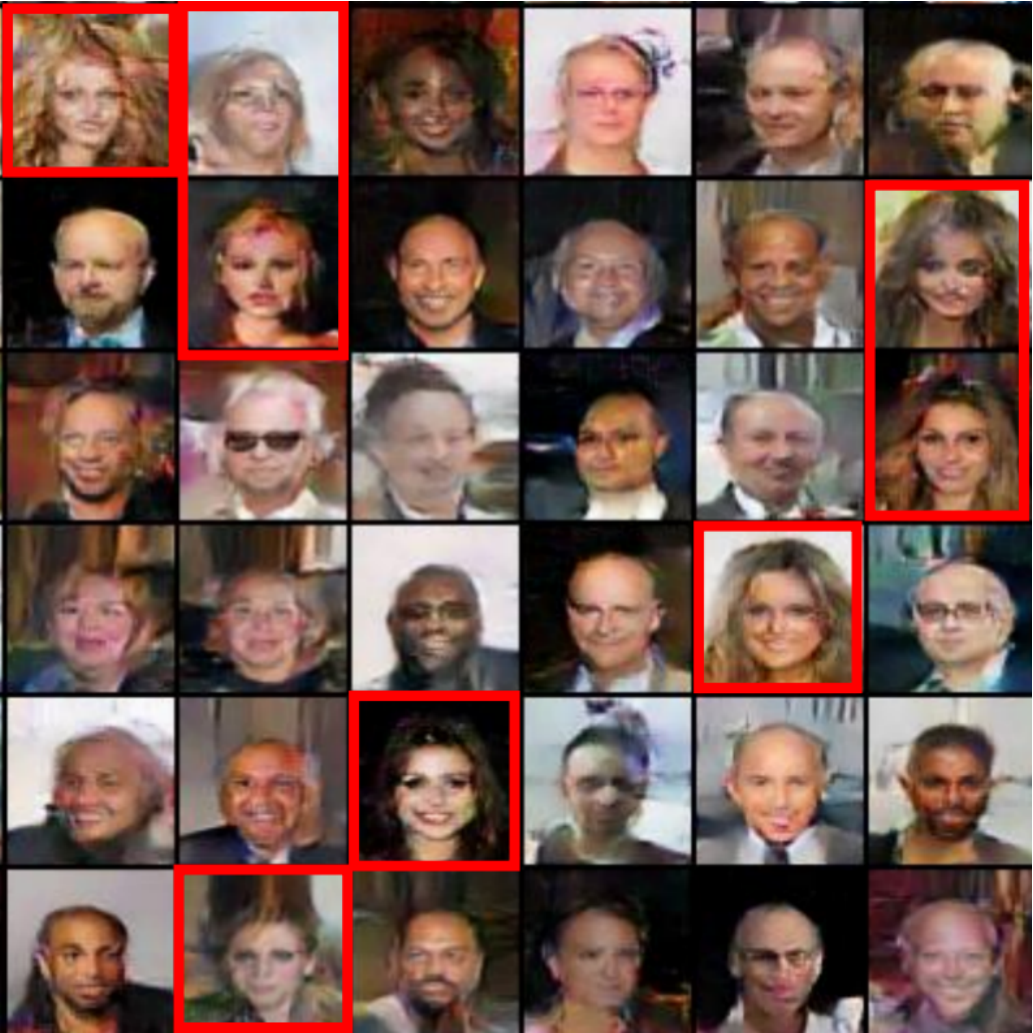} & 
    \includegraphics[width=0.99\linewidth]{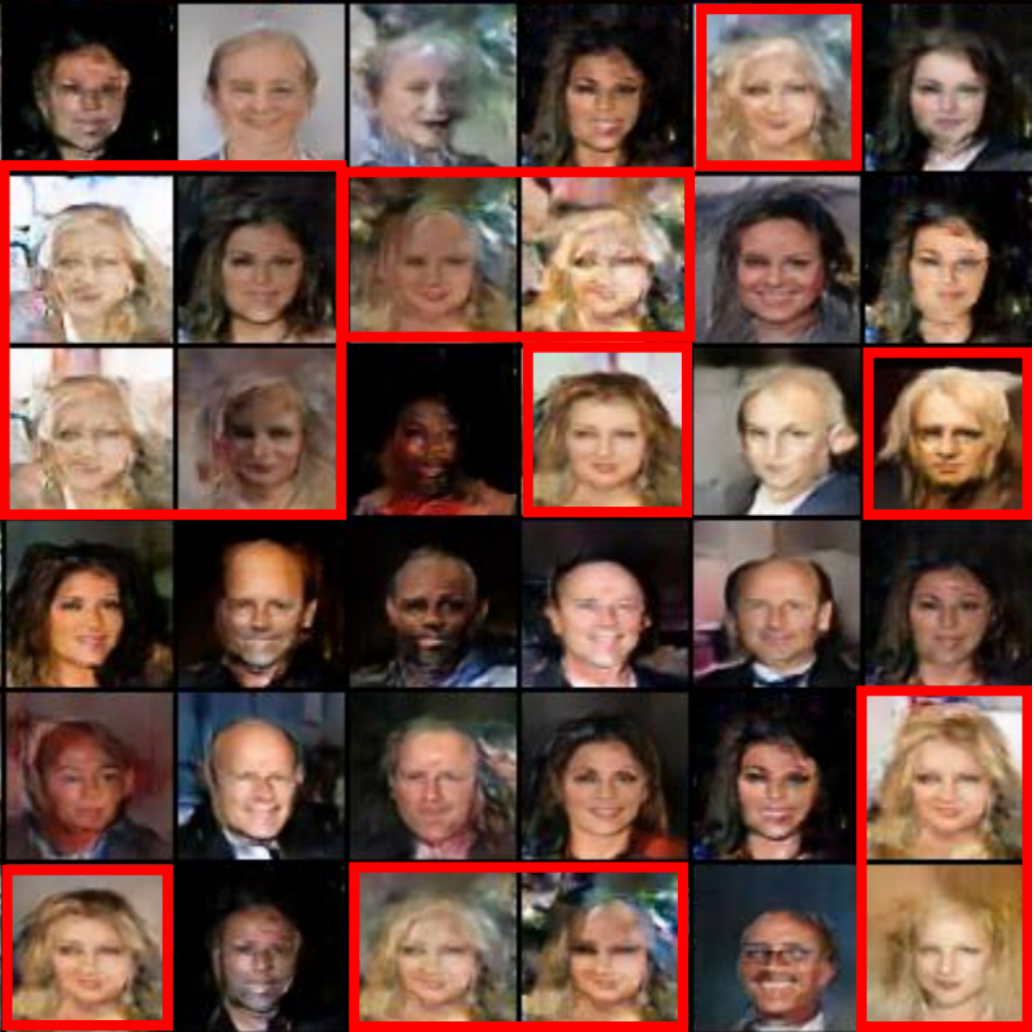}&
     \includegraphics[width=0.99\linewidth]{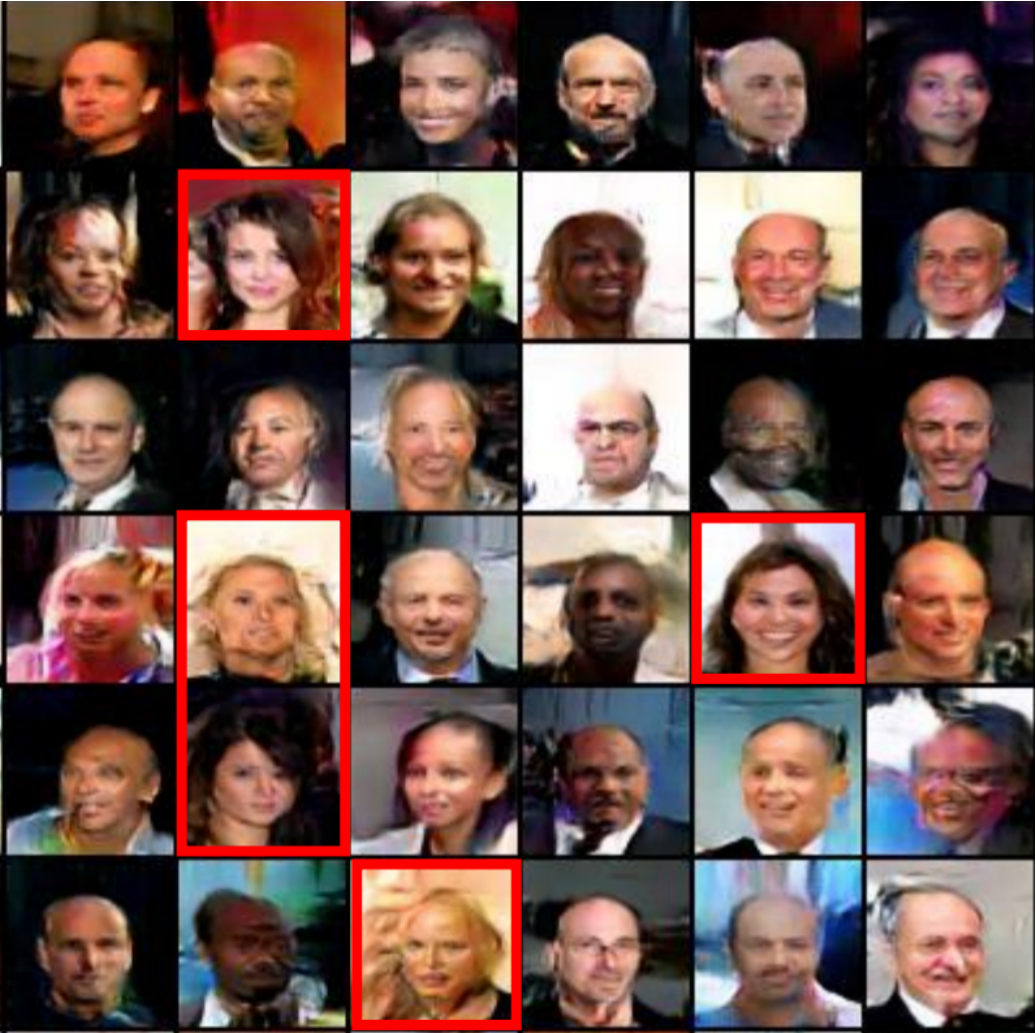} \\
    & \makecell[c]{FID = 55.25}  
    & \makecell[c]{FID = 53.30} 
    & \makecell[c]{FID = 48.06} 
    & \makecell[c]{FID = 47.06}  
    & \makecell[c]{FID = 48.15}
    & \makecell[c]{FID = 39.39}
  \end{tabular} 
  \caption{Complete illustration of images generated by different models trained on contaminated datasets under the scenario of $\gamma_c=0.2$ and $\gamma_p=0.5$.} 
  \label{GAN_Samples}
  \end{flushleft}
  \vspace{-3mm}
\end{figure*}

\paragraph{Complete Results for the Impact of the Number of Classes of Target Instance}
Table \ref{tablek} shows how the generation performance is affected when the target instances are composed of more classes of instances. As can be observed from the table, the increased diversity of desirable target generation images does not have a negative impact on our model, which still outperforms other baseline models on the criterion of generative performance. As our metric FID is a relative value, comparisons between different number of generation types may be not feasible.

\paragraph{Complete Results for Anomaly Detection on Contaminated Datasets}
Table \ref{tableadp} shows the the complete results for anomaly detection tasks. From the table, it can be observed that our proposed PuriGAN-based method has a very stable performance across different ratios of contamination and outperforms all other baselines. By contrast, the performances of all compared unsupervised and semi-supervised methods deteriorate steadily as $\gamma_p$ increases. Unsupervised anomaly detection methods model the probability distributions on contaminated datasets, which is a cause of their poor performance. Semi-supervised anomaly detection methods leverage the collected anomaly data to recognize some anomalies, and thus their performance has achieved a great improvement. However, these methods still not consider the possibility that the training data is not clean, thus leading to their performance not being optimal.

\paragraph{Illustration of FID Variation as Iterations Progress on}
Figure \ref{FID-iterations} shows the variation of FID metric as iterations progress on. From the figure, it can be observed that the FID value of our model decreases rapidly at the beginning of the training phase. As the number of iterations increases, the performance of our model gradually tends to decrease smoothly and eventually terminate in better results. In contrast, the fluctuations of PU-NDA and PU-LSGAN are more dramatic and their instability eventually leads to the poor performance. An important reason why our proposed model performs so well is that our model jointly considers purification and generation in an end-to-end manner. The valuable information of the contamination dataset can be fully utilized in the generation stage while avoiding causing the model to generate undesired instances.

\paragraph{Illustration of Generated Images}
Figure \ref{GAN_Samples} shows the generated images under the specific setting of $\gamma_p=0.5$ and $\gamma_c=0.2$. In this experiment, the instances from categories `6', `pullover', `0' and `bald'  are selected as the target instances for MNIST, F-MNIST, SVHN and CELEBA, respectively, while instances from other categories are viewed as contamination. It is intuitive to observe that our model generates images with almost no contaminated samples, and of perceptually better quality. As discussed in the paper,  unsupervised methods, NDA and LSGAN generate many images with undesired classes. Two-stage methods, PU-NDA and PU-LSGAN, partially alleviate the influences of contamination
instances, but we can see they still generate more contaminated images than our model and the generated images of PU-NDA seem to be very blurred. Notice that, GenPU, which achieve great success in PU-Learning tasks, still generated a lot of contaminated images. This phenomenon can be explained by the fact that GenPU focus on obtaining a better classifier/discriminator rather than generating high-quality desired instances.

\end{document}